\PassOptionsToPackage{table,xcdraw,prologue}{xcolor}

\documentclass[sigconf]{acmart}

\usepackage{graphicx}
\usepackage{balance}
\usepackage{bbm}
\usepackage{url}
\usepackage{amsmath}
\usepackage[linesnumbered,ruled,noend]{algorithm2e}

\usepackage{caption}
\usepackage{subcaption}

\usepackage[table,xcdraw,prologue]{xcolor}
\usepackage{cleveref}
\usepackage{multirow}
\usepackage{xcolor}

\usepackage{comment}
\usepackage{enumitem}
\usepackage{xspace}

\usepackage{amsthm}
\usepackage{thmtools,thm-restate}

\usepackage[english]{babel}

\newcommand{\at}[1]{\protect\ensuremath{\mathsf{#1}}\xspace}

\newcommand{\stitle}[1]{\vspace{1ex}\noindent{\bf #1}}

\newcommand{\rev}[1]{{\color {black}#1}}

\newboolean{techreport}
\setboolean{techreport}{true}

\newboolean{revision}
\setboolean{revision}{true}

\newcommand{\squishlist}{ 
   \begin{list}{$\bullet$}
    { \setlength{\itemsep}{0pt}      \setlength{\parsep}{3pt} 
      \setlength{\topsep}{3pt}       \setlength{\partopsep}{0pt}
      \setlength{\leftmargin}{1.5em} \setlength{\labelwidth}{1em}
      \setlength{\labelsep}{0.5em} } }

\newcommand{\squishend}{
    \end{list}  } 

\newcommand{\systems}{iFlipper}
\newcommand{\greedy}{{\em Greedy}}
\newcommand{\gradient}{{\em Gradient}}
\newcommand{\kmeans}{{\em kMeans}}

\newtheorem{definition}{Definition}
\newtheorem{example}{Example}
\newtheorem{theorem}{Theorem}

\newtheorem{lemma}[theorem]{Lemma}

\AtBeginDocument{%
  \providecommand\BibTeX{{%
    \normalfont B\kern-0.5em{\scshape i\kern-0.25em b}\kern-0.8em\TeX}}}

\setcopyright{acmcopyright}
\copyrightyear{2023}
\acmYear{2023}
\acmDOI{10.1145/1122445.1122456}

\acmConference[SIGMOD '23]{SIGMOD '23: ACM SIGMOD International Conference on Management of Data}{2023}{Seattle, USA}
\acmBooktitle{SIGMOD '23, June 20--25, 2023, Seattle, USA}
\acmPrice{15.00}
\acmISBN{978-1-4503-XXXX-X/18/06}

\begin{document}
\fancyhead{}

\setlength{\abovedisplayskip}{3pt}
\setlength{\belowdisplayskip}{4pt}
\setlength{\textfloatsep}{3pt}
\title{\systems{}: Label Flipping for Individual Fairness}

\author{Hantian Zhang}
\authornote{Equal contribution and co-first authors}
\affiliation{%
  \institution{Georgia Institute of Technology}
}
\email{hantian.zhang@cc.gatech.edu}

\author{Ki Hyun Tae}

\authornotemark[1]
\affiliation{%
  \institution{KAIST}
}
\email{kihyun.tae@kaist.ac.kr}

\author{Jaeyoung Park}
\affiliation{%
  \institution{KAIST}
}
\email{jypark@kaist.ac.kr}

\author{Xu Chu}
\affiliation{%
  \institution{Georgia Institute of Technology}
}
\email{xu.chu@cc.gatech.edu}

\author{Steven Euijong Whang}
\affiliation{%
  \institution{KAIST}
}
\email{swhang@kaist.ac.kr}

\begin{abstract}

As machine learning becomes prevalent, mitigating any unfairness present in the training data becomes critical. 
Among the various notions of fairness, this paper focuses on the well-known individual fairness, which states that similar individuals should be treated similarly.
While individual fairness can be improved when training a model (in-processing), we contend that fixing the data before model training (pre-processing) is a more fundamental solution.
In particular, we show that {\em label flipping is an effective pre-processing technique for improving individual fairness}.

Our system \systems{} solves the optimization problem of minimally flipping labels given a limit to the individual fairness violations, where a violation occurs when two similar examples in the training data have different labels.
We first prove that the problem is NP-hard.
We then propose an approximate linear programming algorithm and provide theoretical guarantees on how close its result is to the optimal solution in terms of the number of label flips.
We also propose techniques for making the linear programming solution more optimal without exceeding the violations limit.
Experiments on real datasets show that \systems{} significantly outperforms other pre-processing baselines in terms of individual fairness and accuracy on unseen test sets.
In addition, \systems{} can be combined with in-processing techniques for even better results.
\end{abstract}

\maketitle



\section{Introduction} \label{sec:intro}
Machine learning (ML) impacts our everyday lives where applications include recommendation systems~\cite{chaney2018algorithmic}, job application~\cite{dastin2018amazon}, and face recognition~\cite{buolamwini2018gender}. Unfortunately, ML algorithms are also known to reflect or even reinforce bias in the training data and thus make unfair decisions~\cite{barocas2016big,whang2021responsible}. This issue draws concerns from both the public and research community, so algorithms have been proposed to mitigate bias in the data and improve the fairness of ML models.

There are several prominent notions of fairness, and we focus on individual fairness~\cite{dwork2012fairness}, which states that \textit{similar individuals must be treated similarly}. Suppose that two applicants are applying to the same school. If the two applicants have similar application materials, then it makes sense for them to obtain the same or similar outcomes. Likewise if two individuals are applying for a loan and have similar financial profiles, it is fair for them to be accepted or rejected together as well. In addition to individual fairness, the other prominent fairness notions include group fairness and causal fairness. Group fairness~\cite{zafar2017fairness,agarwal2018reductions,zhang2021omnifair} focuses on the parity between two different sensitive groups (e.g., male versus female), and causal fairness~\cite{kusner2017counterfactual} looks at fairness from a causal perspective (e.g., does gender affect an outcome?). These are orthogonal notations of fairness and do not capture the individual fairness we focus on.

How can one improve a model's individual fairness? There are largely three possible approaches: fixing the data before model training (pre-processing), changing the model training procedure itself (in-processing), or updating the model predictions after training (post-processing). Among them, most of the literature focuses on in-processing~\cite{DBLP:conf/iclr/YurochkinBS20,yurochkin2021sensei,vargo2021individually} and more recently pre-processing techniques~\cite{pmlr-v28-zemel13,ifair,pfr2019}. We contend that pre-processing is important because biased data is the root cause of unfairness, so fixing the data is the more fundamental solution rather than having to cope with the bias during or after model training. The downside of pre-processing is that one cannot access the model and has to address the bias only using the training data. Due to this challenge, only a few pre-processing techniques for individual fairness have been proposed, which we will compare with in the experiments.

We propose {\em label flipping as a way to mitigate data bias for individual fairness} and assume a binary classification setting where labels have 0 or 1 values. Given a training set of examples, the idea is to change the labels of some of the examples such that similar examples have the same labels as much as possible. 
Which examples are considered similar is application dependent and non-trivial and can be learned from input data~\cite{ilvento2020metric, fairmetric, pmlr-v119-mukherjee20a} or obtained from annotators (e.g., humans). This topic is important for individual fairness, but not the focus of this work where we assume the criteria is given as an input. For example, one may compute the Euclidean distance between two examples and consider them similar if the distance is within a certain threshold. As the training set becomes more individually fair, the trained model becomes fairer as well (see Section~\ref{sec:accuracyfairnesstradeoff}). Our label flipping approach is inspired by the robust training literature of learning from noisy labels~\cite{DBLP:journals/corr/abs-2007-08199} where the labeling itself may be imperfect. The standard approach for handling such labels is to ignore or fix them~\cite{DBLP:conf/icml/SongK019}. In our setting, we consider any label that reduces individual fairness as biased and would like to fix it. 


We can use a graph representation to illustrate label flipping as shown in Figure~\ref{fig:graphrep}. Each node represents a training example, and its color indicates the original label (black indicates label 1, and white indicates 0). Two similar nodes (defined in Section~\ref{sec:preliminaries}) are connected with an edge, and a violation occurs when an edge is connecting two nodes with different labels. In  Figure~\ref{fig:graphrep}, there are four nodes where only the nodes 1, 2, and 3 are similar to each other. 
Each edge also has an associated weight, which reflects the similarity of the two nodes. For simplicity, let us assume that all weights are 1.
We can see that there are two ``violations'' of fairness in this dataset: (1,2) and (1,3) because there are edges between them, and they have different colors. After flipping the label of node 1 from 1 to 0, we have no violations.

\begin{figure}[t]
\centering
  \includegraphics[width=0.71\columnwidth]{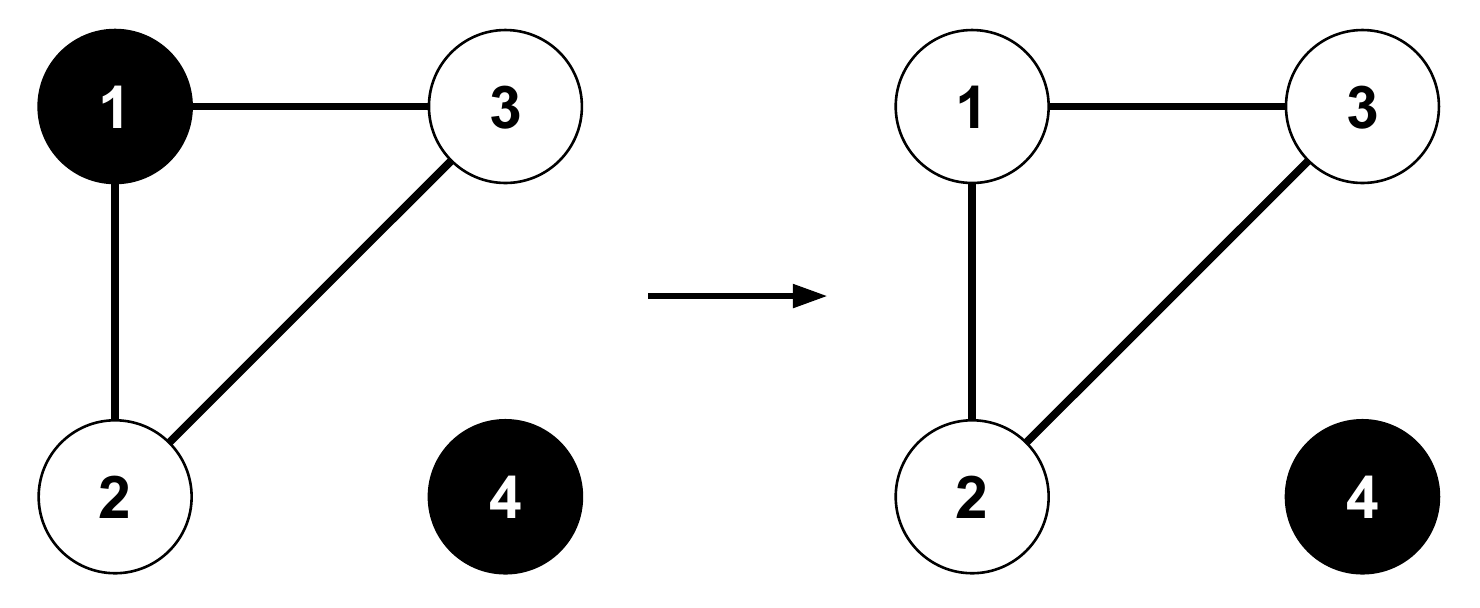}
  \vspace{-0.3cm}
  \caption{Label flipping example for individual fairness using a graph representation. Two similar nodes have an edge, and color indicates the label. By flipping the label of node 1, all pairs of similar nodes have the same label, which is considered individually fair.}
  \label{fig:graphrep}
\end{figure}

Our key contribution is in formulating and solving a constrained label flipping problem for the purpose of individual fairness. Just like in a robust training setup, we assume labelers can make labeling mistakes in a biased fashion. Here the effective solution is to debias the labels by flipping them. One consequence of label flipping is an accuracy-fairness trade-off where the model's accuracy may diminish. As an extreme example, if we simply flip all the labels to be 0, then a trained model that only predicts 0 is certainly fair, but inaccurate to say the least. Even if we carefully flip the labels, we still observe a trade-off as we detail in Section~\ref{sec:accuracyfairnesstradeoff}. We thus formulate the optimization problem where the objective is to minimize the number of label flipping while limiting the {\em total error}, which is the total degree of the individual fairness violations (see Definition~\ref{def:totalerror}). The optimization can be formally stated as an instance of mixed-integer quadratic programming (MIQP) problem, and we prove that it is NP-hard. We then transform the problem to an approximate linear programming (LP) problem for efficient computation. Interestingly, we show that our LP algorithm has theoretical guarantees on how close its result is to the optimal solution in terms of the number of flips performed. We then further optimize the solution given by the LP algorithm to reduce the number of flips while ensuring the total error does not exceed the given limit. 


We call our approach \systems{} and empirically show how its label flipping indeed results in individually-fair models and significantly outperforms other pre-processing baselines on real datasets. In particular, the state-of-the-art baselines~\cite{ifair,pfr2019} use representation learning to place similar examples closer in a feature space. We show that \systems{} has better accuracy-fairness trade-off curves and is also significantly more efficient. We also compare \systems{} with  baselines (e.g., greedy approach and k-means clustering) for solving the optimization problem and show that \systems{} is superior. Finally, we demonstrate how \systems{} can be integrated with in-processing techniques~\cite{DBLP:conf/iclr/YurochkinBS20} for even better results. We release our code as a community resource~\cite{github}.

The rest of the paper is organized as follows.

\squishlist
\item We formulate the label flipping optimization as an MIQP problem and prove that it is NP-hard (Section~\ref{sec:probdef}).
\item We propose \systems{}, which solves this problem by converting it into an approximate LP algorithm that has theoretical guarantees and present an additional optimization to further improve the solution given by the LP solver (Section~\ref{sec:iflipper}).
\item We evaluate \systems{} on real datasets and show how it outperforms other pre-processing baselines and can be integrated with in-processing techniques for better results (Section~\ref{sec:exp}).
\squishend

\section{Problem Definition}\label{sec:probdef}

\subsection{Preliminaries}\label{sec:preliminaries}
We focus on a \emph{binary classification} setting, and assume a training dataset \at{D} = $\{(x_i, y_i)\}_{i=1}^n$ where $x_i$ is an example, and $y_i$ is its label having a value of 0 or 1. A binary classifier $h$ can be trained on \at{D}, and its prediction on a test example $x$ is $h(x)$.


Individual Fairness~\cite{dwork2012fairness} states that similar individuals must be treated similarly. The criteria for determining if two examples are similar depends on the application, and we assume it is given as input, though research on how to automatically discover such similarity measures exists~\cite{ilvento2020metric, fairmetric, pmlr-v119-mukherjee20a}. For example, for each $x$, we may consider all the examples within a certain distance to be similar. Alternatively, we may consider the top-$k$ nearest examples to be similar. In our work, we assume as input a given similarity measure, which can be applied on the training set to produce a similarity matrix $W_{ij} \in \mathbb{R}^{n \times n}$ where $W_{ij} > 0$ if and only if $x_i$ and $x_j$ are deemed similar according to the measure. If $W_{ij}> 0 $, we set it to be the similarity of $x_i$ and $x_j$, although any other positive value can be used as well. In order to satisfy individual fairness, we introduce the notion of {\em individual fairness violation} as follows.

\begin{definition} \textbf{Individual Fairness Violation.}
\label{def:violation}
Given a similarity matrix $W$ on a training set, an individual fairness violation occurs when $W_{ij} > 0$, but $y_i \neq y_j$. The magnitude of the violation is defined to be $W_{ij}$. 

\end{definition}

\begin{definition} \textbf{Total Error.}
\label{def:totalerror}
Given a similarity matrix $W$, the total error is the sum of all $W_{ij}$ values where $y_i \neq y_j$.
\end{definition}

Our goal is to reduce the total error to be within a maximum allowed amount.

\begin{definition} \textbf{m-Individually Fair Dataset.}
\label{def:individualfairness}
Given a similarity matrix $W$, a dataset is considered m-individually fair if the total error is at most $m$.
\end{definition}
A smaller $m$ naturally translates to an m-individually fair dataset that is likely to result in a more individually fair model on the unseen test set as we demonstrate in Section~\ref{sec:accuracyfairnesstradeoff}.

\rev{
\paragraph*{Incomplete Similarity Matrix} 
If we only have partial knowledge of the similarities, \systems{} can be naturally adapted to utilize or fix the partial data. The straightforward solution is to treat the partial similarity as if it is complete information where the pairs with unknown similarities are not included in our optimization formulation. Another solution is to ``extrapolate'' the partial similarity matrix to reconstruct the full similarity matrix. For example, one can learn similarity functions using the partial data. In this case, \systems{}'s performance will largely depend on how accurate the learned similarity function is.}

\stitle{Evaluating the Fairness of a Trained Model $h$.} To evaluate the final individual fairness of a trained model $h$, we compute the widely-used {\em consistency score}~\cite{pfr2019} of model predictions on an unseen test set as defined below. Here $W_{ij}$ is the similarity matrix on the unseen test set.
\[
\text{\em{Consistency Score}} = 1-\frac{\sum_{i}\sum_{j}|h(x_i)-h(x_j)| \times W_{ij}}{\sum_{i}\sum_{j} W_{ij}}
\]

Intuitively, if the model is trained on an individually-fair dataset, the predictions on the test set between similar individuals tend to be the same, so the consistency score on the test set increases. 
In the extreme case, a consistency score of 1 indicates that all similar pairs of test examples get the same predictions from the model. 


\stitle{Local Sensitive Hashing.} 
We construct a similarity matrix $W$ using locality sensitive hashing (LSH)~\cite{Andoni2015falconn} instead of materializing it entirely and thus avoid performing $O(n^2)$ comparisons. We exploit the fact that $W$ is sparse, because most examples are dissimilar. This approach is similar to blocking techniques in entity resolution~\cite{DBLP:books/daglib/0030287}.




\subsection{Label Flipping Optimization Problem}\label{sec:opt}

We define the label flipping optimization problem for individual fairness. Given a training dataset \at{D} and a limit $m$ of total error allowed, our goal is to flip the minimum number of labels in \at{D} such that it becomes \emph{m-individually fair}. This statement can be formalized as a mixed-integer quadratic programming (MIQP) problem:
\begin{equation}\label{equ:miqp}
\begin{split}
    \text{(MIQP)} \quad \min \quad &\sum_{i=1}^{n}{(y_i-y_i')^2} \\
    \text{s.t.} \quad &\sum_{i=1}^{n} \sum_{j=1}^{n} W_{ij} (y_i-y_j)^2 \leq m \\
    & y_i \in \text{\{0,1\}}, \forall i 
\end{split}
\end{equation}
where $y_i$ indicates an output label, and $y_i'$ is its original value. Intuitively, we count the number of flips ($(y_i-y_i')^2$ = 1) while ensuring that the total error ($(y_i-y_j)^2$ = 1 and $W_{ij} > 0$) is within the limit $m$. We call a solution {\em feasible} if it satisfies the error constraints in Equation~\ref{equ:miqp}, but may or may not be optimal.

MIQP is an NP-hard problem in general, and we prove that our specific instance of the MIQP problem is also NP-hard.

\begin{theorem} \label{lem:nphard}
The MIQP problem in Equation~\ref{equ:miqp} is NP-hard.
\end{theorem}

\ifthenelse{\boolean{techreport}}{The full proof for Theorem~\ref{lem:nphard} is in Section~\ref{sec:proofnphard}}{The full proof for Theorem~\ref{lem:nphard} can be found in our technical report~\cite{iflippertr}}. The key idea is to reduce the well-known NP-hard \textit{at most $k$-cardinality $s$-$t$ cut problem}~\cite{DBLP:journals/dam/BruglieriME04} to our MIQP problem.
\subsection{Baseline Algorithms}
\label{sec:naivesolutions}

Our key contribution is to propose  algorithms for the label flipping problem that not only scales to large datasets, but also provides feasible and high-quality solutions. We present three na\"ive algorithms that are efficient, but may fail to produce feasible solutions. In comparison, our method in Section~\ref{sec:iflipper} always produces feasible solutions with theoretical guarantees.


\stitle{Greedy Algorithm.} 
The greedy algorithm repeatedly flips labels of nodes that reduce the total error the most. The algorithm terminates if the total error is $m$ or smaller, or if we cannot reduce the error anymore. For example, suppose we start from the graph in Figure~\ref{fig:graphrep} where there are initially two violations (recall we assume that $W_{ij}=1$ for all edges for simplicity) and set $m$ = 1. We need to determine which label leads to the largest reduction in error when flipped. We can see that flipping node 1 will reduce the total error by 2 while flipping the other nodes does not change the total error. Hence, the greedy algorithm flips node 1 to reduce the total error to 0 and then terminates. While the greedy approach seems to work for Figure~\ref{fig:graphrep},  in general it does not always find an optimal result and may fail to produce a feasible solution even if one exists. Consider the example in Figure~\ref{fig:chain} where there is one violation between nodes 2 and 3. Again, we assume that $W_{ij} = 1$ for all edges for simplicity. If we set $m$ = 0, the feasible solution is to flip nodes 1 and 2 or flip nodes 3 and 4 together. Unfortunately, the greedy algorithm immediately terminates because it only flips one node at a time, and no single flip can reduce the error.

\begin{figure}[ht]
\centering
  \vspace{-0.4cm}
  \includegraphics[width=0.75\columnwidth]{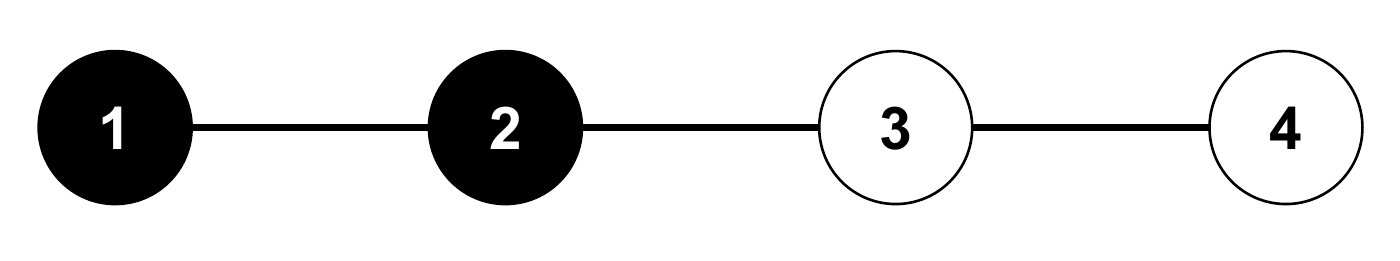}
  \vspace{-0.4cm}
  \caption{A graph where the greedy algorithm fails to find a feasible solution with zero violations.}
 \vspace{-0.4cm}
  \label{fig:chain}
\end{figure}

The computational complexity of the greedy algorithm is $O(n^2)$ because we flip at most $O(n)$ labels, and each time a label is flipped, we update the total error of each neighbor of the node. 

\stitle{Gradient Based Algorithm.}
The second na\"ive approach is to use a Lagrangian multiplier to move the error constraint into the objective function and solve the problem via gradient descent. This approach is common in machine learning, and the following equation shows the problem setup:
\begin{equation}\label{equ:lagrangian}
\begin{split}
    \text{(Lagrangian)    } \min \text{    } &\sum_{i=1}^{n}{(y_i-y_i')^2} + \lambda \sum_{i=1}^{n} \sum_{j=1}^{n} W_{ij} (y_i-y_j)^2 \\
    & y_i \in \text{[0,1]}, \forall i 
\end{split}
\end{equation}
where $\lambda$ is a hyperparameter that controls the trade-off between fairness and accuracy. A higher $\lambda$ favors better fairness. Although this gradient based algorithm is efficient, it shares the same problem as the greedy algorithm where it may get stuck in a local minima and thus fail to find a feasible solution. 

\stitle{Clustering Based Algorithm.}
The third approach is to use a clustering algorithm for clustering examples and assigning the same label to each cluster. We use k-means as a representative clustering algorithm. If $k=1$, then all examples will have the same label, and we can simply choose the majority label of the examples. If $k$ is set correctly, and the clustering is perfect, then only the similar examples will be clustered together. However, this approach is unsuitable for our problem, which is to flip the minimum number of labels to have at most $m$ total error. Reducing the total error to 0 is not the primary goal as there may be a large degradation in accuracy. To cluster for our purposes, one would have to adjust $k$ to find the clusters with just the right amount of total error, but this fine-tuning is difficult as we show in
Section~\ref{sec:optimizationcomparison}.


\section{iFlipper}\label{sec:iflipper}

We explain how \systems{} converts the MIQP problem into an approximate linear program (LP) problem and produces a feasible solution with theoretical guarantees. The conversion is done in two steps: from MIQP to an equivalent integer linear program (ILP) using linear constraints (Section~\ref{sec:milp}) and from the ILP problem to an approximate LP problem (Section~\ref{sec:linearrelaxation}). We present \systems{}'s algorithm to solve the approximate LP problem and show why its result always leads to a feasible solution (Section~\ref{sec:algorithm}). We then provide theoretical guarantees on how far the result is from the optimal solution of the ILP problem, and propose a reverse-greedy approach to further optimize the solution (Section~\ref{sec:optimality}). Finally, we present \systems{}'s overall workflow with a complexity analysis (Section~\ref{sec:systems}).

\subsection{From MIQP to Equivalent ILP} \label{sec:milp}

We convert the MIQP problem to an equivalent ILP problem. We first replace the squared terms in Equation~\ref{equ:miqp} to absolute terms:
\vspace{-0.25mm}
\begin{equation}\label{equ:miap}
\begin{split}
     \quad \min \quad &\sum_{i=1}^{n}{|y_i-y_i'|} \\
    \text{s.t.} \quad &\sum_{i=1}^{n} \sum_{j=1}^{n} W_{ij} |y_i-y_j| \leq m \\
    & y_i \in \text{\{0,1\}}, \forall i 
\end{split}
\end{equation}
\vspace{-0.25mm}
The resulting formulation is equivalent to the original MIQP because $y_i$ and $y'_i$ have binary values. 

To convert this problem into an equivalent ILP problem, we replace each absolute term with an XOR expression, which can be expressed as four linear constraints. For two binary variables $x$ and $y$, one can easily see that $|x - y| = x \text{ XOR } y$. Also, each expression $z = x \text{ XOR } y$ is known to be equivalent to the following four linear constraints: $z \leq x + y$, $z \geq y - x$, $z \geq x - y$, and $z \leq 2 - x - y$. For example, if $x=1$ and $y=1$, $z$ is bounded by $z \leq 2 - x - y$ and is thus 0. For other combinations of $x$ and $y$, $z$ will be bounded to its correct $\text{XOR}$ value as well. We thus introduce the auxiliary variables $z_i$ and $z_{ij}$ to represent $y_i \text{ XOR } y'_i$ and $y_{i} \text{ XOR } y_j$, respectively, and obtain the following integer linear programming (ILP) problem:
\vspace{-0.25mm}
\begin{equation}\label{equ:milp}
\begin{split}
    \text{(ILP)} \quad \min \quad &\sum_{i=1}^{n}{z_i} \\
    \text{s.t.} \quad &\sum_{i=1}^{n} \sum_{j=1}^{n} W_{ij }z_{ij} \leq m \\
    & y_i, z_i\in \text{\{0,1\}}, \forall i, \quad z_{ij}\in \text{\{0,1\}}, \forall i,j\\
    & z_i - y_i\leq y'_i, \quad z_i - y_i \geq  - y'_i \\
    &z_i + y_i \geq y'_i, \quad z_i + y_i \leq 2 - y'_i \\
    & z_{ij} - y_i - y_j \leq 0, \quad z_{ij} - y_i + y_j \geq 0 \\
    & z_{ij} + y_i - y_j\geq 0, \quad z_{ij} + y_i + y_j \leq 2 \\
\end{split}
\end{equation}
\vspace{-0.25mm}
Since the ILP problem is equivalent to the MIQP problem, we know it is NP-hard as well. In the next section, we convert the ILP problem to an approximate linear program (LP), which can be solved efficiently.

\subsection{From ILP to Approximate LP} \label{sec:linearrelaxation}
We now relax the ILP problem to an approximate LP problem. At this point, one may ask why we do not use existing solvers like CPLEX~\cite{cplex2009v12}, MOSEK~\cite{mosek}, and Gurobi~\cite{gurobi}, which also use LP relaxations. All these approaches repeatedly solve LP problems using branch-and-bound methods and have time complexities exponential to the number of variables in the worst case. In Section~\ref{sec:optimizationcomparison}, we demonstrate how slow the ILP solvers are. Our key contribution is finding a near-exact solution to the original ILP problem by solving the LP problem \emph{only once}.

We first replace the integer constraints in \rev{Equation~\ref{equ:milp}} with range constraints to obtain the following LP problem:
\begin{equation}\label{equ:mlp}
\begin{split}
    \text{(LP)} \quad \min \quad &\sum_{i=1}^{n}{z_i} \\
    \text{s.t.} \quad &\sum_{i=1}^{n} \sum_{j=1}^{n} W_{ij} z_{ij} \leq m \\
    & y_i, z_i\in \text{[0,1]}, \forall i, \quad z_{ij}\in \text{[0,1]}, \forall i,j\\
    & z_i - y_i\leq y'_i, \quad z_i - y_i \geq  - y'_i \\
    &z_i + y_i \geq y'_i, \quad z_i + y_i \leq 2 - y'_i \\
    & z_{ij} - y_i - y_j \leq 0, \quad z_{ij} - y_i + y_j \geq 0 \\
    & z_{ij} + y_i - y_j\geq 0, \quad z_{ij} + y_i + y_j \leq 2 \\
\end{split}
\end{equation}

Now the violation between two points becomes the product of the weight of the edge $W_{ij}$ and the absolute difference between the two labels $z_{ij}=|y_i-y_j|$. Although this problem can be solved more efficiently than the ILP problem, its result cannot be used as is because of the continuous values. Hence, we next convert the result into a binary solution that is close to the optimal solution of the ILP problem. A na\"ive method is to round the continuous values to their nearest integers, but this does not guarantee a feasible solution.

\begin{example}
Suppose we start from the graph in Figure~\ref{fig:rounding}a. Just like previous examples, we again assume that $W_{ij}$ is 1 for simplicity. The total error in Figure~\ref{fig:rounding}a is 4, and we set $m = 2$. Solving the LP problem in Equation~\ref{equ:mlp} can produce the solution in Figure~\ref{fig:rounding}b where the labels of nodes 1 and 4 are flipped from 1 to 0.5 (gray color). One can intuitively see that the labels are minimally flipped while the total error is exactly 2. However, rounding the labels changes the two 0.5's back to 1's as shown in Figure~\ref{fig:rounding}c, resulting in an infeasible solution, because the total error becomes 4 again.
\end{example}

\vspace{-0.2cm}
\begin{figure}[h]
\centering
  \includegraphics[width=\columnwidth]{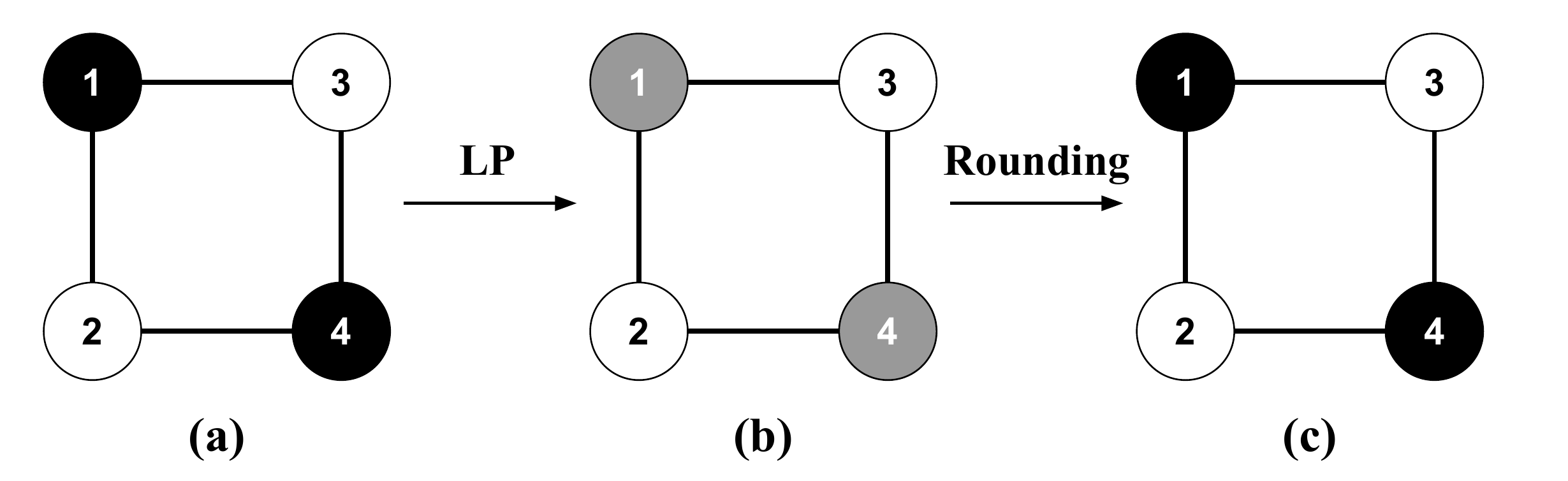}
  \vspace{-0.5cm}
  \caption{Performing simple roundings on the optimal solution's values of the LP problem may result in an infeasible solution for the ILP problem.}
  \label{fig:rounding}
\end{figure}




\subsection{Constructing a Feasible Solution}\label{sec:algorithm}

We now explain how to construct a feasible integer solution for the ILP problem (Equation~\ref{equ:milp}) from an optimal solution of the LP problem (Equation~\ref{equ:mlp}). We first prove a surprising result that any optimal solution $y^*$ can be converted to another optimal solution $\check{y}$ where all of its variables have one of the three values: 0, 1, or some $\alpha \in (0, 1)$. That is, $\alpha$ is between 0 and 1, but not one of them. Next, we utilize this property and propose an adaptive rounding algorithm that converts $\check{y}$ into a feasible solution whose values are only 0's and 1's.

\stitle{Optimal Solution Conversion.}
We prove the following lemma for converting an optimal solution to our desired form.

\begin{lemma} \label{lem:converting}
    Given an optimal solution $y^*$ of Equation~\ref{equ:mlp}, we can always convert $y^{*}$ to a new solution $\check{y}$ where $\check{y}$ is also optimal and only has 0, 1, or a unique $\alpha \in (0, 1)$ as its values.
\end{lemma}

\begin{proof}
Suppose there are at least two distinct values $\alpha$ and $\beta$ that are not 0 or 1 in $y^*$. We can combine all the $y^*$ nodes whose value is $\alpha$ to one node while maintaining the same total error because an $\alpha$-$\alpha$ edge has no violation (see Figure~\ref{fig:clusterexample} for an illustration). We call this collection an $\alpha$-cluster. Likewise, we can generate a $\beta$-cluster.

\vspace{-0.4cm}
\begin{figure}[ht]
\centering
  \includegraphics[width=0.7\columnwidth]{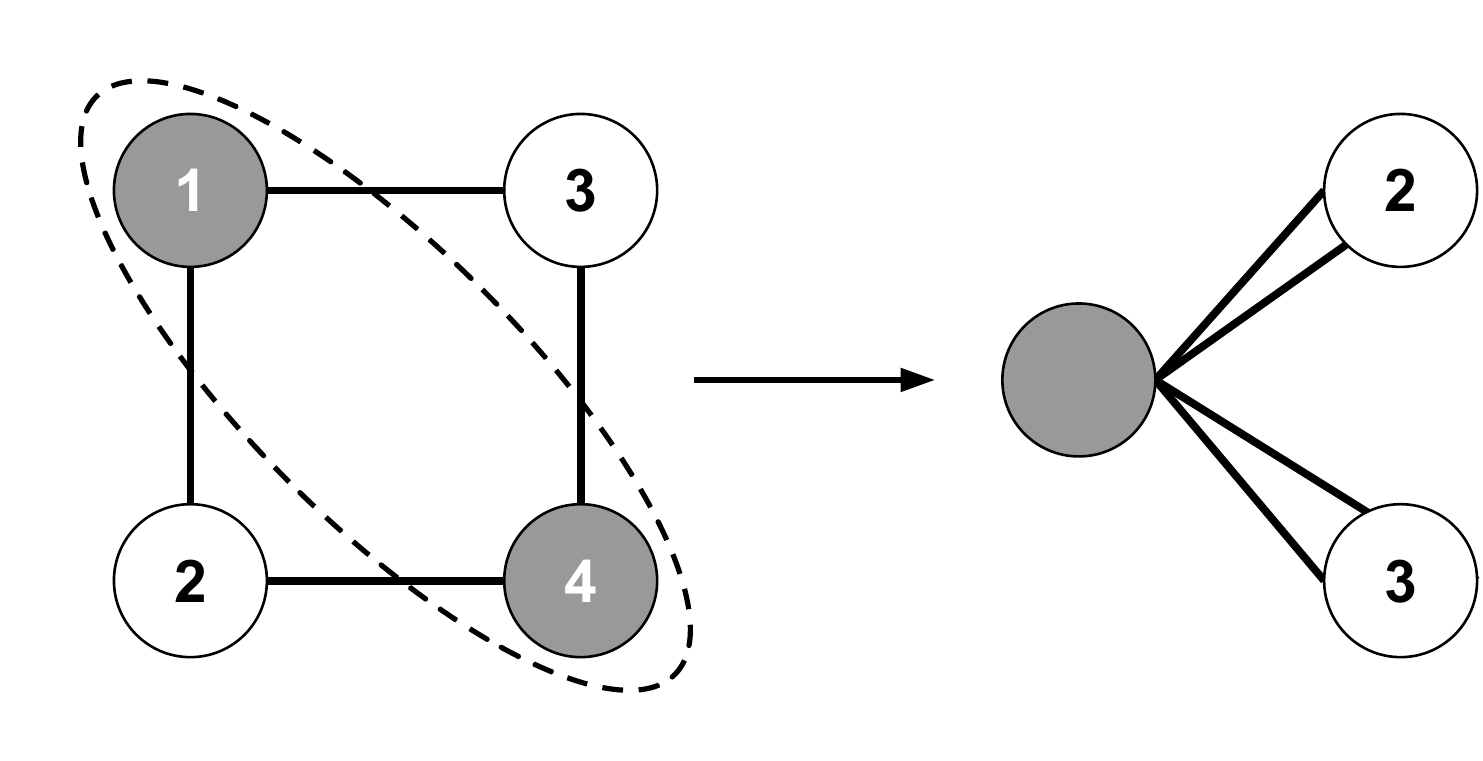}
  \vspace{-0.3cm}
  \caption{Continuing from Figure~\ref{fig:rounding}, recall that the optimal LP solution has nodes 1 and 4 with 0.5 labels. We can combine nodes 1 and 4 to construct a 0.5-cluster connected to nodes 2 and 3 with four edges as shown on the right. The total error is still 2.}
  \label{fig:clusterexample}
\end{figure}
\vspace{-0.1cm}

Without loss of generality, let us assume that 0 < $\alpha$ < $\beta$ < 1, and the sum of the pairwise node similarities between the two clusters is $E$. Suppose an $\alpha$-cluster and a $\beta$-cluster have $A_0$ and $B_0$ nodes whose initial labels are 0, respectively, and $A_1$ and $B_1$ nodes whose initial values are 1, respectively. Let $N_\alpha=A_0-A_1$ and $N_\beta=B_0-B_1$. Now suppose there are $U$ nodes connected to the $\alpha$-cluster by an edge $W_{a_i\alpha}$ and $V$ nodes connected to the $\beta$-cluster by an edge $W_{b_i\beta}$, whose values satisfy
\begin{equation*}
\begin{split}
    &0 \leq a_{1} \leq ... \leq a_{k} < \alpha < a_{k+1} \leq ... \leq a_{U} \leq 1 \text{ and } \\
    &0 \leq b_{1} \leq ... \leq b_{l} < \beta < b_{l+1} \leq ... \leq b_{V} \leq 1.
\end{split}
\end{equation*}

Note that there is no connected node with a value of $\alpha$ or $\beta$ by construction. Let $S_\alpha = \sum_{i=1}^{k}W_{a_i \alpha} - \sum_{i=k+1}^{U}W_{a_i \alpha}$ and $S_\beta = \sum_{i=1}^{l}W_{b_i \beta} - \sum_{i=l+1}^{V}W_{b_i \beta}$. Let us also add the following nodes for convenience: $a_{0}=0$, $a_{U+1}=1$, $b_{0}=0$, and $b_{V+1}=1$. We can then reduce at least one unique non-0/1 value in $y^*$, by either changing $\alpha$ to $a_{k}$ or $a_{k+1}$, or changing $\beta$ to $b_{l}$ or $b_{l+1}$, or changing both $\alpha$ and $\beta$ to the same value, while keeping the solution optimal using the following Lemmas \ref{lem:onecluster} and \ref{lem:twocluster}. The key idea is that at least one of the conversions allows the solution to have the same number of label flippings and the same or even smaller amount of total error, keeping the solution optimal. The conversion process only depends on $N_\alpha$, $S_\alpha$, $N_\beta$, $S_\beta$, and $E$. The complete proof and detailed conditions for each case are given in \ifthenelse{\boolean{techreport}}{Sections~\ref{sec:proofonecluster} and~\ref{sec:prooftwocluster}}{our technical report~\cite{iflippertr}}.


\begin{restatable}{sublemma}{primelemmaone}
\label{lem:onecluster}
For an $\alpha$-cluster with $N_\alpha=0$ in the optimal solution, we can always convert $\alpha$ to either $a_k$ or $a_{k+1}$ while maintaining an optimal solution. As a result, we can reduce exactly one non-0/1 value in the optimal solution.
\end{restatable}

\begin{restatable}{sublemma}{primelemmatwo}
\label{lem:twocluster}
For an $\alpha$-cluster with $N_\alpha\ne0$ and a $\beta$-cluster with $N_\beta\ne0$ in the optimal solution, we can always convert $(\alpha, \beta)$ to one of $(a_k, \beta+\frac{N_\alpha}{N_\beta}(a_k-\alpha))$, $(a_{k+1}, \beta-\frac{N_\alpha}{N_\beta}(a_{k+1}-\alpha))$, $(\alpha+\frac{N_\beta}{N_\alpha}(\beta-b_l), b_{l})$, $(\alpha-\frac{N_\beta}{N_\alpha}(b_{l+1}-\beta), b_{l+1})$, or $(\frac{\alpha N_\alpha+\beta N_\beta}{N_\alpha+N_\beta}, \frac{\alpha N_\alpha+\beta N_\beta}{N_\alpha+N_\beta})$, while maintaining an optimal solution. As a result, we can reduce at least one non-0/1 value in the optimal solution.
\end{restatable}

We can repeat this adjustment until the solution $\check{y}$ has at most one unique value $\alpha$ that is neither 0 nor 1.
\end{proof}
\vspace{-2mm}
Based on Lemma~\ref{lem:converting}, Algorithm~\ref{alg:converting} shows how to convert the optimal LP solution $y^*$ to another optimal solution $\check{y}$ whose values are in \{0, $\alpha$, 1\}. We first obtain all the unique non-0/1 values in $y^*$ and for each value $v$ construct a $v$-cluster. We then pick a cluster (say the $\alpha$-cluster) with $N_\alpha=0$ and change $\alpha$ using the \textsc{TransformWithOneCluster} function, which implements the converting process of Lemma~\ref{lem:onecluster} (see \ifthenelse{\boolean{techreport}}{Section~\ref{sec:proofonecluster}}{our technical report~\cite{iflippertr}} for details) until there are no more such clusters (Lines 4--7). Among the rest of the clusters, we choose two (say the $\alpha$-cluster and $\beta$-cluster) and transform them using the \textsc{TransformWithTwoClusters} function, which implements the combining process of Lemma~\ref{lem:twocluster} described in \ifthenelse{\boolean{techreport}}{Section~\ref{sec:prooftwocluster}}{our technical report~\cite{iflippertr}} (Lines 8--11). We repeat these steps until there is at most one non-0/1 value in $\check{y}$. The details of the \textsc{TransformWithOneCluster} and \textsc{TransformWithTwoClusters} functions are shown in Algorithm~\ref{alg:convertingfunctions}.


\SetKwComment{Comment}{// }{}
\begin{algorithm}[t]
    \SetKwInput{Input}{Input}
    \SetKwInOut{Output}{Output}
    \Input{Optimal solution $y^*$ for the LP problem, similarity matrix $W$}
    \Output{Transformed optimal solution $\check{y}$ where each value is one of \{0, $\alpha$, 1\}}
    \SetKwFunction{TFWO}{\textnormal{\textsc{TransformWithOneCluster}}}
    \SetKwFunction{TFWT}{\textnormal{\textsc{TransformWithTwoClusters}}}
    \SetKwProg{Fn}{Function}{:}{}
    
    $\check{y}$ = $y^*$\;
    \Comment{There are T unique non-0/1 values in $y^*$}
    $T$ = \textsc{GetNumClusters}($y^*$)\;
    \While{$T > 1$}{
        $ZeroClusterList$ = \textsc{GetZeroClusters}($\check{y}$)\;
        \For{$\alpha$ in $ZeroClusterList$}{
            \Comment{Details are in Algorithm~\ref{alg:convertingfunctions}}
            \textsc{TransformWithOneCluster}($\check{y}$, $\alpha$, $W$)\;
            $T$ = $T$ - 1\;
        }
        \If {$T > 1$}{
        $\alpha$, $\beta$ = \textsc{GetNonZeroTwoClusters}($\check{y}$)\;
        \Comment{Details are in Algorithm~\ref{alg:convertingfunctions}}
        \textsc{TransformWithTwoClusters}($\check{y}$, $\alpha$, $\beta$, $W$)\;
        $T$ = \textsc{GetNumClusters}($\check{y}$)\;
        }
    }
    \Comment{There is at most one unique non-0/1 value in $\check{y}$}
    {\bf return $\check{y}$}\;
    \caption{\systems{}'s LP solution conversion algorithm.}
    \label{alg:converting}
\end{algorithm}

\SetKwComment{Comment}{// }{}
\begin{algorithm}[t]
    \SetKwFunction{TFWO}{\textnormal{\textsc{TransformWithOneCluster}}}
    \SetKwFunction{TFWT}{\textnormal{\textsc{TransformWithTwoClusters}}}
    \SetKwProg{Fn}{Function}{:}{}
    
    
    \Fn{\TFWO{$y$, $\alpha$, $W$}}{
    
    \Comment{\rev{Replace all $\alpha$ with $\alpha_{new}$ to reduce one unique non-0/1 value}}
        
    $a_k$, $a_{k+1}$, $S_\alpha$ = \textsc{GetOneClusterInfo}($y$, $\alpha$, $W$)\;

    \If {$S_\alpha \leq 0$}{
    $\alpha \leftarrow{} a_{k+1}$\;}
    \Else{
    $\alpha \leftarrow{} a_{k}$\;}
    \Comment{Now $\alpha \in \{a_k, a_{k+1}\}$}
    
    }
    \Fn{\TFWT{$y$, $\alpha$, $\beta$, $W$}}{
    \Comment{\rev{Replace all ($\alpha$, $\beta$) with ($\alpha_{new}$, $\beta_{new}$) using one of the five cases in Lemma~\ref{lem:twocluster} to reduce at least one unique non-0/1 value}}

    $a_k$, $a_{k+1}$, $N_\alpha$, $S_\alpha$, $b_l$, $b_{l+1}$, $N_\beta$, $S_\beta$, $E$ = \textsc{GetTwoClustersInfo}($y$, $\alpha$, $\beta$, $W$)\;

    $X$, $Y$ = $\frac{N_\alpha}{N_\beta}$, $\frac{(S_\alpha-E)N_\beta-(S_\beta+E)N_\alpha}{N_\beta}$\;
    
    \Comment{\rev{Details are in our technical report~\cite{iflippertr}}}
    
    \If {$X<0$, $Y\leq0$}{
    \If {$X+1 \leq 0$}{$\alpha \leftarrow{} \alpha+min(a_{k+1}-\alpha, -\frac{N_\beta}{N_\alpha}(b_{l+1}-\beta))$\;}
    \Else{$\alpha \leftarrow{} \alpha+min(a_{k+1}-\alpha, -\frac{N_\beta}{N_\alpha}(b_{l+1}-\beta), \frac{N_\beta(\beta-\alpha)}{N_\alpha+N_\beta})\;$
    }
    $\beta \leftarrow{} \beta-\frac{N_\alpha}{N_\beta}\epsilon_\alpha$\;
    }
    \ElseIf{$X<0$, $Y>0$}{
    \If {$X+1 \geq 0$}{$\alpha \leftarrow{} \alpha-min(\alpha-a_{k}, -\frac{N_\beta}{N_\alpha}(\beta-b_l))$\;
    }
    \Else{$\alpha \leftarrow{} \alpha-min(\alpha-a_{k}, -\frac{N_\beta}{N_\alpha}(\beta-b_l), -\frac{N_\beta(\beta-\alpha)}{N_\alpha+N_\beta})$\;}
$\beta \leftarrow{} \beta+\frac{N_\alpha}{N_\beta}\epsilon_\alpha$\;
    }
    \ElseIf{$X>0$, $Y\leq0$}{
        $\alpha \leftarrow{} \alpha+min(a_{k+1}-\alpha, \frac{N_\beta}{N_\alpha}(\beta-b_l), \frac{N_\beta(\beta-\alpha)}{N_\alpha+N_\beta})$\;
        $\beta \leftarrow{} \beta-\frac{N_\alpha}{N_\beta}\epsilon_\alpha$\;
    }
    \ElseIf{$X>0$, $Y>0$}{
        $\alpha \leftarrow{} \alpha-min(\alpha-a_{k}, \frac{N_\beta}{N_\alpha}(b_{l+1}-\beta))$\;
        $\beta \leftarrow{} \beta+\frac{N_\alpha}{N_\beta}\epsilon_\alpha$\;
    }
    \Comment{Now $(\alpha, \beta)$ is one of the cases in Lemma~\ref{lem:twocluster} } 
    }
    \caption{Transformation functions in Algorithm~\ref{alg:converting}.}
    \label{alg:convertingfunctions}
\end{algorithm}

\begin{example}
To illustrate Algorithm~\ref{alg:converting}, consider Figure~\ref{fig:rounding}a again where $m = 2$. Say we obtain an optimal solution from an LP solver as shown in Figure~\ref{fig:convertingexample}a where nodes 1 and 4 have the labels $\alpha = 0.1$ and $\beta = 0.9$, respectively, while the other nodes have 0 labels. This solution uses one flip and has a total error of 2. We first construct an $\alpha$-cluster and a $\beta$-cluster where each cluster only contains a single node, and there is no edge between them as shown in Figure~\ref{fig:convertingexample}b. We then convert ($\alpha$, $\beta$) using the \textsc{TransformWithTwoClusters}. 
As a result, we set ($\alpha$, $\beta$) to (0.5, 0.5) and reduce one non-0/1 unique value in the solution. The solution now has only one non-0/1 value (i.e., 0.5) and still has a total error of 2 while using one flip.
\end{example}

\begin{figure}[ht]
\centering
  \vspace{-0.4cm}
  \includegraphics[width=\columnwidth]{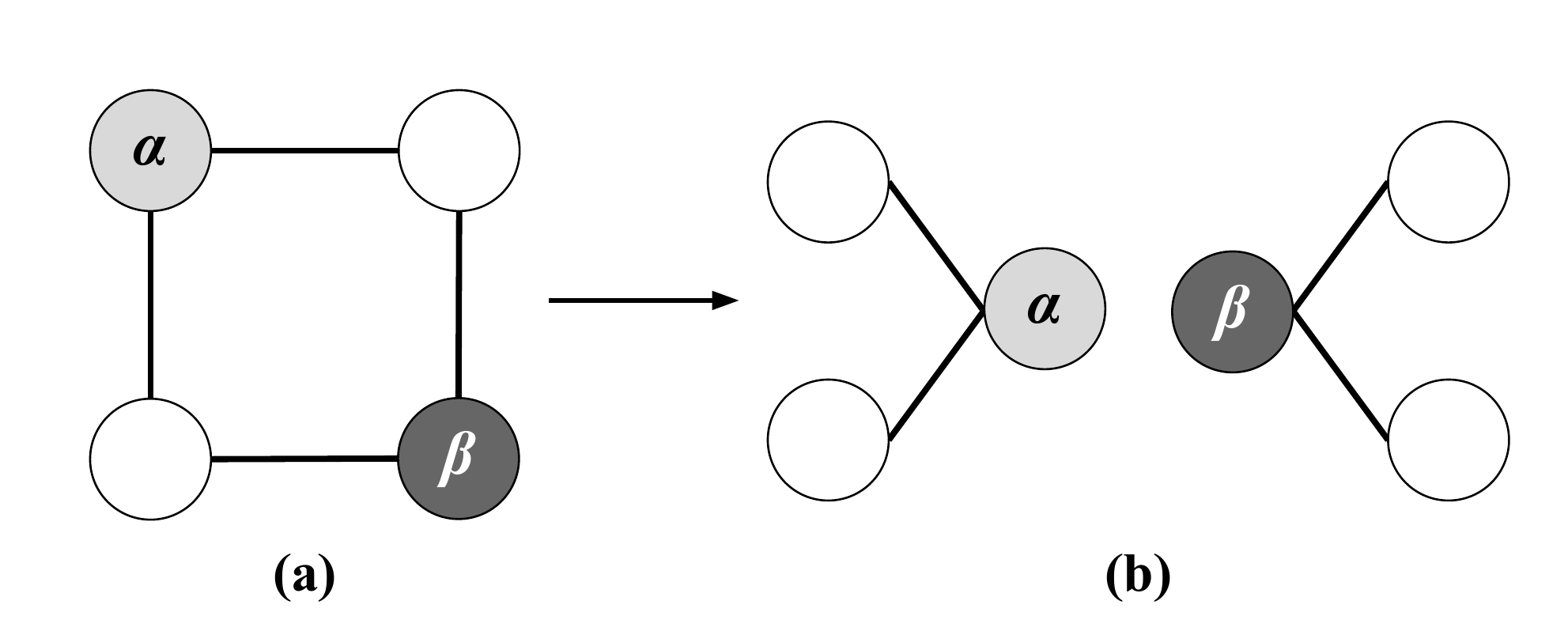}
  \vspace{-0.6cm}
  \caption{Cluster construction example for Algorithm~\ref{alg:converting}.}
  \label{fig:convertingexample}
\end{figure}

\stitle{Adaptive Rounding into a Binary Solution.} We now convert $\check{y}$ into a feasible integer solution with only 0's and 1's. If we use simple rounding to change $\alpha$ to 0 or 1 as in the na\"ive method in Section~\ref{sec:linearrelaxation}, the resulting integer solution may not be feasible like the example in Figure~\ref{fig:rounding}c. Fortunately, the fact that we only have three possible values allows us to propose an adaptive rounding algorithm (Algorithm~\ref{alg:rounding}) that guarantees feasibility. We first denote $M_{ab}$ as the sum of similarity values of ($a$-$b$) label pairs with edges.
For example, $M_{0\alpha}$ is the sum of similarities of similar node pairs whose labels are 0 and $\alpha$, respectively. Then the key idea is to round $\alpha$ to 1 if the (1-$\alpha$) label pairs have a higher similarity sum than that of the (0-$\alpha$) label pairs (i.e., $M_{0\alpha} \leq M_{1\alpha}$), and round $\alpha$ to 0 otherwise. For example, consider Figure~\ref{fig:rounding}b again. Here $\alpha=0.5$, $M_{0\alpha}=4$, and $M_{1\alpha}=0$. We have $M_{0\alpha} > M_{1\alpha}$, so rounding $\alpha$ to 0 will result in a feasible solution with zero violations.


\begin{algorithm}[t]
    \SetKwInput{Input}{Input}
    \SetKwInOut{Output}{Output}
    \Input{Transformed optimal solution $\check{y}$ for the LP problem whose values are in \{0, $\alpha$, 1\}, similarity matrix $W$}
    \Output{Feasible binary integer solution $\overline{y}$}
    \Comment{Get the total similarity values of ($0$-$\alpha$) and ($1$-$\alpha$) label pairs}
    $M_{0\alpha}, M_{1\alpha}$ = \textsc{GetSimilaritySumOfPairs}($\check{y}$, $W$)\;
    \If{$M_{0\alpha} \leq M_{1\alpha}$}{
        $\overline{\alpha} \leftarrow{} 1$\;
    }
    \Else{
        $\overline{\alpha} \leftarrow{} 0$\;
    }

    
    \Comment{Replace $\alpha$ with $\overline{\alpha} \in \{0,1\}$}
    $\overline{y}$ = \textsc{ApplyRounding}($\check{y}, \overline{\alpha}$)\;
    {\bf return $\overline{y}$}\;
    
    \caption{\systems{}'s adaptive rounding algorithm.}
    \label{alg:rounding}
\end{algorithm}

We prove the correctness of Algorithm~\ref{alg:rounding} in Lemma~\ref{lem:rounding}.

\begin{lemma} \label{lem:rounding}
Given an optimal solution $\check{y}$ of Equation~\ref{equ:mlp} where each value is one of \{0, $\alpha$, 1\}, applying adaptive rounding (Algorithm~\ref{alg:rounding}) always results in a feasible solution.
\end{lemma}

\begin{proof}
The total error in $\check{y}$ can be expressed as a function of $\alpha$:
\begin{equation}\label{equ:optimalviolation}
\begin{split}
    \text{Total Error} &= \sum_{i=1}^{n} \sum_{j=1}^{n} W_{ij}\check{z}_{ij} = \sum_{i=1}^{n} \sum_{j=1}^{n} W_{ij}|\check{y}_i-\check{y}_j|\\
    &= f(\alpha) = M_{01}+\alpha M_{0\alpha}+(1-\alpha) M_{1\alpha} \leq m \\
\end{split}
\end{equation}

In addition, the error constraint in Equation~\ref{equ:mlp} is satisfied because $\check{y}$ is an optimal solution.

We now want to round $\alpha$ to either 0 or 1 while satisfying the error constraint. If $M_{0\alpha} \leq M_{1\alpha}$, we have $f(1) = M_{01}+M_{0\alpha} \leq M_{01}+\alpha M_{0\alpha}+(1-\alpha) M_{1\alpha} \leq m$ $(\because 0<\alpha<1)$, so setting $\alpha = 1$ guarantees a feasible solution. On the other hand, if $M_{1\alpha} < M_{0\alpha}$, we have $f(0) = M_{01}+M_{1\alpha} < M_{01}+\alpha M_{0\alpha}+(1-\alpha) M_{1\alpha} \leq m$ $(\because 0<\alpha<1)$, so setting $\alpha = 0$ ensures feasibility.
\end{proof}


Lemma~\ref{lem:rounding} shows that the rounding choice of $\alpha$ depends on how $M_{0\alpha}$ and $M_{1\alpha}$ compare instead of $\alpha$'s value itself. This result explains why simple rounding as in the na\"ive method does not necessarily lead to a feasible solution. 


In the next section, we analyze how accurate the rounded solution is compared to the optimal solution of the ILP problem and investigate whether it can be further improved.

\subsection{Optimality and Improvement}\label{sec:optimality}

We analyze the optimality of Algorithm~\ref{alg:rounding} and propose a technique called {\em reverse greedy} for further optimizing it without exceeding the total error limit.

\stitle{Optimality Analysis.} We prove the theoretical bounds of Algorithm~\ref{alg:rounding} in Lemma~\ref{lem:optimalbound}:

\begin{lemma} \label{lem:optimalbound}
For a given optimal solution $\check{y}$ of Equation~\ref{equ:mlp}, let us denote $N_{ab}$ as the number of nodes in $\check{y}$ where the label $a$ was flipped to $b$. Then the objective value of the output $\overline{y}$ from Algorithm~\ref{alg:rounding} is at most $C$ more than the optimal objective value of the original ILP problem where the value of $C$ depends on $\alpha$:
\begin{equation}\label{equ:test}
\begin{split}
    &\alpha=1 \rightarrow{} C=(1-\alpha)(N_{0\alpha}-N_{1\alpha})\\
    &\alpha=0 \rightarrow{} C=\alpha(N_{1\alpha}-N_{0\alpha})\\
\end{split}
\end{equation}

\end{lemma}
\begin{proof}

Since the optimal solution of the ILP problem is always one of the feasible solutions of the LP problem, the optimal objective values ($OPT$) of the two optimization problems should satisfy:
\begin{equation}\label{equ:optimalityrelation}
\begin{split}
    OPT_{LP} \leq  OPT_{ILP}\\
\end{split}
\end{equation}

The objective value of $\check{y}$ can then be expressed as follows:
\begin{equation}\label{equ:optimalflip}
\begin{split}
    OPT_{LP}(\check{y}) &= \sum_{i=1}^{n} \check{z}_{ij} = (N_{01}+N_{10})+\alpha N_{0\alpha}+(1-\alpha) N_{1\alpha} \\
\end{split}
\end{equation}

We first consider the case where $\alpha=1$. Here the objective value of $\overline{y}$ is $OPT_{LP}(\overline{y})=(N_{01}+N_{10})+N_{0\alpha}$. We can then derive the bound on $OPT_{LP}(\overline{y})$ from Equation~\ref{equ:optimalityrelation} and Equation~\ref{equ:optimalflip} as follows:
\begin{equation}\label{equ:optimalitybound}
\begin{split}
    &OPT_{LP}(\check{y})-OPT_{LP}(\overline{y}) \leq  OPT_{ILP}-OPT_{LP}(\overline{y}) \\
    &(1-\alpha)(N_{1\alpha}-N_{0\alpha}) \leq OPT_{ILP} - OPT_{LP}(\overline{y})\\
    &OPT_{LP}(\overline{y}) \leq OPT_{ILP} + (1-\alpha)(N_{0\alpha}-N_{1\alpha})\\
\end{split}
\end{equation}

We note that $(1-\alpha)(N_{0\alpha}-N_{1\alpha})$ is always non-negative because $\check{y}$ is an optimal solution, i.e., $OPT_{LP}(\check{y}) \leq OPT_{LP}(\overline{y})$. 

Similarly, if $\alpha=0$, we can obtain the objective value bound of $\alpha(N_{1\alpha}-N_{0\alpha})$.
\end{proof}

\stitle{Reverse Greedy Algorithm.} Lemma~\ref{lem:optimalbound} shows that the bound may sometimes be loose because it depends on $\alpha$ and the number of nodes whose labels are flipped to $\alpha$. There is a possibility that Algorithm~\ref{alg:rounding} flips nodes unnecessarily, which results in a total error smaller than $m$ (see Section~\ref{sec:ablationstudy} for empirical results). Hence, we propose an algorithm called {\em reverse greedy} (Algorithm~\ref{alg:reversegreedy}), which unflips flipped labels as long as the total error does not exceed $m$. As the name suggests, we run the greedy algorithm in Section~\ref{sec:opt} in the reverse direction where for each iteration, we unflip nodes that increase the error the least to recover the original labels as much as possible. Thus, reverse greedy can only improve the optimality of Algorithm~\ref{alg:rounding}.

\begin{algorithm}[t]
    \SetKwInput{Input}{Input}
    \SetKwInOut{Output}{Output}
    \Input{Feasible binary integer solution $\overline{x}$, total error limit $m$}
    \Output{Improved binary integer solution $\tilde{x}$}
    
    $\tilde{x}$ = $\overline{x}$\;
    $flipped\_labels$ = \textsc{GetFlippedLabel}($\tilde{x}$)\;
    $total\_error$ = \textsc{GetTotalError}($\tilde{x}$)\;
    \While{$total\_error \leq m$}{
        \Comment{Unflip flipped labels of the nodes so that the total error is increased the least}
        $\tilde{x}$ = \textsc{FlipLabelLeastViolation}($\tilde{x}$, $flipped\_labels$)\;
        $flipped\_labels$ = \textsc{GetFlippedLabel}($\tilde{x}$)\;
        $total\_error$ = \textsc{GetTotalError}($\tilde{x}$)\;
    }
    
    {\bf return $\tilde{x}$}\;
    
    \caption{\systems{}'s reverse greedy algorithm.}
    \label{alg:reversegreedy}
\end{algorithm}

\begin{example}
Continuing from our example using Figure~\ref{fig:rounding} and $m = 2$, suppose we run Algorithm~\ref{alg:rounding} on Figure~\ref{fig:rounding}b and round the 0.5 labels to 0 to obtain the feasible solution in Figure~\ref{fig:reversegreedy}a. A total of two flips are performed compared to the initial graph Figure~\ref{fig:rounding}a. However, there exists an optimal solution like Figure~\ref{fig:reversegreedy}b where only one flip is necessary. Running Algorithm~\ref{alg:reversegreedy} will unflip nodes 1 or 4, and unflipping node 1 produces this solution. 
\end{example}

\begin{figure}[ht]
\centering
  \vspace{-0.5cm}
  \includegraphics[width=0.7\columnwidth]{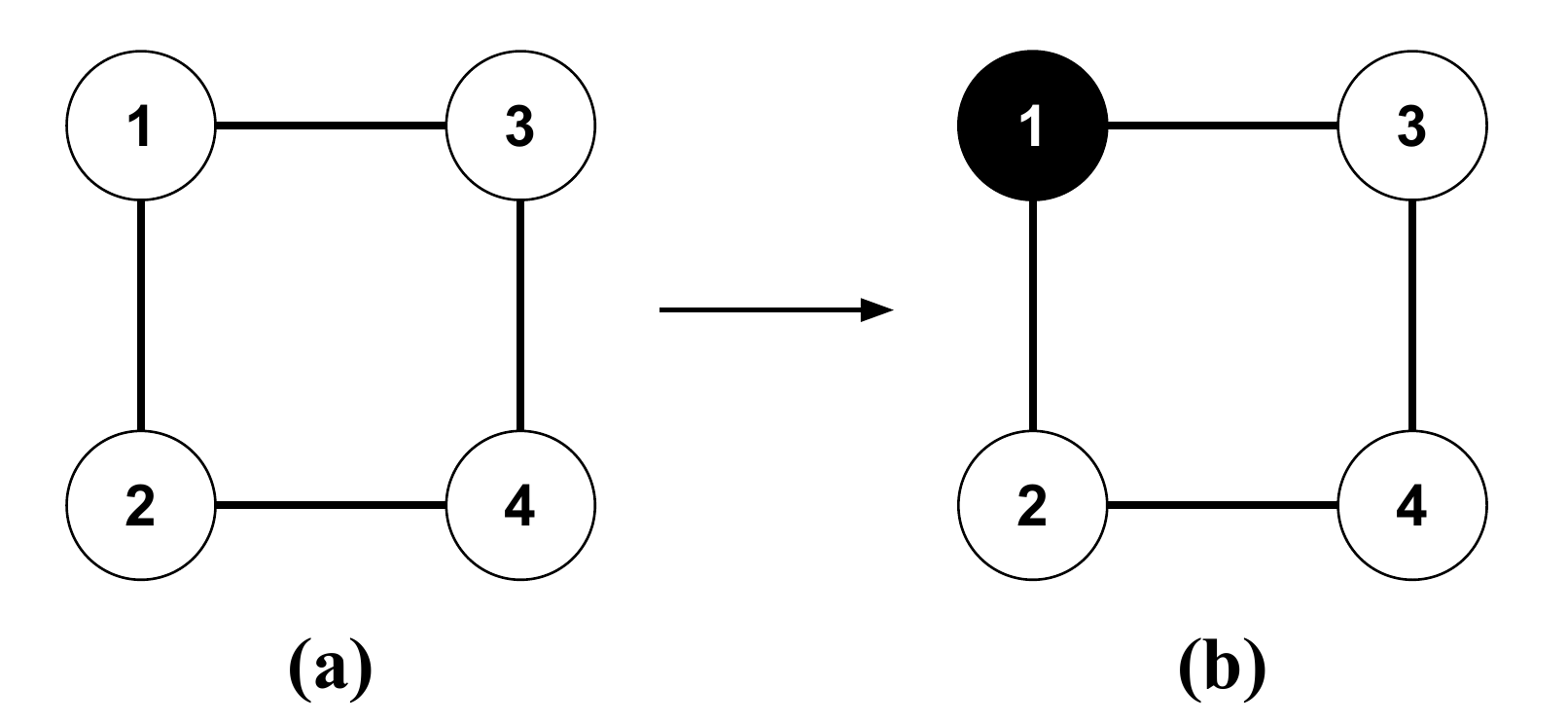}
  \vspace{-0.4cm}
  \caption{The reverse greedy algorithm can further optimize a rounded solution from Algorithm~\ref{alg:rounding}.}
  \vspace{-0.3cm}
  \label{fig:reversegreedy}
\end{figure}

The time complexity of reverse greedy is the same as the greedy algorithm, i.e., $O(n^2)$, but in practice requires fewer iterations than greedy because the total error is relatively close to $m$.

\subsection{Putting Everything Together}\label{sec:systems}
We now describe the overall workflow of \systems{}. For a given ILP problem, we first convert it into an approximate LP problem. We then solve the approximate LP problem and convert its optimal solution to another optimal solution that only has values in \{0, $\alpha$, 1\} using Algorithm~\ref{alg:converting}. We then apply adaptive rounding using Algorithm~\ref{alg:rounding}. Finally, we possibly improve the rounded solution by unflipping labels using Algorithm~\ref{alg:reversegreedy}.


\stitle{Complexity Analysis.}
\systems{} consists of four parts: solving the approximate LP problem, the converting algorithm, the adaptive rounding algorithm, and the reverse-greedy algorithm. The interior-points algorithm for solving LP has a time complexity of $O(n^{2+\varepsilon})$, where $\varepsilon$ is smaller than 1. The converting algorithm has a time complexity of $O(n^2)$: in the worst case, we may need to combine clusters $n$ times, and for each combining process, we need to check all neighbors of a particular cluster, which is up to $n$ points. For the adaptive rounding algorithm, the time complexity is also $O(n^2)$, because we need to count the number of $(1,\alpha)$ edges and $(0,\alpha)$ edges, which is $O(n^2)$ in the worst case. The reverse greedy algorithm, as analyzed in Section~\ref{sec:optimality}, is $O(n^2)$. Overall, the complexity is bounded by the LP solver, which is $O(n^{2+\varepsilon})$. This result may look worse than the greedy algorithm, but \systems{} runs much faster than greedy in practice. The reason is that we use efficient LP solvers, and the empirical running times of the adaptive rounding and reverse greedy algorithms are much faster than their theoretical worst case complexities. We show the empirical runtimes in Section~\ref{sec:exp}.

\section{Experiments} \label{sec:exp}
In this section, we evaluate \systems{} on real datasets and address the following key questions:
\begin{itemize}
    \item Is there an accuracy-fairness trade-off for \systems{}?
    \item How does \systems{} compare with various baselines in terms of model accuracy, individual fairness, and efficiency?
    \item How useful is each component of \systems{}?
    \item Does the LP solution conversion run correctly?
    \item Can \systems{} be integrated with in-processing techniques?
\end{itemize}
We implement \systems{} in Python and use Scikit-learn~\cite{scikit-learn} and PyTorch~\cite{NEURIPS2019_9015} libraries for model training. For performing optimization, we use two software packages: CPLEX~\cite{cplex2009v12} and MOSEK~\cite{mosek}. We run all experiments on Intel Xeon Gold 5115 CPUs.

\subsection{Setting}
\label{sec:experimentalsetting}
\paragraph*{Datasets} We experiment on the following three popular datasets in the fairness literature. We randomly split each dataset into training, test, and validation sets. We then use the same feature pre-processing as in IBM's AI Fairness 360 toolkit~\cite{DBLP:journals/corr/abs-1810-01943} where examples with missing values are discarded. See Table~\ref{tbl:datasets} for more details.

\begin{itemize}
    \item COMPAS~\cite{machinebias}: Contains 6,167 people examples and is used to predict criminal recidivism rates. The features include gender, race, and criminal history.
    \item AdultCensus~\cite{DBLP:conf/kdd/Kohavi96}: Contains 45,222 people examples and is used to predict whether an individual's income exceeds \$50K per year. The features include gender, age, and occupation.
    \item Credit~\cite{Dua:2019}: Contains 1,000 examples and is used to predict an individual's credit risk. The features contain race, credit amount, and credit history.
    \item \rev{Synthetic~\cite{pmlr-v54-zafar17a}: We generate 200,000 samples with two attributes ($x_1$, $x_2$) and a binary label $y$ following a process introduced in~\cite{pmlr-v54-zafar17a}. Tuples with a positive label ($x_1$, $x_2$, $y=1$) are drawn from a Gaussian distribution and, tuples with a negative label ($x_1$, $x_2$, $y=0$) are drawn from another Gaussian distribution.}
\end{itemize}

We also considered two additional fairness datasets: Bank~\cite{bankdataset} and LSAC~\cite{lsacdataset}. However, these datasets exhibit fewer individual fairness issues as shown in \ifthenelse{\boolean{techreport}}{Section~\ref{sec:banklawschool}}{our technical report~\cite{iflippertr}}. Hence, we do not include them in our main experiments because the fairness improvements are not as significant.


\paragraph*{Similarity Matrices} We consider two similarity matrices using Euclidean distance (i.e., $d(x_i, x_j)=\| x_{i}-x_{j}\|^2 $). When measuring distances, we follow previous approaches~\cite{ifair,pmlr-v28-zemel13} and exclude sensitive attributes like gender and race. The rationale is that individuals who have similar information other than the sensitive attributes values must be treated similarly. We quantify the similarity as $W_{ij} = e^{-\theta d(x_i, x_j)}$ where $\theta>0$ is a scale parameter. Our design follows a related work on individual fairness that also uses similarity matrices~\cite{petersen2021post}. See Table~\ref{tbl:datasets} for the configurations. 


\begin{itemize}
    \item {\em kNN-based}: Considers $(x_i, x_j)$ as a similar pair if $x_i$ is one of $x_j$'s nearest neighbors or vice versa:
    \[
    W_{ij} = 
    \begin{cases}
      e^{-\theta d(x_i, x_j)} & \text{if $x_i \in NN_{k}(x_j)$ or $x_j \in NN_{k}(x_i)$}\\
      0 & \text{otherwise}
    \end{cases}
    \]
    where $NN_{k}(x_j)$ denotes the set of $k$ nearest examples of $x_j$ in a Euclidean space.
    \item {\em threshold-based}: Considers $(x_i, x_j)$ as a similar pair if their distance is smaller than a threshold $T$:
    \[
    W_{ij} = 
    \begin{cases}
      e^{-\theta d(x_i, x_j)} & \text{if $d(x_i, x_j) \leq T $}\\
      0 & \text{otherwise}
    \end{cases}
    \]
\end{itemize}

We use the FALCONN LSH library~\cite{Andoni2015falconn} to efficiently construct both similarity matrices. For all datasets, we set the LSH hyperparameters (number of hash functions and number of hash tables) to achieve a success probability of 0.98 on a validation set where 98\% of similar pairs of examples are in the same buckets.

\begin{table}[t]
  \setlength{\tabcolsep}{3pt}
  \centering
  \begin{tabular}{cccccccc}
    \toprule
    {\bf Dataset} & {\bf \# Train} & {\bf \# Test}& {\bf \# Valid} & {\bf Sen. Attr} & {\bf $k$} & {\bf $T$} & {\bf $\theta$}\\
    \midrule
    COMPAS & 3,700 & 1,850 & 617 & gender & 20 & 3 & 0.05\\
    AdultCensus & 27,133 & 13,566 & 4,523 & gender & 20 & 3 & 0.1\\
    Credit & 700 & 201 & 199 & age & 20 & 7 & 0.05\\
    \rev{Synthetic} & \rev{100,000} & \rev{60,000} & \rev{40,000} & \rev{N/A}\footnotemark & \rev{20} & \rev{3} & \rev{0.05}  \\
    \bottomrule
  \end{tabular}
  \caption{Settings for the four datasets.}
  \label{tbl:datasets}
  \vspace{-0.4cm}
\end{table}
\addtocounter{footnote}{0}
\footnotetext{\rev{There is no sensitive attribute to exclude when measuring distances.}}

\ifthenelse{\boolean{revision}}{
\rev{There are two important choices to make for constructing the similarity matrix: the similarity/distance function (e.g., Euclidean) and similarity/distance thresholds (e.g., $k$ and $T$). The similarity function depends on the user’s domain knowledge on what similarity means when defining individual fairness. For example, if the user thinks each feature contributes the same to individual fairness, then Euclidean distance would be a good choice, otherwise the users can consider distances like Mahalanobis distance. The thresholds can be tuned based on the following trade-off: a more relaxed threshold makes the similarity matrix denser, which slows down the optimizations as more edges need to be considered. Hence, the goal is to keep enough information about the similarities, but ignore enough edges with small weights to make \systems{} run fast enough. 
}
}{}

\paragraph*{Measures} 
We evaluate a model's accuracy by computing the portion of correctly-predicted data points. We evaluate individual fairness using the consistency score~\cite{pfr2019} defined in Section~\ref{sec:preliminaries}, which measures the consistency of predictions on the test set between similar individuals. 
For both scores, higher values are better. We report mean values over five trials for all measures.

\paragraph*{Hyperparameter Tuning}


\systems{} provides a total error limit hyperparameter $m$, which impacts the trade-off between the model accuracy and fairness where a lower $m$ results in better fairness, but worse accuracy (see Section~\ref{sec:tradeoff}). Given a desired consistency score, we can use binary search to adjust $m$ by comparing the consistency score of the trained model with the desired score.


\paragraph*{Methods Compared} We compare \systems{} with the following existing pre-processing algorithms for individual fairness:
\begin{itemize}
    \item {\em LFR}~\cite{pmlr-v28-zemel13}: A fair representation learning algorithm that optimizes between accuracy, group fairness, and individual fairness in terms of data reconstruction loss.
    \item {\em iFair}~\cite{ifair}: A fair representation learning algorithm that optimizes accuracy and individual fairness. Compared to LFR, iFair does not optimize group fairness, which enables the classifier to have better individual fairness. If two data points are close to each other in the input space, iFair aims to map them close to each other in the feature space as well.
    \item {\em PFR}~\cite{pfr2019}: A fair representation learning algorithm that tries to learn an embedding of the fairness graph. The fairness graph used in PFR is similar to the one used in \systems{}. Similar to iFair, PFR optimizes individual fairness while preserving the original data by mapping similar individuals to nearby points in the learned representation. Among the three baselines, only PFR is able to support a general similarity matrix. PFR uses an efficient trace optimization algorithm, which can learn representations much faster than iFair.
    

\end{itemize}

For all baselines, there are multiple hyperparameters that can balance the competing objectives. We start with the same sets of hyperparameters as described in their papers, and tune them to provide the best results. 

\paragraph*{Optimization Solutions Compared.} We compare \systems{} with the following optimization baselines described in Sections~\ref{sec:naivesolutions} and~\ref{sec:milp}.
\begin{itemize}
    \item {\em Greedy}: Flips labels that reduce the total error the most.
    \item {\em Gradient}: Solves an unconstrained optimization problem.
    \item {\em kMeans}: Applies k-means clustering and, for each cluster, make its examples have the majority label.
    \item {\em ILP Solver}: Solves the ILP problem exactly using CPLEX~\cite{cplex2009v12}, which is a state-of-the-art solver.
\end{itemize}

For \gradient{} and \kmeans{}, we tune $\lambda$ and $k$, respectively, to satisfy the total error limit. We use CPLEX to solve ILP problems as it is faster than MOSEK. When solving LP problems in \systems{}, however, CPLEX turns out to be much slower than MOSEK due to its long construction time, so we use MOSEK instead.


An interesting behavior of MOSEK is that empirically all of its optimal solutions only contain values in $\{0, \alpha, 1\}$ without having to run our Algorithm~\ref{alg:converting}. Note that this behavior is limited to the label flipping problems we are solving and does not necessarily occur for other LP problems. We suspect that MOSEK's algorithm effectively contains the functionality of Algorithm~\ref{alg:converting}. We also make the same observations when using CPLEX for our LP problems. We find these results encouraging because it means that Algorithm~\ref{alg:converting} is not a significant runtime overhead and can even be tightly integrated with the LP solving code. A detailed analysis of MOSEK's code is an interesting future work. For any other solver that does not have this behavior, one can always run Algorithm~\ref{alg:converting} on its solutions. Hence to make sure our Algorithm~\ref{alg:converting} is correct, we separately evaluate it in Section~\ref{sec:conversioneval}.


\paragraph*{Model Setup}
We use logistic regression (LR), random forest (RF), and neural network (NN) models for our experiments. We tune the model hyperparameters (e.g., learning rate and maximum depth of the tree) such that the model has the highest validation accuracy. Our goal is to pre-process the training data to make these models train fairly.


\subsection{Accuracy and Fairness Trade-off}
\label{sec:accuracyfairnesstradeoff}

Figure~\ref{fig:tradeoffcompas} shows the trade-off between fairness and accuracy for the COMPAS dataset when running \systems{} with different allowed amounts of total error. The x-axis is accuracy, and the y-axis is the consistency score. There are two curves, which correspond to two different similarity matrices: kNN-based and threshold-based. Each data point on the curve has two other results: (total error after repairing, number of flips). As we flip more labels, there is less total error in the repaired dataset, and the resulting model has a higher consistency score on the test set, which means it has better individual fairness. However, the accuracy drops as we flip more labels. The trade-off curves of the other datasets are similar and \ifthenelse{\boolean{techreport}}{shown in Section~\ref{sec:tradeoffotherdatasets}}{can be found in our technical report~\cite{iflippertr}.}

\vspace{-0.35cm}
\begin{figure}[ht]
\centering
  \includegraphics[width=0.75\columnwidth]{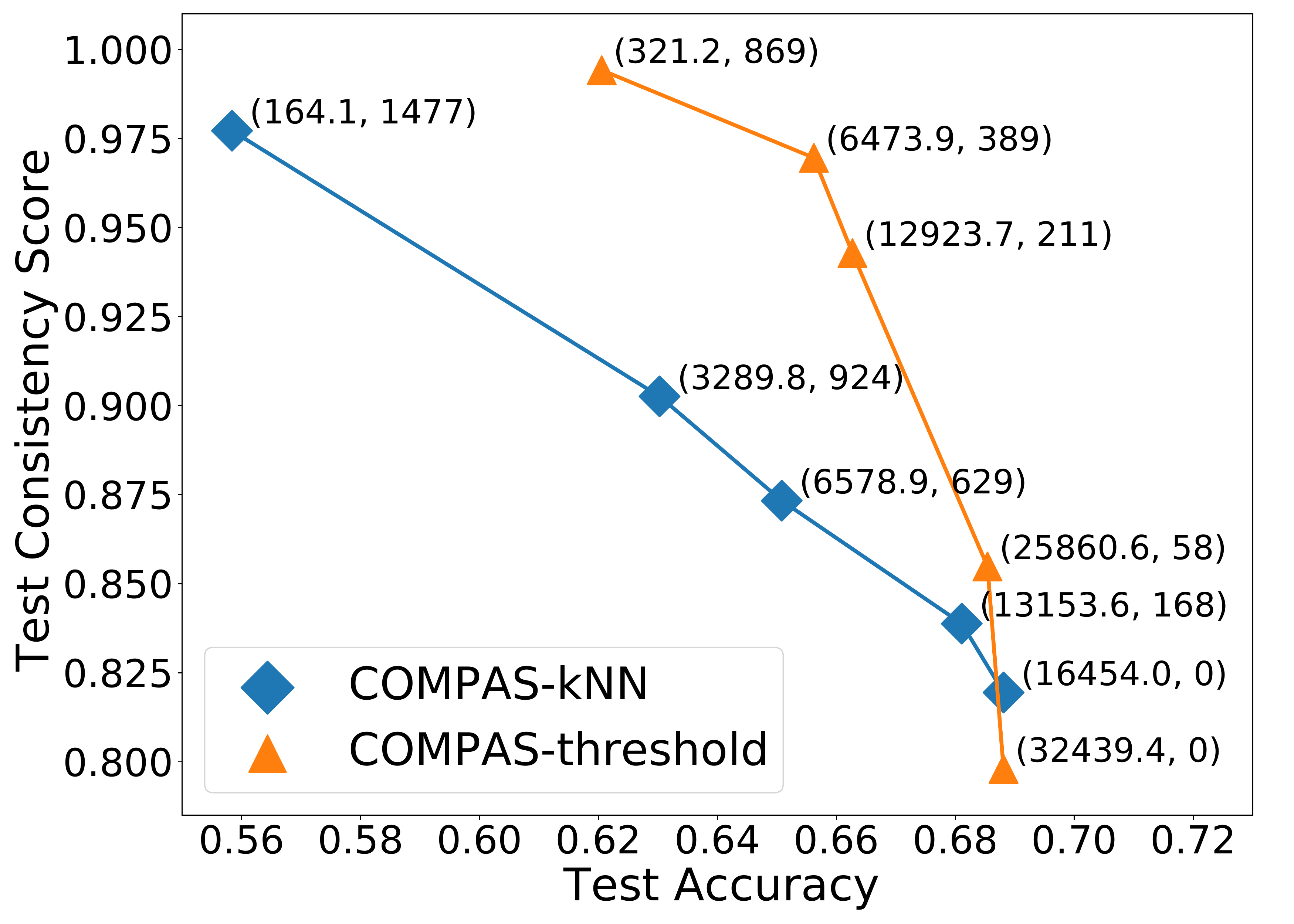}
  \vspace{-0.4cm}
  \caption{Accuracy-fairness trade-off curves for \systems{}. For each data point, we also show the total error after repairing and the number of flips in parentheses.}
  \label{fig:tradeoffcompas}
   \vspace{-0.6cm}
\end{figure}

\begin{figure*}[t]
  \centering
  \begin{subfigure}{0.33\textwidth}
     \includegraphics[width=\columnwidth]{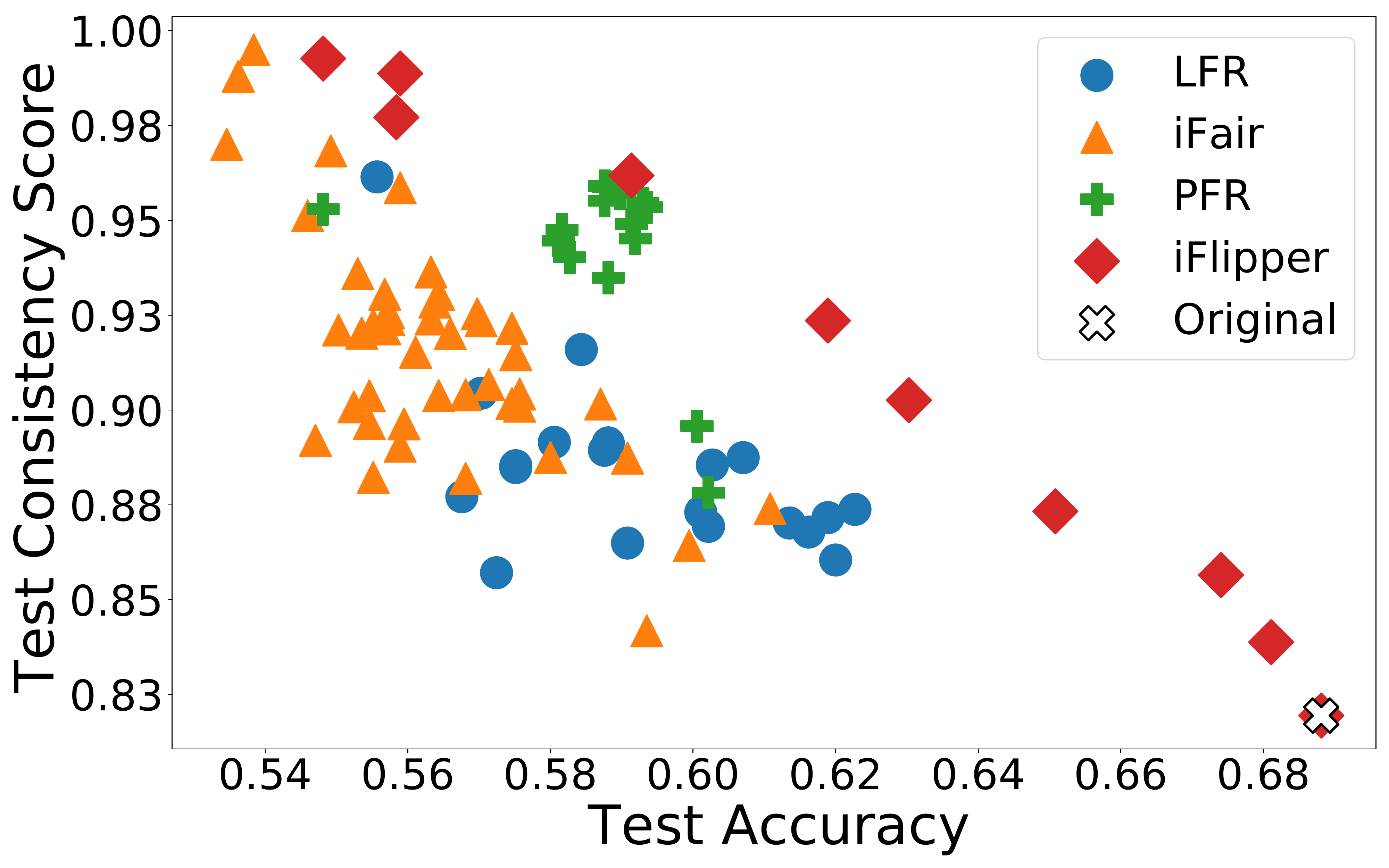}
     \vspace{-0.6cm}
     \caption{{\sf COMPAS-kNN}}
     \label{fig:COMPAS-kNN}
  \end{subfigure}
  \begin{subfigure}{0.33\textwidth}
     \includegraphics[width=\columnwidth]{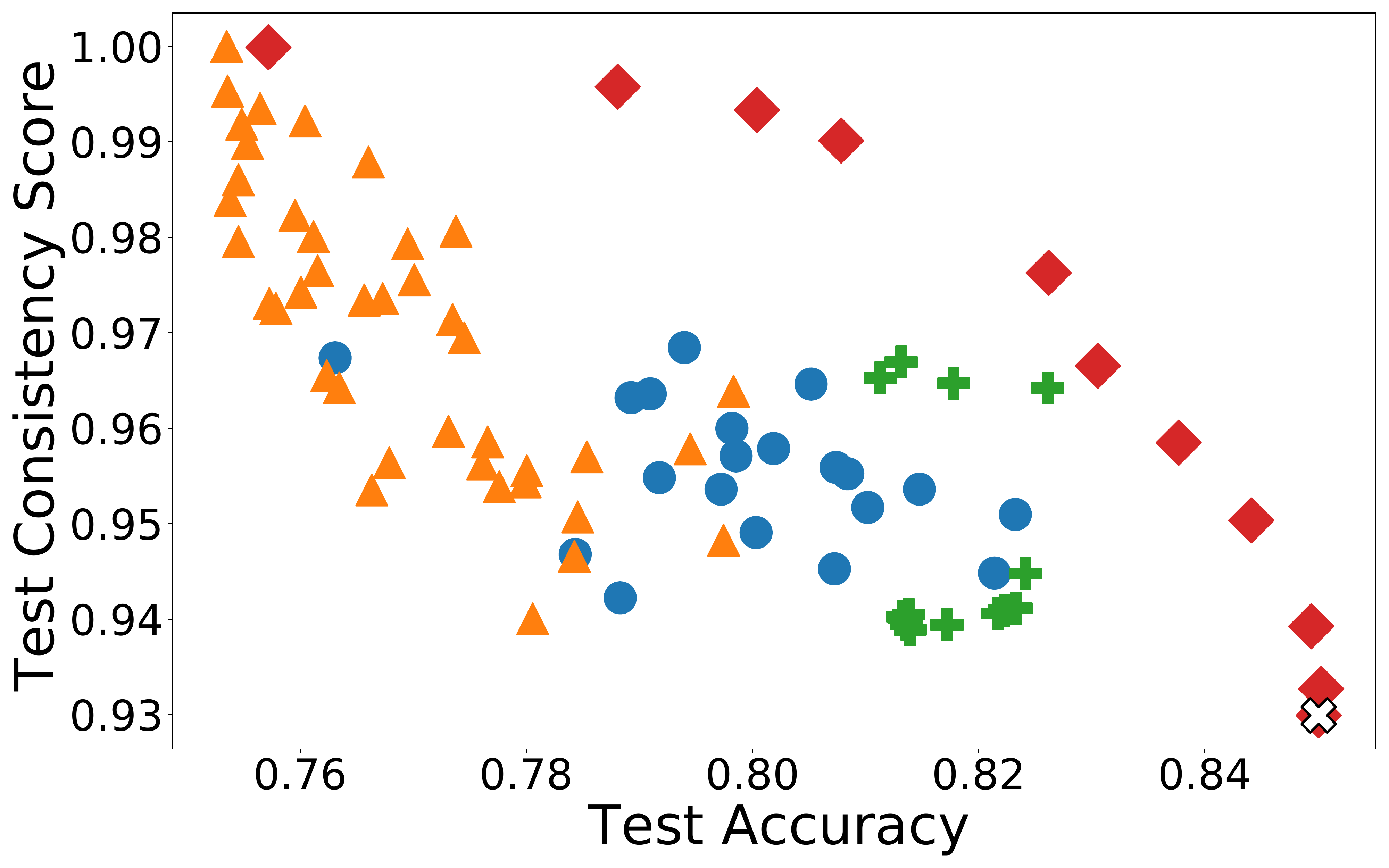}
     \vspace{-0.6cm}
     \caption{{\sf AdultCensus-kNN}}
     \label{fig:AdultCensus-kNN}
  \end{subfigure} 
  \begin{subfigure}{0.33\textwidth}
     \includegraphics[width=\columnwidth]{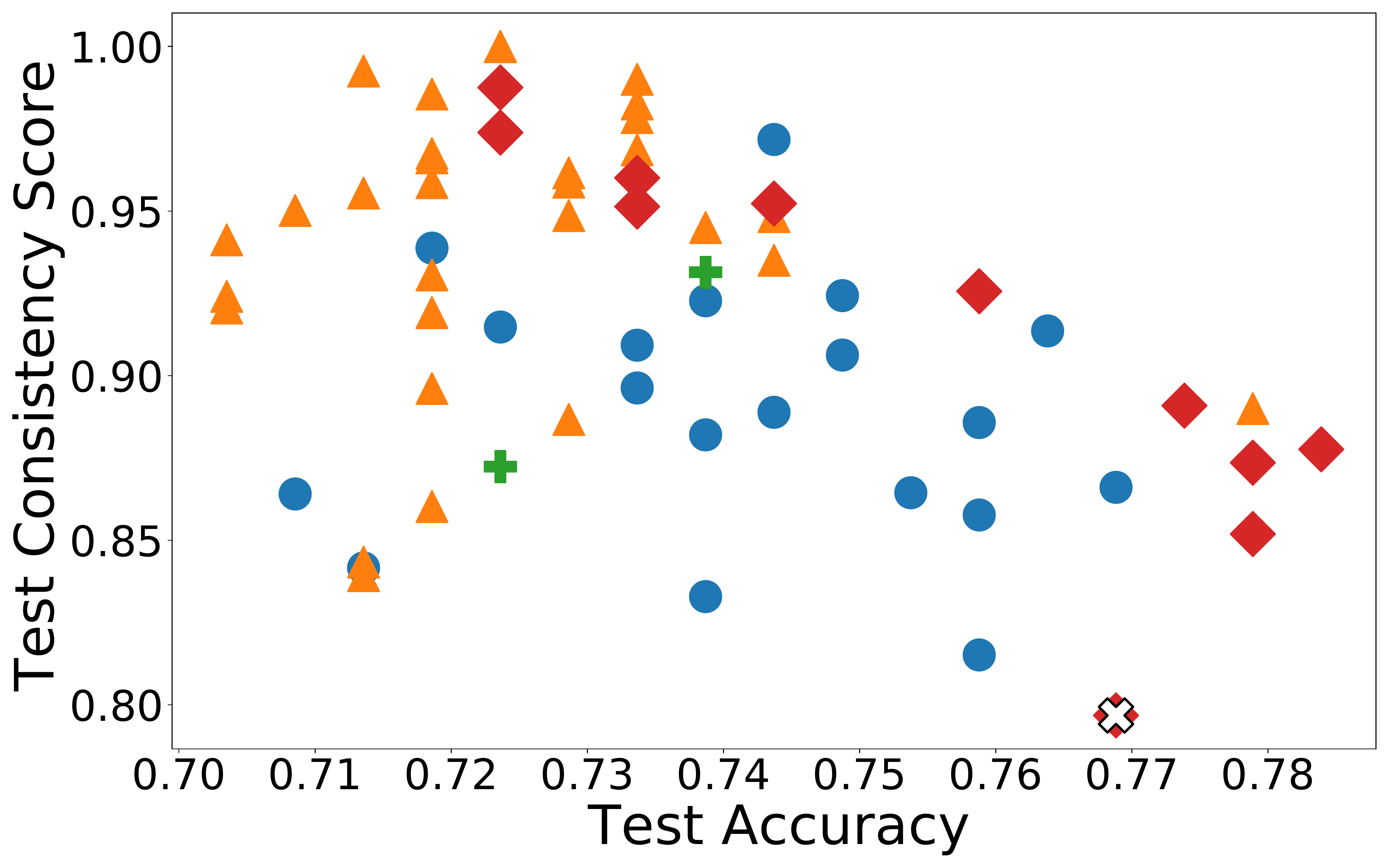}
     \vspace{-0.6cm}
     \caption{{\sf Credit-kNN}}
     \label{fig:Credit-kNN}
  \end{subfigure}
  \begin{subfigure}{0.33\textwidth}
    \includegraphics[width=\columnwidth]{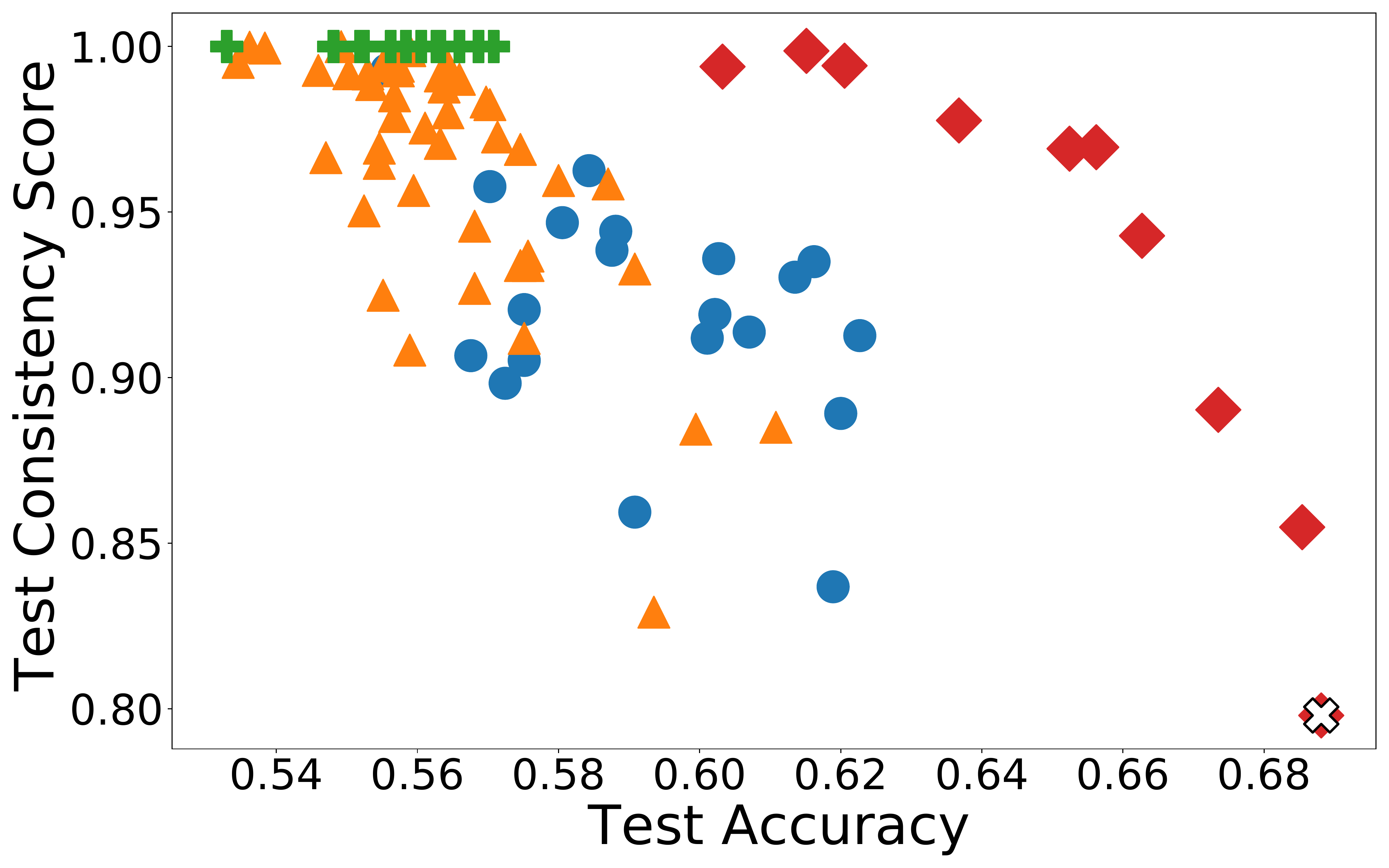}
     \vspace{-0.6cm}
     \caption{{\sf COMPAS-threshold}}
     \label{fig:COMPAS-Threshold}
  \end{subfigure} 
  \begin{subfigure}{0.33\textwidth}
    \includegraphics[width=\columnwidth]{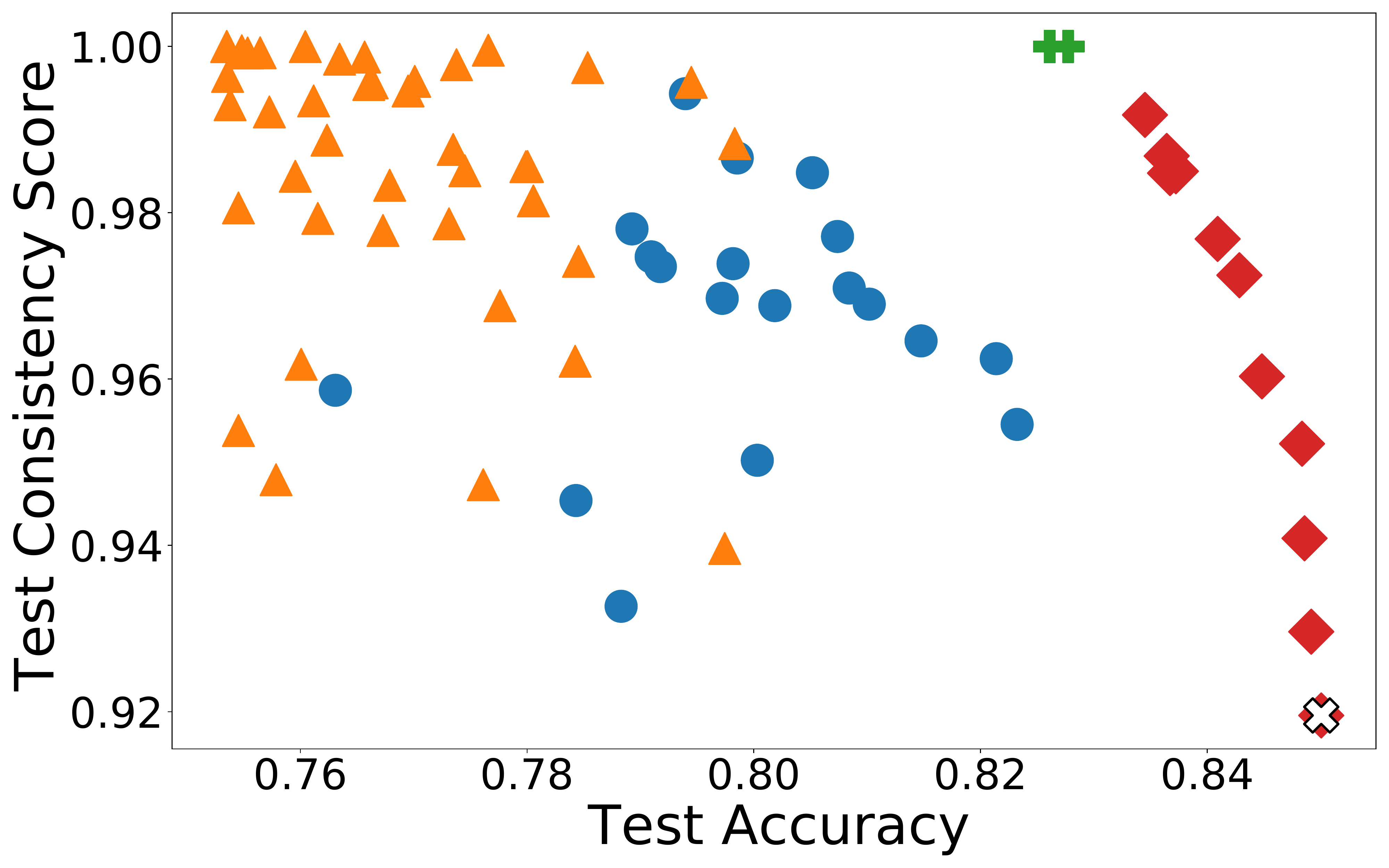}
     \vspace{-0.6cm}
     \caption{{\sf AdultCensus-threshold}}
     \label{fig:AdultCensus-Threshold}
  \end{subfigure}
  \begin{subfigure}{0.33\textwidth}
     \includegraphics[width=\columnwidth]{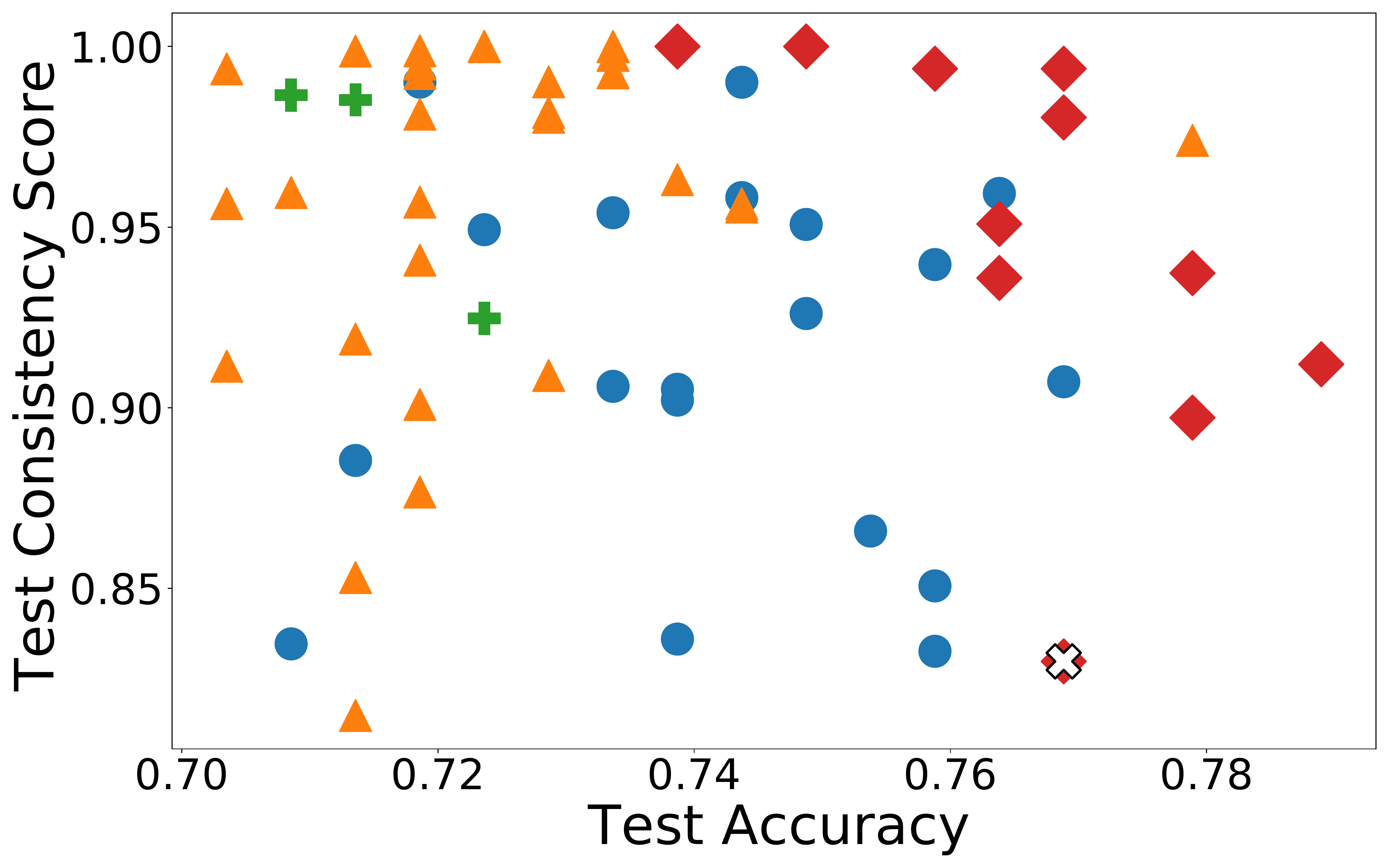}
         \vspace{-0.6cm}
     \caption{{\sf Credit-threshold}}
     \label{fig:Credit-Threshold}
  \end{subfigure}
  \vspace{-0.4cm}
     \caption{Accuracy-fairness trade-offs of logistic regression on the three datasets using the two similarity matrices. In addition to the four methods LFR, iFair, PFR, and \systems{}, we add the result of model training without any pre-processing and call it ``Original.'' As a result, only \systems{} shows a clear accuracy and fairness trade-off.}
 \label{fig:tradeoffcurves}
 \vspace{-0.35cm}
\end{figure*}

\subsection{Performance and Runtime Results}
\label{sec:baselinecomparison}

We now compare \systems{} with other baselines using the three datasets and two similarity matrices.
\subsubsection{Accuracy and Fairness Comparison} 
\label{sec:tradeoff}
We first compare the accuracy and fairness results of \systems{} with the other methods. Figure~\ref{fig:tradeoffcurves} shows the trade-off results with logistic regression where the x-axis is the accuracy, and the y-axis is the consistency score on a test set. {\em Original} is where we train a model on the original data with no data pre-processing. For LFR, iFair, and PFR, we employ 20, 45, and 15 different hyperparameter sets, respectively, as described in their papers.  For \systems{}, we use 10--11 different $m$ values to control accuracy and fairness. For a clear visualization, we omit some points in the bottom left region of the graph for methods other than \systems{}. As a result, \systems{} significantly outperforms the baselines in terms of both test accuracy and consistency score for all cases, which demonstrates that our label flipping approach is effective in improving individual fairness while maintaining high accuracy. We also observe that \systems{} performs better than the three baselines on both similarity matrices. This result is important because the similarity matrix may vary depending on the application. The results for a random forest and neural network are in \ifthenelse{\boolean{techreport}}{Section~\ref{sec:accuracyfairnessothermodels}}{our technical report~\cite{iflippertr}}, and the key trends are similar to Figure~\ref{fig:tradeoffcurves} where \systems{} shows a better trade-off than the baselines. We note that the advantage of \systems{} compared to the baselines is relatively less clear on the Credit dataset mainly because the data is small, making the variance of the results higher.

\begin{table}[t]
\setlength{\tabcolsep}{6.5pt}
  \centering
  \begin{tabular}{ccccccc}
    \toprule
    {\bf Dataset} &\multicolumn{3}{c}{\bf Adult-kNN} & \multicolumn{3}{c}{\bf Adult-threshold} \\
    \midrule
    {\bf Model} & {\bf LR} & {\bf RF}& {\bf NN} & {\bf LR} & {\bf RF}& {\bf NN} \\
    {(C. Score)} & (0.94) & (0.90) & (0.90) & (0.95) & (0.83) & (0.95)\\
    \midrule
    LFR & .815 & .827  & .806 & .821 & .827 & .806\\
    iFair & .797 & .820 & .805 & .798 & .820 & .806\\
    PFR & .826 & .808 & .844 & .826 & .808 & .829\\
    \systems{} & {\bf .844 } & {\bf .851}& {\bf .848}& {\bf .845}& {\bf .850}& {\bf .845}\\
    \bottomrule
  \end{tabular}
  \caption{Accuracy comparison of methods with similar individual fairness on different models. In order to align the individual fairness, for each model we make the methods have similar consistency scores on the validation data (C. Score) by tuning their hyperparameters.}
  \label{tbl:detailcomparison}
  \vspace{-0.5cm}
\end{table}

Table~\ref{tbl:detailcomparison} shows a more detailed comparison with the baselines using the AdultCensus dataset. For each ML model, we fix the consistency score (denoted as ``C. Score'') to be some value as shown below each model name in Table~\ref{tbl:detailcomparison} in parentheses. Then for each method, we report the test accuracy when we tune its hyperparameter such that the trained model has the highest validation accuracy while achieving the fixed consistency score on the validation set. As a result, \systems{} consistently shows the highest test accuracy compared to the other baselines for all models.



In comparison to the other baselines, \systems{}'s accuracy-fairness trade-off is also much cleaner as shown in Figure~\ref{fig:tradeoffcurves} where the other baselines have noisy trade-offs and even overlapping results for different hyperparameter sets. The benefit of having a clean trade-off is that \systems{} can easily balance between accuracy and fairness to obtain a desired fairness by varying the total error limit $m$. In Figure~\ref{fig:tradeoffcurves}, we also observe that \systems{} is flexible and can obtain a wide range of test consistency scores, all the way up to nearly 1.0.

\begin{table}[t]
  \centering
  \begin{tabular}{ccccccc}
    \toprule
    {\bf Datasets} & {\bf Sim. Matrix} & \multicolumn{4}{c}{\bf Avg. Runtime (sec)} \\
    \midrule
    & & {\bf LFR} & {\bf iFair} & {\bf PFR} & {\bf \systems{}} \\
    \midrule
    \multirow{2}{*}{COMPAS} & kNN & 786 & 29,930 & 0.41 & 5.70 \\
    \cmidrule{2-6}
    & threshold & 786 & 29,930 & 0.35 & 8.69\\
    \midrule
    \multirow{2}{*}{\begin{tabular}{@{}c@{}}Adult- \\ Census\end{tabular}} & kNN & 700 & 21,321$\dagger$ & 15.62 & 140\\
    \cmidrule{2-6}
    & threshold & 700 & 21,321$\dagger$ & 14.98 & 685\\
    \midrule
    \multirow{2}{*}{Credit} & kNN & 22.78 & 394 & 0.01 & 1.74\\
    \cmidrule{2-6}
    & threshold & 22.78 & 394 & 0.01 & 0.68\\
    \bottomrule
  \end{tabular}
  \caption{Runtime results of LFR, iFair, PFR, and \systems{} on the COMPAS, AdultCensus, and Credit datasets. For each method, we show the average runtime for all experiments in Figure ~\ref{fig:tradeoffcurves}. The symbol $\dagger$ indicates that we reduce the size of the training data for better efficiency.}
  \label{tbl:runtimecomparison}
  \vspace{-0.8cm}
\end{table}

\begin{figure}[t]
\centering
  \includegraphics[width=0.7\columnwidth]{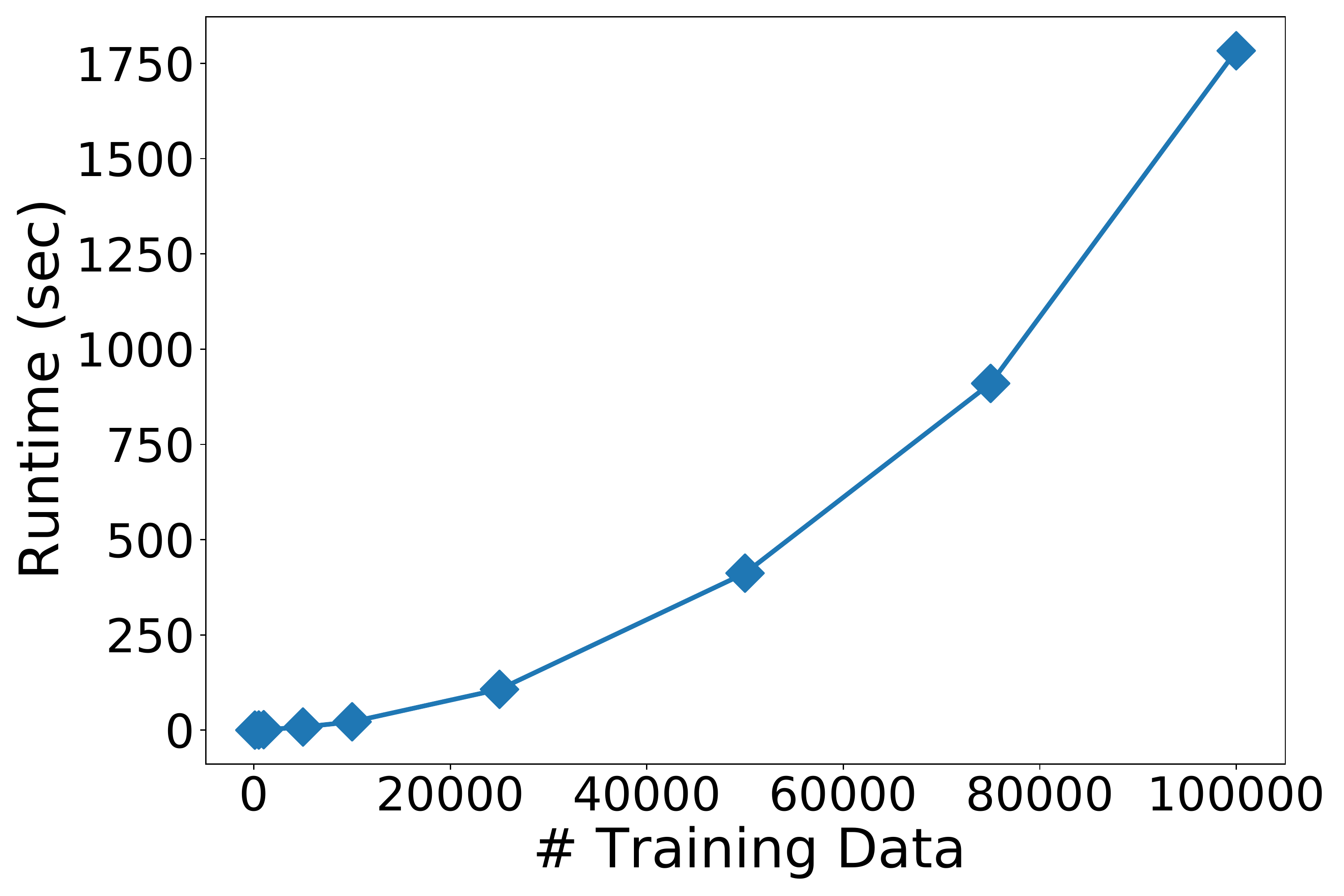}
  \vspace{-0.4cm}
  \caption{\rev{Runtime results of \systems{} on synthetic datasets.}}
  \label{fig:synthetic}
\end{figure}

\begin{figure*}[t]
  \centering
  \begin{subfigure}{0.32\textwidth}
     \includegraphics[width=\columnwidth]{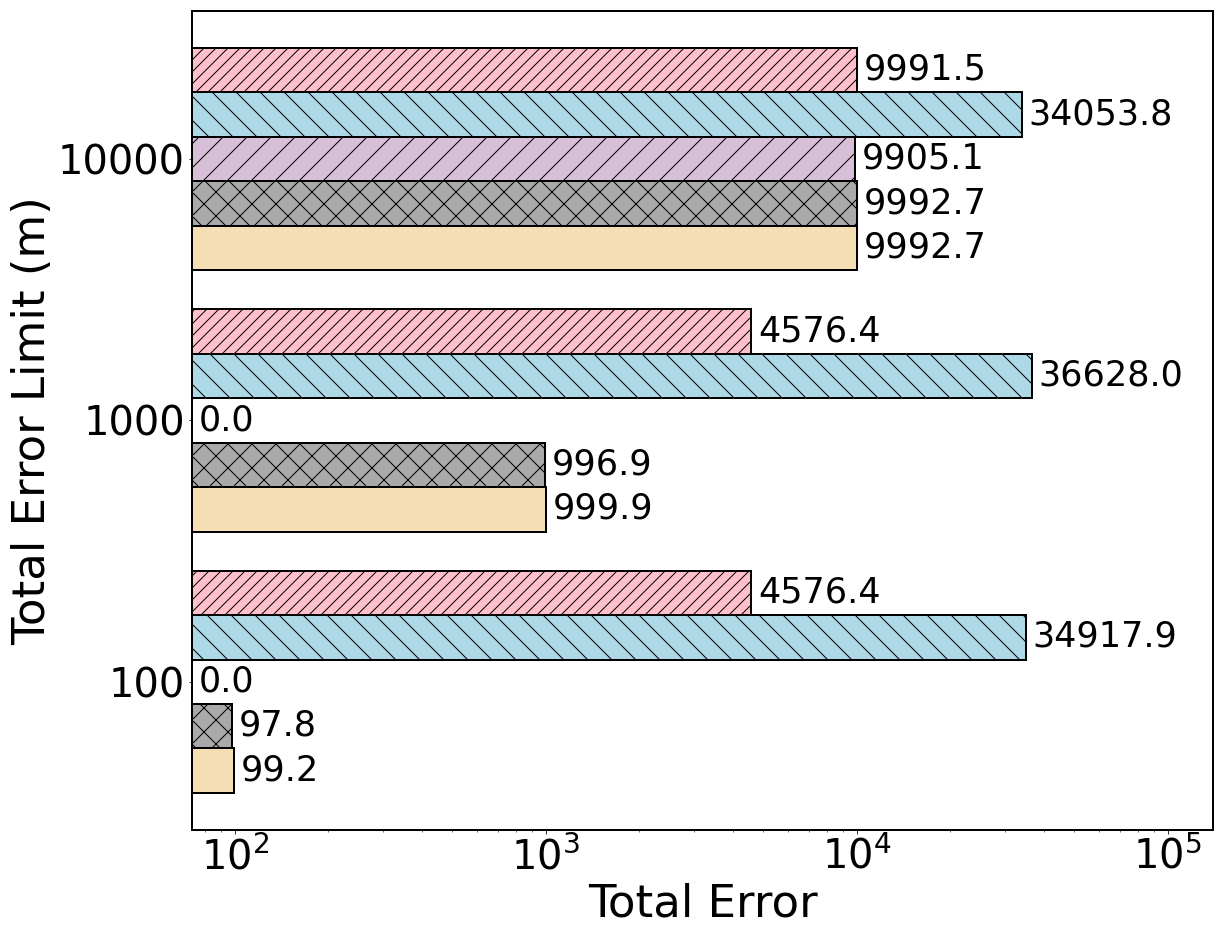}
      \vspace{-0.5cm}
     \caption{{\sf Total error}}
     \label{fig:solutionviolation}
  \end{subfigure}
  \begin{subfigure}{0.32\textwidth}
    \includegraphics[width=\columnwidth]{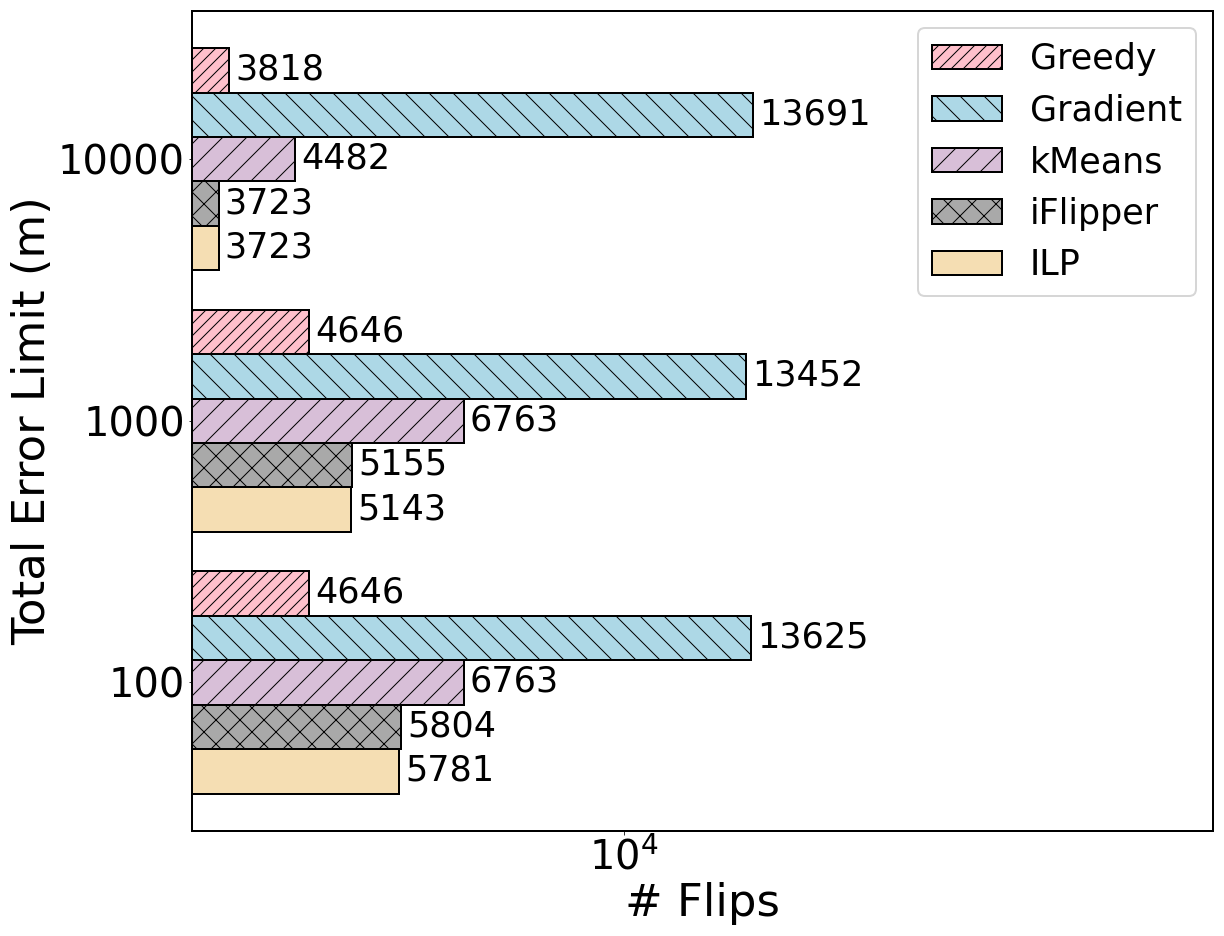}
     \vspace{-0.5cm}
     \caption{{\sf Number of flips}}
     \label{fig:solutionflip}
  \end{subfigure} 
  \begin{subfigure}{0.32\textwidth}
     \includegraphics[width=\columnwidth]{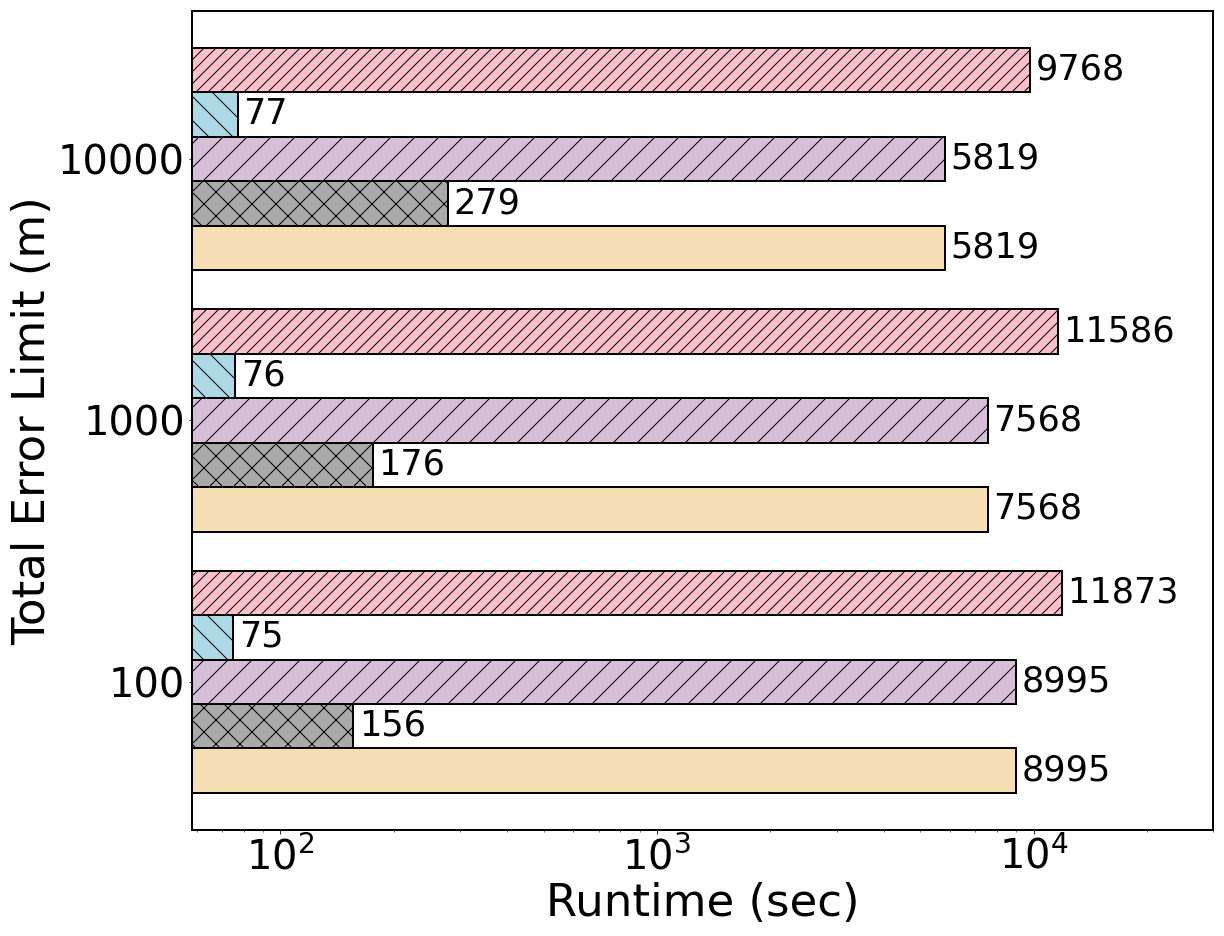}
      \vspace{-0.5cm}
     \caption{{\sf Runtime (sec)}}
     \label{fig:solutionruntime}
  \end{subfigure}
  \vspace{-0.3cm}
     \caption{A detailed comparison of \systems{} against the three na\"ive optimization solutions (\greedy{}, \gradient{}, and \kmeans{}) and ILP solver on the AdultCensus dataset where we use the kNN-based similarity matrix. Here the initial amount of total error is 65,742.5. We show the results for three different total error limits ($m$).}
 \label{fig:solutioncomparison}
 \vspace{-0.3cm}
\end{figure*}

\subsubsection{Runtime Comparison} We evaluate the efficiency of \systems{} and the three baselines in Table~\ref{tbl:runtimecomparison}. For each method, we show the average runtime for all experiments in Figure~\ref{fig:tradeoffcurves}. We note that the runtimes of LFR and iFair do not depend on the similarity matrix. For the iFair experiments on the AdultCensus dataset, we lower the training data using uniform sampling to fit iFair because fitting iFair on the entire data takes more than 24 hours. We confirm that this relaxation does not hurt the original performance reported in the iFair paper~\cite{ifair}. As a result, \systems{} is much faster than LFR and iFair for all cases. Although PFR seems to be the fastest, it performs much worse than \systems{} in terms of accuracy and fairness as shown in Figure~\ref{fig:tradeoffcurves}.

\rev{Another observation is that as the dataset size increases, \systems{}’s runtime increases following the time complexity analysis in Section~\ref{sec:systems}. For a detailed evaluation, we conduct the experiment using random subsets of the synthetic dataset. For each experiment, we set the total error limit $m$ to 20\% of the initial amount of total error. Figure~\ref{fig:synthetic} shows runtime results of \systems{} on datasets of different sizes. As a result, we observe a quadratic increase in running time as the size of the training data increases, which is consistent with our theoretical analysis.}

\subsection{Optimization Solution Comparison}
\label{sec:optimizationcomparison}
We now compare \systems{} with other optimization solutions: \greedy{}, \gradient{}, \kmeans{}, and an ILP solver. Figure~\ref{fig:solutioncomparison} makes a comparison in terms of optimality and efficiency on the AdultCensus dataset where we use the kNN-based similarity matrix. We set the total error limit $m$ to three different values for a detailed comparison. Note that all three subfigures within Figure~\ref{fig:solutioncomparison} use the same legends.

We first compare the resulting total error of the methods in Figure~\ref{fig:solutionviolation}. Obviously, the optimal solution from the ILP solver is the closest to the total error limit. We observe that \systems{} never exceeds the target. On the other hand, both \greedy{} and \gradient{} are unable to find a feasible solution in some cases. For example, \greedy{} and \gradient{} cannot reduce the total error to less than 4,576.4 and 34,053.8, respectively. This result is expected because these methods may fall into local minima as we explained in Section~\ref{sec:naivesolutions}. Also, \kmeans{} does find feasible solutions, but flips too many labels unnecessarily. For example, for $m$ = 100 and $m$ = 1,000, \kmeans{} has zero violations using more flips than \systems{}, which defeats the purpose of the optimization altogether. 



We also compare the number of flips in Figure~\ref{fig:solutionflip}. As a result, \systems{} produces solutions that are close to the optimal ones. When $m$ = 100 and 1,000, \greedy{} shows the fewest number of label flips because it fails to find a feasible solution. \gradient{} has the largest number of flips. We suspect this poor performance is due to the fact that \gradient{} (1) solves an unconstrained optimization problem that makes it hard to satisfy the limit on total error, and (2) provides continuous values, which introduces rounding errors. \kmeans{} also performs worse than \systems{} because its clustering cannot be fine-tuned to have close to $m$ total error.


Finally, we compare the average runtimes (i.e., wall clock times in seconds) of each method in Figure~\ref{fig:solutionruntime}. As a result, \systems{} is much faster than the ILP solver as it solves an approximate LP problem only once. In addition, \greedy{} and \kmeans{} are slower than \systems{} because they use many iterations to reduce the total error and adjust the number of clusters, respectively. \gradient{} is the most efficient, but the result is not meaningful because the total error and the numbers of flips are even worse than \greedy{}.

We also perform the above experiments on the COMPAS dataset, and the results are similar (see \ifthenelse{\boolean{techreport}}{Section~\ref{sec:optimizationcopmas}}{our technical report~\cite{iflippertr}}).

\subsection{Ablation Study}\label{sec:ablationstudy} 
We perform an ablation study in Figure~\ref{fig:ablationperformance} to demonstrate the effectiveness of each component in \systems{} using the AdultCensus dataset and kNN-based similarity matrix. The results for the COMPAS dataset are similar and shown in \ifthenelse{\boolean{techreport}}{Section~\ref{sec:ablationcompas}}{our technical report~\cite{iflippertr}}. As we explained in Section~\ref{sec:experimentalsetting}, MOSEK always produces an optimal LP solution that only has values in \{0, $\alpha$, 1\}, so we do not have an ablation study without the LP solution conversion. Instead, we separately evaluate the conversion algorithm in Section~\ref{sec:conversioneval}. Again, the fact that MOSEK effectively contains the conversion functionality indicates that the conversion overhead is not significant and that a tight integration with the LP solving is possible. We thus compare \systems{} (LP solver with both adaptive rounding and reverse greedy algorithms) with the following variants: (1) LP solver with simple rounding (LP-SR); and (2) LP solver with adaptive rounding (LP-AR). We also consider an optimal solution from the ILP solver to compare the optimality of each solution.

Figure~\ref{fig:ablationsolutionviolation} shows that the adaptive rounding algorithm always provides a feasible solution while simple rounding to a fractional optimal solution results in an infeasible solution as we discussed in Lemma~\ref{alg:rounding}. However, the rounding algorithm does flip more labels than the optimal solution as shown in Figure~\ref{fig:ablationsolutionflip}, resulting in an unintended fairness improvement and accuracy drop. In this case, the reverse greedy algorithm in \systems{} is used to reduce the optimality gap with the optimal solution by recovering as many original labels as possible without exceeding the total error limit.

\begin{figure}[t]
  \centering
  \begin{subfigure}{0.54\columnwidth}
     \includegraphics[width=\columnwidth]{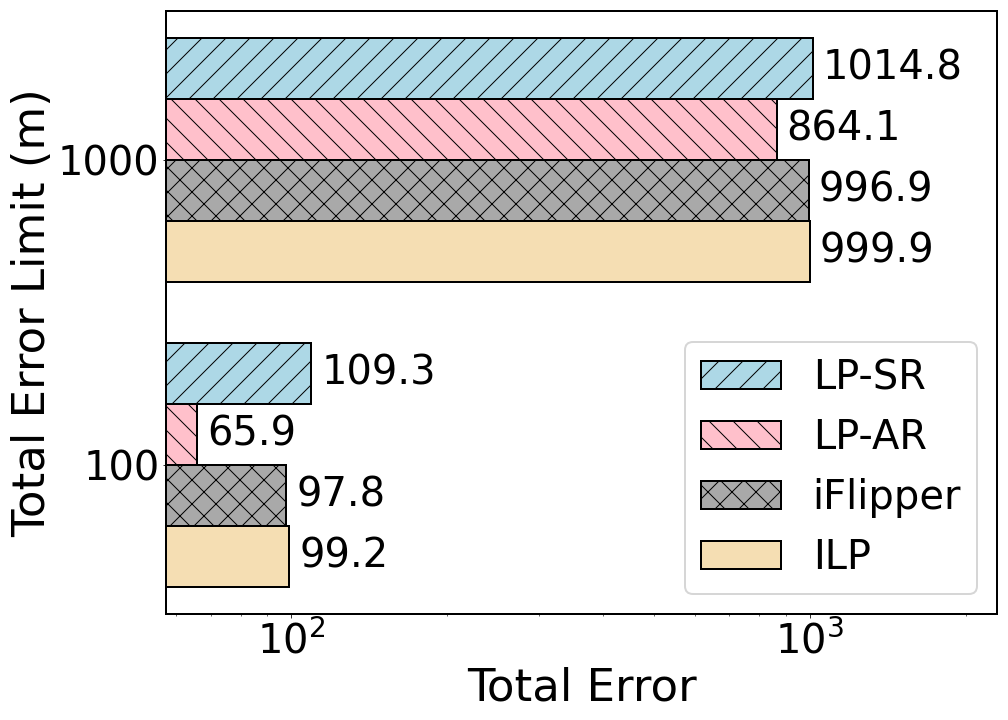}
     \vspace{-0.5cm}
     \caption{{\sf Total error}}
     \label{fig:ablationsolutionviolation}
  \end{subfigure}
  \begin{subfigure}{0.44\columnwidth}
     \includegraphics[width=\columnwidth]{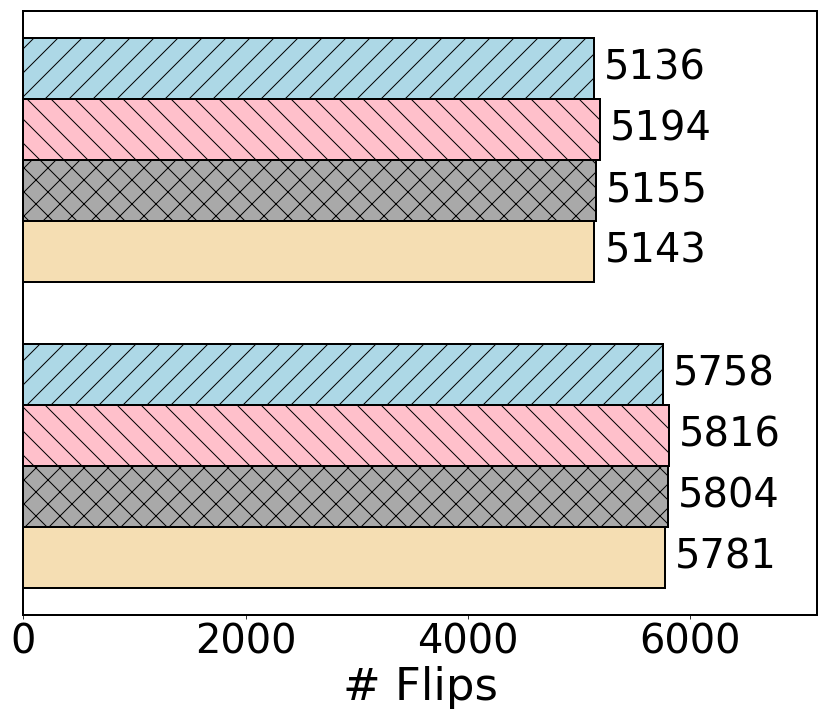}
     \vspace{-0.5cm}
     \caption{{\sf Number of flips}}
     \label{fig:ablationsolutionflip}
  \end{subfigure} 
  \vspace{-0.3cm}
     \caption{Ablation study for \systems{} on the AdultCensus dataset and kNN-based similarity matrix.}
 \label{fig:ablationperformance}
\end{figure}

Table~\ref{tbl:ablationruntime} shows the average runtime of each component (LP solver, adaptive rounding, and reverse greedy) in \systems{} in Figure~\ref{fig:ablationperformance}. As a result, the runtime for solving the LP problem is dominant, which demonstrates that \systems{} is able to provide a near-exact solution with minimal time overhead.

\vspace{-0.1cm}
\begin{table}[ht]
  \centering
  \begin{tabular}{lc}
    \toprule
    \multicolumn{1}{c}{\bf Method} & {\bf Avg. Runtime (sec)} \\ 
    \midrule
    LP Solver (effectively includes Alg.~\ref{alg:converting}) & 153.61\\
    + Adaptive Rounding (Alg.~\ref{alg:rounding}) & 0.48\\
    + Reverse Greedy (Alg.~\ref{alg:reversegreedy}) & 13.26\\
    \bottomrule
  \end{tabular}
  \caption{Avg. runtimes of \systems{}'s components in Figure~\ref{fig:ablationperformance}.}
  \label{tbl:ablationruntime}
  \vspace{-0.5cm}
\end{table}

\subsection{Algorithm~\ref{alg:converting} Evaluation}
\label{sec:conversioneval}
We now evaluate the LP solution conversion (Algorithm~\ref{alg:converting}) separately using the three datasets and kNN-based similarity matrices. For all datasets, we start with an LP problem solution by assigning random values and perform the conversion. As a result, we successfully convert random solutions to new solutions whose values are in \{0, $\alpha$, 1\} where the exact value of $\alpha$ for each dataset is shown in Table~\ref{tbl:converting}. In addition, each converted solution always has the same number of label flips and a smaller total error. This result is consistent with Lemma~\ref{lem:converting}, which says the conversion maintains the optimality of the original solution. Finally, the conversion is much faster than the overall runtime of \systems{} (see Table~\ref{tbl:runtimecomparison}) even in this worst-case scenario where the labels are randomly initialized. 

\vspace{-0.1cm}
\begin{table}[ht]
  \setlength{\tabcolsep}{1.9pt}
  \centering
  \begin{tabular}{ccccccc}
    \toprule
    & & \multicolumn{2}{c}{\bf Before Conv.} & \multicolumn{2}{c}{\bf After Conv.} &  \\
    \cmidrule(lr){1-2}\cmidrule(lr){3-4} \cmidrule(lr){5-6}\cmidrule(lr){7-7}
    {\bf Dataset} & $\boldsymbol{\alpha}$ & {\bf Tot. Err.} & {\bf \# Flips} & {\bf Tot. Err.} & {\bf \# Flips} & {\bf Time(s)} \\
    \midrule
    COMPAS & 0.471 & 12645.1 & 1,838 & 1038.2 & 1,838 & 2.47\\
    AdultCensus & 0.494 & 93519.7 & 13,514 & 3164.4 & 13,514 & 93.63\\
    Credit & 0.448 & 2447.0 & 361 & 128.7 & 361 & 0.25 \\
    \bottomrule
  \end{tabular}
  \caption{\systems{}'s conversion algorithm (Algorithm ~\ref{alg:converting}) on LP problem solutions that have random labels.}
  \label{tbl:converting}
  \vspace{-0.5cm}
\end{table}

\subsection{Compatibility with In-processing Method}
In this section, we demonstrate how \systems{} can be combined with an in-processing algorithm to further improve individual fairness. We evaluate \systems{} with SenSR~\cite{DBLP:conf/iclr/YurochkinBS20}, which is an in-processing fair training algorithm that is robust to sensitive perturbations to the inputs, on the AdultCensus dataset. For \systems{}, we use the same distance function in SenSR, which computes the distance using the features that are not correlated to sensitive attributes and use the threshold $T$ = 2 to construct the similarity matrix. For a fair comparison, we report both our consistency score and the GR-Con.\@ (gender and race consistency) / S-Con.\@ (spouse consistency) metrics used by SenSR to evaluate the individual fairness. Here Con.\@ measures the consistency between $h(x_i)$ and $h(x_i)$ when $x_i$ and $x_j$ are the same except the sensitive attributes. To evaluate Con.\@, we use the same method in SenSR where the test data examples are duplicated, but assigned different sensitive attribute values. 
As a result, Table~\ref{tbl:inprocessing} shows that using both methods gives the best individual fairness for both metrics while having little accuracy drop. Thus, \systems{} complements in-processing algorithms by removing the bias inherent in the data during the pre-processing step.

\vspace{-0.1cm}
\begin{table}[ht]
  \centering
  \begin{tabular}{cccccc}
    \toprule
    {\bf Dataset} & {\bf Method} & {\bf Test Acc.} & {\bf Con. Score} & {\bf GR/S-Con.} \\ 
    \midrule
    \multirow{4}{*}{\begin{tabular}{@{}c@{}}Adult- \\ Census\end{tabular}} & Original & 0.855 & 0.917 & 0.919 / 0.867\\
    & \systems{} & 0.853 & 0.955 & 0.931 / 0.907\\
    & SenSR & 0.836 & 0.953 & 0.990 / 0.945\\
    & Both & 0.829 & {\bf 0.960} & {\bf 0.992} / {\bf 0.984}\\
    \bottomrule
  \end{tabular}
  \caption{Using \systems{}, SenSR, or both on the AdultCensus dataset.}
  \label{tbl:inprocessing}
   \vspace{-0.3cm}
\end{table}


\section{Related Work} \label{sec:related_work}

\ifthenelse{\boolean{revision}}
{\rev{Various fairness measures have been proposed to capture legal and social issues~\cite{narayanan2018translation, verma2018definition}}}
{Various fairness measures have been proposed to capture legal and social issues~\cite{narayanan2018translation}}. The prominent measures include individual fairness~\cite{dwork2012fairness}, group fairness~\cite{zafar2017fairness,agarwal2018reductions,zhang2021omnifair}, and causal fairness~\cite{kusner2017counterfactual}. Individual fairness captures the notion that similar individuals should be treated similarly. Group fairness measures like equal opportunity~\citep{DBLP:conf/nips/HardtPNS16}, equalized odds~\citep{DBLP:conf/nips/HardtPNS16}, and demographic parity~\citep{DBLP:conf/kdd/FeldmanFMSV15} ensure that two sensitive groups have similar statistics. Causal fairness identifies causal relationships among attributes. \rev{Although optimizing for a different objective function, the causality-based methods are often used to improve group or individual fairness as well.} All these measures complement each other, and there is a known conflict between group fairness and individual fairness~\cite{binns2020apparent,friedler2021possibility} that they cannot be satisfied at the same time because of their different assumptions. Our primary focus is on individual fairness.

Recently, many unfairness mitigation techniques for individual fairness~\cite{DBLP:journals/corr/abs-1810-01943} have been proposed where they can be categorized into pre-processing~\cite{pmlr-v28-zemel13,ifair,pfr2019}, in-processing~\cite{DBLP:conf/iclr/YurochkinBS20,yurochkin2021sensei,vargo2021individually}, and post-processing~\cite{petersen2021post} techniques depending on whether the fix occurs before, during, or after model training, respectively. Among the categories, we focus on pre-processing because fixing the data can solve the root cause of unfairness. 

The recent pre-processing works LFR~\cite{pmlr-v28-zemel13}, iFair~\cite{ifair}, and PFR~\cite{pfr2019} all propose fair representation learning algorithms that optimize accuracy and individual fairness. Using an adjacency matrix that represents a fairness graph, the goal is to optimize for a combined objective function with reconstruction loss and individual fairness loss based on the fairness graph. In comparison, \systems{} takes the alternative approach of flipping labels to fix bias, which results in better accuracy-fairness trade-offs and efficiency.

It is also important to understand the in-processing and post-processing techniques. SenSR~\cite{DBLP:conf/iclr/YurochkinBS20} enforces the model to be invariant under certain sensitive perturbations to the inputs using adversarial learning. Several extensions have been proposed as well. SenSeI~\cite{yurochkin2021sensei} enforces invariance on certain sensitive sets and uses a regularization-based approach for jointly optimizing accuracy and fairness, and BuDRO~\cite{vargo2021individually} extends the fairness loss used in SenSR to use gradient boosting. For post-processing, GLIF~\cite{petersen2021post} formulates a graph smoothening problem and uses Laplacian regularization to enforce individual fairness. In comparison, \systems{} can complement all these techniques as we demonstrated with SenSR. 

Pre-processing techniques for other fairness measures have been proposed as well. For group fairness, Kamiran et al.~\cite{kamiran2012data} and Calmon et al.~\cite{calmon2017optimized} change features of the training data to remove dependencies between the sensitive attribute and label and thus achieve statistical parity. It is worth noting that Kamiran et al.~\cite{kamiran2012data} also proposed label flipping (called massaging) to reduce bias in the data for group fairness, but not for individual fairness, which is our focus. 
\rev{There is also a possibility of extending \systems{} to support group fairness. However, we do not believe \systems{} can directly support group fairness with the current optimization objective, so an interesting direction is to develop new label flipping techniques that can support group fairness notions beyond statistical parity by formulating the right optimization problems. }

For causal fairness, Capuchin~\cite{salimi2019interventional} makes a connection between multivalued dependencies (MVDs) and causal fairness, and proposes a data repair algorithm for MVDs using tuple inserts and deletes that also ensures causal fairness. 

\ifthenelse{\boolean{revision}}{
\rev{Finally, there is a line of recent work on data programming~\cite{ratner2017snorkel, ratner2017snorkel2}, which generates probabilistic labels for machine learning datasets by combining weak labels. In comparison, \systems{} focuses on generating (flipping) labels with the different purpose of satisfying individual fairness constraints. \systems{} can be potentially used in combination with data programming systems to provide an individual fairness guarantee on the weakly-supervised labels. For example, the labels generated by data programming paradigm may have uncertainty due to the combination of multiple weak labels. A data programming model can output a label that is say 80\% positive and 20\% negative. Here we can view the label as having a continuous value of 0.8. \systems{} already supports continuous similarity matrices when solving its LP problems (c.f. Section~\ref{sec:milp}), and thus can support continuous labels in a similar fashion.}
}{}

\section{Conclusion and Future Work} \label{sec:conclusion}
We proposed \systems{}, which flips labels in the training data to improve the individual fairness of trained binary classification models.
\systems{} uses pairwise similarity information and minimizes the number of flipped labels to limit the total error to be within a threshold. We proved that this MIQP optimization problem is NP-hard. We then converted the problem to an approximate LP problem, which can be solved efficiently. Our key finding is that the proposed LP algorithm has theoretical guarantees on how close its result is to the optimal solution in terms of the number of label flips. In addition, we further optimized the algorithm without exceeding the total error limit. We demonstrated on real datasets that \systems{} outperforms pre-processing unfairness mitigation baselines in terms of fairness and accuracy, and can complement existing in-processing techniques for better results.

\rev{

In the future, we would like to support multi-class classification and regression tasks. For multi-class classification, we can group the labels to favorable and unfavorable labels and make the problem a binary classification one. Alternatively, we can solve a one-vs-all problem for each class (binary classification) and use a softmax to calculate the final label. For regression, the action becomes changing the label instead of flipping it. Since the current optimization problem can be extended to support continuous labels, we plan to change the optimization setup to support regression problems. 

}

\section*{Acknowledgments}
Ki Hyun Tae, Jaeyoung Park, and Steven Euijong Whang were supported by a Google Research Award and by the National Research Foundation of Korea(NRF) grant funded by the Korea government(MSIT) (No. NRF-2018R1A5A1059921 and NRF-2022R1A2C2004382).

\balance
\bibliographystyle{abbrv}
\bibliography{ref}

\ifthenelse{\boolean{techreport}}{\clearpage 
\newpage
\appendix

\section{Proof for Theorem~\ref{lem:nphard}}
\label{sec:proofnphard}
We continue from Section~\ref{sec:opt} and provide a full proof for Theorem~\ref{lem:nphard}.

\vspace{5pt}

\noindent
{\bf Theorem 1. }{\em The MIQP problem in Equation~\ref{equ:miqp} is NP-hard.}

\begin{proof}
We prove that the MIQP problem in Equation~\ref{equ:miqp} is NP-hard by showing its sub-problem where $W_{ij}$ is binary ($W_{ij} \in \{0,1\}^{n \times n}$) is NP-hard. We prove the sub-problem is NP-hard by reducing it from the well-known \textit{at most $k$-cardinality $s$-$t$ cut problem}.
Given an undirected graph $G =(V,E)$, the \textit{at most $k$-cardinality minimum $s$-$t$ cut problem} is to find the minimum sum of edge weights in the cut edge set $C$ to divide the vertex set $V$ into two sets $V_1$ and $V_2$, where $s \in V_1$ and $t \in V_2$, and the cut edge set $C=\{(v_1,v_2) \in E: v_1 \in V_1, v_2\in V_2 \}$ has at most $k$ edges. This problem is known to be NP-hard even if all the edge weights are 1~\cite{DBLP:journals/dam/BruglieriME04,kim2021solving,kratsch2020multi}. Furthermore, the \textit{existence of at most $k$-cardinality $s$-$t$ cut problem} is also known to be NP-hard~\cite{DBLP:journals/dam/BruglieriME04}.

We will show the reduction as a three-step approach. First, given an instance of the \textit{at most $k$-cardinality $s$-$t$ cut problem} on $G =(V,E)$ with all edge weights equal to 1, we show how to create $k+1$ label flipping problems in~Equation~\ref{equ:miqp} accordingly. Second, we show how a solution to any of the $k+1$ MIQP problems represents an $s$-$t$ cut to the original graph cut problem. Third, we establish that if there exists an \textit{at most $k$-cardinality $s$-$t$ cut}, it must be one of the $k+1$ $s$-$t$ cuts obtained in the previous step.

\underline{Step 1}: We show how to create MIQP problems for a given \textit{at most $k$-cardinality $s$-$t$ cut problem} on $G =(V,E)$. We create a variable $y_i$ for each $v_i \in V-\{s,t\}$. We make $W_{ij}=1$ for each edge $(i,j) \in E$ in the graph, and $W_{ij}=0$ for other pairs of $(i,j)$. If a node $v_i$ is only connected to $s$ ($t$), we make $y_i'=1 (0)$ and add $(y_i - y_i')^2$ to the objective function. That is, we set the initial label of nodes directly connected to $s$ ($t$) as 1 (0). If a node $v_i$ is directly connected to both $s$ and $t$, we add two terms to the objective function for this node, one with $y_i'=0$ and the other one with $y_i'=1$. If a node $v_i$ is not directly connected to $s$ or $t$, we do not add any terms to the objective. Now that we have defined all $y_i$, their initial assignments $y_i'$, and the weight $W_{ij}$, we vary $m$ from 0 to $k$ to create $k+1$ instances of MIQP problems with different allowed amounts of total error as the constraints.

The intuition for the above process of creating $k+1$ MIQP problem is that an $s$-$t$ cut is a binary partition of a graph and can be viewed as a binary-valued labeling process. The nodes that are connected to the source $s$ have label 1 and the nodes that are connected to sink $t$ have label 0. A cut consists of two types of edges: $e_f$ and $e_v$. A flip edge $e_f:(v_1,v_2):v_1 \in \{s,t\}, v_2 \in V - \{s,t\}$ is an edge between $s$ ($t$) and other nodes. If an $e_f$ edge exists in a cut, that means a node which was connected to $s$ ($t$) is no longer in the same partition as $s$ ($t$), which can be represented as a flip of label. See the $(s,1)$ edge in Figure~\ref{fig:mincutrep} as an example for flip edges. A violation edge $e_v:(v_1,v_2):v_1 \in V-\{s,t\}, v_2 \in V - \{s,t\}$ is an edge between two nodes that are not $s$ or $t$. If an $e_v$ edge exists in a cut, that means two nodes that are supposed to be similar and have the same label ended up in two different partitions, which can be represented as a violation. See the $(1,2)$ edge in Figure~\ref{fig:mincutrep} as an example for violation edges. The cardinality of the cut equals to the sum of the counts of the two types of edges, which equals to the number of flips and the total error in the MIQP solution. Our mapping from graph cutting to label flipping is inspired by a similar process used in image segmentation~\cite{kolmogorov2004energy} that aims at using graph cutting to split an image into foreground and background.

\underline{Step 2}: We show how a solution to any of the $k+1$ MIQP problems can be mapped to an $s$-$t$ cut. After solving a corresponding MIQP, we put all nodes with $y_i=1$ in $V_1$ and $y_i=0$ in $V_2$ as the resulting $s$-$t$ cut. The cardinality of the cut equals to the sum of the number of label flips and the total error in the label flipping problem. 

\vspace{-0.15cm}
\begin{figure}[ht]
\centering
  \includegraphics[width=0.8\columnwidth]{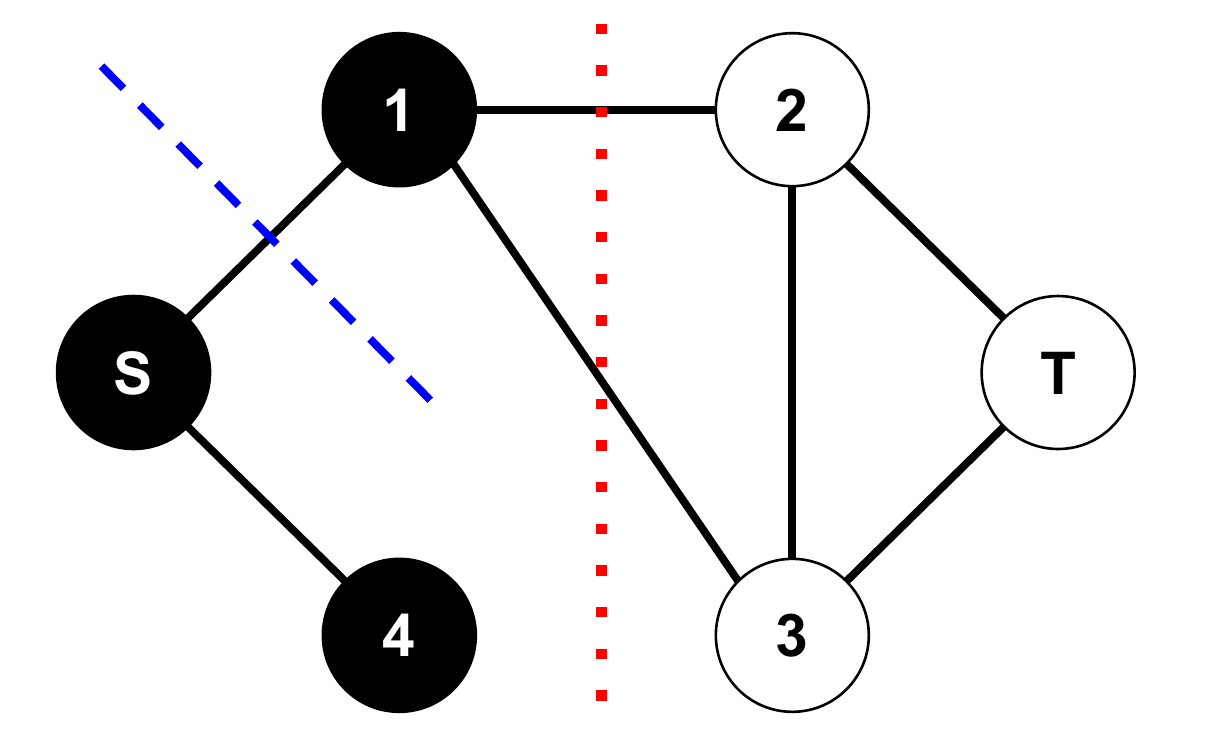}
  \vspace{-0.1cm}
  \caption{An example for the $s$-$t$ cut. The cut represented by the blue dashed line consists of one $e_f$ edge, and the cut represented by the red dotted line consists of two $e_v$ edges.}
  \label{fig:mincutrep}
\end{figure}

In Figure~\ref{fig:mincutrep}, after removing $s$ and $t$, we get the same graph representation as Figure~\ref{fig:graphrep}. By solving the label flipping problem with $m$ equal to 0 or 1, the resulting label assignment is $y_4=1$ and $y_1=y_2=y_3=0$. The corresponding cut is shown as the blue dashed line in Figure~\ref{fig:mincutrep}. The cut is $V_1=\{s,4\}, V_2=\{1,2,3,t\},\text{ and } C=\{(s,1)\}$. This label assignment has one flip (node 1 flips to label 0) and zero violations. This flip is represented as an $e_f$ edge (s,1) in the cut. If we solve the problem with $m=2$, the result is $y_1=y_4=1$ and $y_2=y_3=0$. The corresponding cut is shown as the red dotted line in Figure~\ref{fig:mincutrep}. The cut is $V_1=\{s,1,4\}, V_2=\{2,3,t\},\text{ and } C=\{(1,2),(1,3)\}$. This label assignment has zero flips and a total error of 2, represented by two $e_v$ edges in the cut: (1,2) and (1,3). 

\underline{Step 3}: Since the cardinality of the cut equals to the sum of the number of label flips and the total error in the label flipping problem, finding an at most $k$-cardinality $s$-$t$ cut can be solved by finding a label assignment where the sum of flips and the total error equals to at most $k$. To find such a label assignment, we can repeatedly solve the MIQP problem $k+1$ times, for all $m$ values in \{0, $\ldots$, $k$\}, and check if the sum of label flips and the total error is less than or equal to $k$. The $k$-cardinality minimum $s$-$t$ cut, if exists, must equal to the result of the MIQP problem with at least one of the possible $m$ values. Therefore, if the label flipping MIQP problem is not NP-hard, then the  $k$-cardinality minimum $s$-$t$ cut problem would also not be NP-hard, which contradicts the known results.

\end{proof}

\section{Proof for Lemma~\ref{lem:onecluster}}
\label{sec:proofonecluster}
We continue from Section~\ref{sec:algorithm} and provide a full proof for Lemma~\ref{lem:onecluster}. Here we restate Lemma~\ref{lem:onecluster} with the problem setup:
\vspace{5pt}

\noindent
{\bf Lemma 2.1. }{\em Suppose an $\alpha$-cluster has $A_0$ points whose initial labels are 0 and $A_1$ points whose initial values are 1. Let $N_\alpha=A_0-A_1$. Now suppose there are $U$ nodes connected to the $\alpha$-cluster (shown in Figure~\ref{fig:onecluster}) by an edge $W_{a_i \alpha}$ satisfying 
\begin{equation*}
\label{equ:oneclusercondition}
\begin{split}
    &0 \leq a_{1} \leq ... \leq a_{k} < \alpha < a_{k+1} \leq ... \leq a_{U} \leq 1.
\end{split}
\end{equation*}

Note that there is no connected node with a value of $\alpha$ by construction. Let $S_\alpha = \sum_{i=1}^{k}W_{a_i \alpha} - \sum_{i=k+1}^{U}W_{a_i \alpha}$. Let us also add the following nodes for convenience: $a_{0}=0$ and $a_{U+1}=1$. For an $\alpha$-cluster with $N_\alpha=0$ in the optimal solution, we can always convert $\alpha$ to either $a_k$ or $a_{k+1}$ while maintaining an optimal solution. As a result, we can reduce exactly one non-0/1 value in the optimal solution.
}

\begin{figure}[ht]
\centering
    \vspace{-0.4cm}
    \includegraphics[width=0.45\columnwidth]{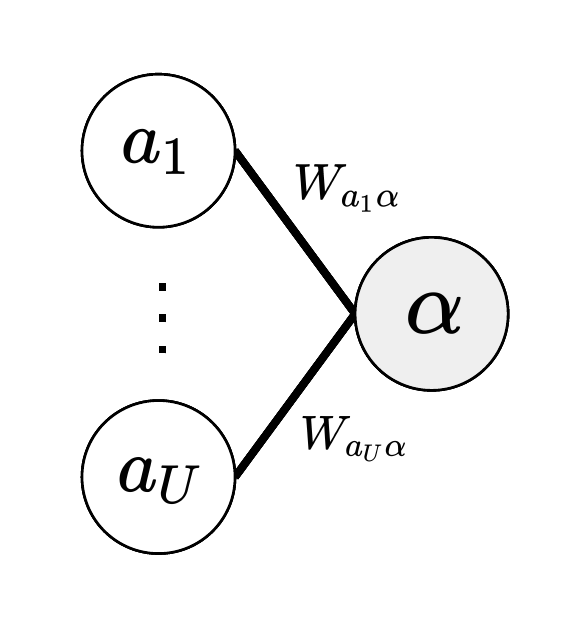}
    \vspace{-0.5cm}
    \caption{Problem setup for Lemma~\ref{lem:onecluster}.}
    \label{fig:onecluster}
\end{figure}

\vspace{-0.4cm}

\begin{proof}
We compute the initial number of flips and total error in Figure~\ref{fig:onecluster} as follows:
\begin{equation}
\label{equ:oneclusterinitial}
\begin{split}
    &\# \text{ Flips} = \alpha A_0 + (1-\alpha) A_1 = A_1 + \alpha N_\alpha = A_1 (\because N_\alpha=0) \\
    &\text{Total Error} = \sum_{i=1}^{U}W_{a_i\alpha}|a_i-\alpha| = S_\alpha \alpha + C \\
    & \qquad \qquad \quad (C=-\sum_{i=1}^{k} W_{a_i\alpha}a_i+\sum_{i=k+1}^{U}W_{a_i\alpha}a_i) \\
\end{split}
\end{equation}

We first observe that the number of flips is independent of the $\alpha$ value. Hence, even if we change $\alpha$ to an arbitrary value, the solution still has the same objective value. 

Now consider a small positive value $\epsilon_\alpha$. Suppose we change $\alpha$ by $\epsilon_\alpha$ such that 
\begin{equation}
\label{equ:oneclusterafterpositive}
    0 \leq a_{1} \leq ... \leq a_{k} \leq \alpha+\epsilon_\alpha \leq a_{k+1} \leq ... \leq a_{U} \leq 1\\
\end{equation}

Similarly, we also change $\alpha$ by $-\epsilon_\alpha$ such that
\begin{equation}
\label{equ:oneclusterafternegative}
    0 \leq a_{1} \leq ... \leq a_{k} \leq \alpha-\epsilon_\alpha \leq a_{k+1} \leq ... \leq a_{U} \leq 1\\
\end{equation}


Note that such $\epsilon_\alpha$ always exists because $a_{k} < \alpha < a_{k+1}$. If we change $\alpha$ to $\alpha+\epsilon_\alpha$ while satisfying Equation~\ref{equ:oneclusterafterpositive}, the total error becomes $S_\alpha (\alpha+\epsilon_\alpha) + C$. Similarly, if we change $\alpha$ to $\alpha-\epsilon_\alpha$ while satisfying Equation~\ref{equ:oneclusterafternegative}, the total error becomes $S_\alpha (\alpha-\epsilon_\alpha) + C$. Hence, the change in the total error for each case can be computed as follows:
\begin{equation}
\label{equ:onechangeinviolations}
\begin{split}
    &\alpha+\epsilon_\alpha \rightarrow{} \Delta(\text{Total Error}) = S_\alpha \epsilon_\alpha \\
    &\alpha-\epsilon_\alpha \rightarrow{} \Delta(\text{Total Error}) = -S_\alpha \epsilon_\alpha \\
\end{split}
\end{equation}


From Equation~\ref{equ:onechangeinviolations}, we observe that one of the transformations always maintains or even reduces the total error according to the sign of $S_\alpha$, i.e., the solution is feasible. Specifically, if $S_\alpha \leq 0$, we change $\alpha$ to $a_{k+1}$ by setting $\epsilon_\alpha=a_{k+1}-\alpha$, otherwise we change $\alpha$ to $a_{k}$ by setting $\epsilon_\alpha=\alpha-a_k$. As a result, we can remove one non-0/1 value $\alpha$ while maintaining an optimal solution.

\end{proof}

\section{Proof for Lemma~\ref{lem:twocluster}}
\label{sec:prooftwocluster}
We continue from Section~\ref{sec:algorithm} and provide a full proof for Lemma~\ref{lem:twocluster}. Here we restate Lemma~\ref{lem:twocluster} with the problem setup:
\vspace{5pt}

\noindent
{\bf Lemma 2.2. }{ \em Let us assume that 0 < $\alpha$ < $\beta$ < 1, and the sum of the pairwise node similarities between the two clusters between the two clusters is $E$. Suppose an $\alpha$-cluster and a $\beta$-cluster have $A_0$ and $B_0$ points whose initial labels are 0, respectively, and $A_1$ and $B_1$ points whose initial values are 1, respectively. Let $N_\alpha=A_0-A_1$ and $N_\beta=B_0-B_1$. Now suppose there are $U$ nodes connected to the $\alpha$-cluster by an edge $W_{a_i \alpha}$ and $V$ nodes connected to the $\beta$-cluster by an edge $W_{b_i \beta}$ satisfying

\begin{equation*}
\label{equ:twoclusercondition}
\begin{split}
    &0 \leq a_{1} \leq ... \leq a_{k} < \alpha < a_{k+1} \leq ... \leq a_{U} \leq 1 \text{ and } \\
    &0 \leq b_{1} \leq ... \leq b_{l} < \beta < b_{l+1} \leq ... \leq b_{V} \leq 1.
\end{split}
\end{equation*}

Note that there is no connected node with a value of $\alpha$ or $\beta$ by construction. Let $S_\alpha = \sum_{i=1}^{k}W_{a_i \alpha} - \sum_{i=k+1}^{U}W_{a_i \alpha}$ and $S_\beta = \sum_{i=1}^{l}W_{b_i \beta} - \sum_{i=l+1}^{V}W_{b_i \beta}$. Let us also add the following nodes for convenience: $a_{0}=0$, $a_{U+1}=1$, $b_{0}=0$, and $b_{V+1}=1$. For an $\alpha$-cluster with $N_\alpha\ne0$ and a $\beta$-cluster with $N_\beta\ne0$ in the optimal solution, we can always convert $(\alpha, \beta)$ to one of $(a_k, \beta+\frac{N_\alpha}{N_\beta}(a_k-\alpha))$, $(a_{k+1}, \beta-\frac{N_\alpha}{N_\beta}(a_{k+1}-\alpha))$, $(\alpha+\frac{N_\beta}{N_\alpha}(\beta-b_l), b_{l})$, $(\alpha-\frac{N_\beta}{N_\alpha}(b_{l+1}-\beta), b_{l+1})$, or $(\frac{\alpha N_\alpha+\beta N_\beta}{N_\alpha+N_\beta}, \frac{\alpha N_\alpha+\beta N_\beta}{N_\alpha+N_\beta})$, while maintaining an optimal solution. As a result, we can reduce at least one non-0/1 value in the optimal solution.}


\begin{figure}[ht]
\centering
  \vspace{-0.4cm}
  \includegraphics[width=0.85\columnwidth]{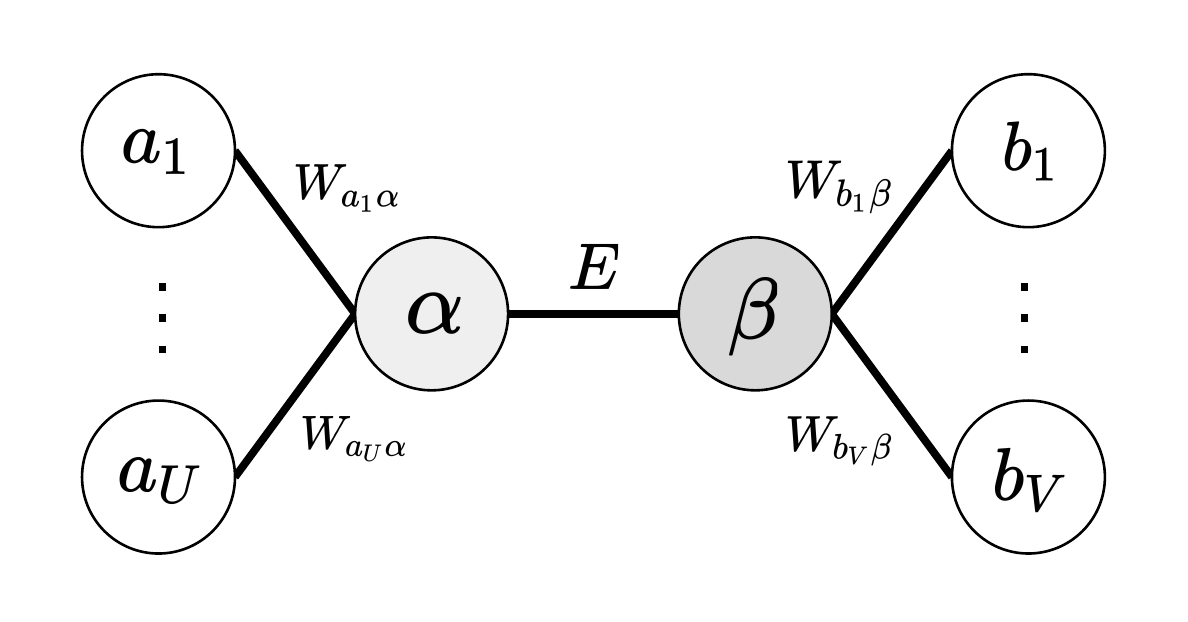}
  \vspace{-0.5cm}
  \caption{Problem setup for Lemma~\ref{lem:twocluster}.}
  \label{fig:twocluster}
\end{figure}

\vspace{-0.4cm}

\begin{proof}
We compute the initial number of flips and total error in Figure~\ref{fig:twocluster} as follows:
\begin{equation}
\label{equ:twoclusterinitial}
\begin{split}
    &\# \text{ Flips} = \alpha A_0 + (1-\alpha) A_1 + \beta B_0 + (1-\beta) B_1 = A_1 + B_1 + \alpha N_\alpha + \beta N_\beta \\
    &\text{Total Error} = \sum_{i=1}^{U}W_{a_i\alpha}|a_i-\alpha|+\sum_{i=1}^{V}W_{b_i\beta}|b_i-\beta|+E(\beta-\alpha) \\
    & \qquad \qquad \qquad \;= (S_\alpha-E)\alpha + (S_\beta+E)\beta + C \\
    & \quad (C =-\sum_{i=1}^{k}W_{a_i \alpha}a_i+\sum_{i=k+1}^{U}W_{a_i \alpha}a_i-\sum_{i=1}^{l}W_{b_i \beta}b_i +\sum_{i=l+1}^{V}W_{b_i \beta}b_i)\\
\end{split}
\end{equation}

Now consider two small positive values $\epsilon_\alpha$ and $\epsilon_\beta$. Both $N_\alpha$ and $N_\beta$ are non-zero, so we have only two cases: $\frac{N_\alpha}{N_\beta}<0$ and $\frac{N_\alpha}{N_\beta}>0$.

{\bf Case 1 $\frac{N_\alpha}{N_\beta}<0$:} Suppose we change $\alpha$ by $\epsilon_\alpha$ and $\beta$ by $\epsilon_\beta$, respectively, such that
\begin{equation}
\label{equ:twoclustercase1positive}
\begin{gathered}
    0 \leq a_{1} \leq ... \leq a_{k} \leq \alpha+\epsilon_\alpha \leq a_{k+1} \leq ... \leq a_{U} \leq 1 \\
    0 \leq b_{1} \leq ... \leq b_{l} \leq \beta+\epsilon_\beta \leq b_{l+1} \leq ... \leq b_{V} \leq 1\\
    \alpha+\epsilon_\alpha \leq \beta+\epsilon_\beta \\
\end{gathered}
\end{equation}

We then compute the number of flips for $(\alpha+\epsilon_\alpha, \beta+\epsilon_\beta)$ as $A_1+B_1+(\alpha+\epsilon_\alpha)N_\alpha+(\beta+\epsilon_\beta)N_\beta$. In order to have the same number of flips as the initial value in Equation~\ref{equ:twoclusterinitial}, ($\epsilon_\alpha$, $\epsilon_\beta$) should satisfy $\epsilon_\alpha N_\alpha+\epsilon_\beta N_\beta=0$.

Similarly, we also change $\alpha$ by $-\epsilon_\alpha$ and $\beta$ by $-\epsilon_\beta$, respectively, such that
\begin{equation}
\label{equ:twoclustercase1negative}
\begin{gathered}
    0 \leq a_{1} \leq ... \leq a_{k} \leq \alpha-\epsilon_\alpha \leq a_{k+1} \leq ... \leq a_{U} \leq 1 \\
    0 \leq b_{1} \leq ... \leq b_{l} \leq \beta-\epsilon_\beta \leq b_{l+1} \leq ... \leq b_{V} \leq 1\\
    \alpha-\epsilon_\alpha \leq \beta-\epsilon_\beta \\
\end{gathered}
\end{equation}

In this case, $(\alpha-\epsilon_\alpha, \beta-\epsilon_\beta)$ also maintains the same number of label flips if $\epsilon_\alpha N_\alpha+\epsilon_\beta N_\beta=0$. 

From now, we consider ($\epsilon_\alpha$, $\epsilon_\beta$) that also satisfies $\epsilon_\alpha N_\alpha+\epsilon_\beta N_\beta=0$, i.e., $\epsilon_\beta=-\frac{N_\alpha}{N_\beta}\epsilon_\alpha$. Note that such $\epsilon_\alpha$ and $\epsilon_\beta$ always exist because $a_{k} < \alpha < a_{k+1}$, $b_{l} < \beta < b_{l+1}$, and $\frac{N_\alpha}{N_\beta}<0$. Therefore, both $(\alpha+\epsilon_\alpha, \beta+\epsilon_\beta)$ and $(\alpha-\epsilon_\alpha, \beta-\epsilon_\beta)$ have the same number of label flipping as the initial number.


If we change $(\alpha, \beta)$ to $(\alpha+\epsilon_\alpha, \beta+\epsilon_\beta)$ while satisfying Equation~\ref{equ:twoclustercase1positive} and $\epsilon_\alpha N_\alpha+\epsilon_\beta N_\beta=0$, the total error becomes $(S_\alpha-E)(\alpha+\epsilon_\alpha)+(S_\beta+E)(\beta+\epsilon_\beta)+C$. Similarly, if we change $(\alpha, \beta)$ to $(\alpha-\epsilon_\alpha, \beta-\epsilon_\beta)$ while satisfying Equation~\ref{equ:twoclustercase1negative} and $\epsilon_\alpha N_\alpha+\epsilon_\beta N_\beta=0$, the total error becomes $(S_\alpha-E)(\alpha-\epsilon_\alpha)+(S_\beta+E)(\beta-\epsilon_\beta)+C$. Hence, the change in the total error for each case can be computed as follows:
\vspace{-0.015cm}
\begin{equation}
\label{equ:twocase1changeinviolations}
\begin{split}
    (\alpha+\epsilon_\alpha, \beta+\epsilon_\beta) \rightarrow{}& \Delta(\text{Total Error}) = (S_\alpha-E)\epsilon_\alpha+(S_\beta+E)\epsilon_\beta\\
    &=\frac{(S_\alpha-E)N_\beta-(S_\beta+E)N_\alpha}{N_\beta}\epsilon_\alpha\\
    (\alpha-\epsilon_\alpha, \beta-\epsilon_\beta) \rightarrow{}& \Delta(\text{Total Error}) = -(S_\alpha-E)\epsilon_\alpha-(S_\beta+E)\epsilon_\beta \\
    &=-\frac{(S_\alpha-E)N_\beta-(S_\beta+E)N_\alpha}{N_\beta}\epsilon_\alpha \\
\end{split}
\end{equation}
\vspace{-0.015cm}
From Equation~\ref{equ:twocase1changeinviolations}, we observe that one of the transformations always maintains or even reduces the total error according to the sign of $\frac{(S_\alpha-E)N_\beta-(S_\beta+E)N_\alpha}{N_\beta}$, i.e., the solution is feasible.

\begin{itemize}
    \item If $\frac{(S_\alpha-E)N_\beta-(S_\beta+E)N_\alpha}{N_\beta} \leq 0$, we change $(\alpha, \beta)$ to $(\alpha+\epsilon_\alpha, \beta+\epsilon_\beta)$ so that the solution is still optimal. Recall $(\epsilon_\alpha, \epsilon_\beta)$ satisfies the three inequalities and one condition: $\alpha+\epsilon_\alpha \leq a_{k+1}$, $\beta+\epsilon_\beta \leq b_{l+1}$, $\alpha+\epsilon_\alpha \leq \beta+\epsilon_\beta$, and $\epsilon_\alpha N_\alpha+\epsilon_\beta N_\beta=0$. Among the possible $(\epsilon_\alpha, \epsilon_\beta)$, we choose the upper bound of $\epsilon_\alpha$ and the corresponding $\epsilon_\beta$ ($\epsilon_\beta=-\frac{N_\alpha}{N_\beta}\epsilon_\alpha$). To get an upper bound of $\epsilon_\alpha$, we find the equality conditions for each inequality and take the smallest value among them. If $1+\frac{N_\alpha}{N_\beta} \leq 0$, the last inequality ($\alpha+\epsilon_\alpha \leq \beta+\epsilon_\beta$) always hold because $\epsilon_\alpha \leq \epsilon_\beta$. Hence, we consider only the first two inequalities and set $\epsilon_\alpha$ to $min(a_{k+1}-\alpha, -\frac{N_\beta}{N_\alpha}(b_{l+1}-\beta)$. On the other hand, if $1+\frac{N_\alpha}{N_\beta} > 0$, we set $\epsilon_\alpha$ to $min(a_{k+1}-\alpha, -\frac{N_\beta}{N_\alpha}(b_{l+1}-\beta), \frac{N_\beta(\beta-\alpha)}{N_\alpha+N_\beta})$ from the three inequalities. As a result, we can convert $(\alpha, \beta)$ to $(a_{k+1}, \beta-\frac{N_\alpha}{N_\beta}(a_{k+1}-\alpha))$, $(\alpha-\frac{N_\beta}{N_\alpha}(b_{l+1}-\beta), b_{l+1})$, or $(\frac{\alpha N_\alpha+\beta N_\beta}{N_\alpha+N_\beta}, \frac{\alpha N_\alpha+\beta N_\beta}{N_\alpha+N_\beta})$, which is one of the cases in Lemma~\ref{lem:twocluster}.
    \item If $\frac{(S_\alpha-E)N_\beta-(S_\beta+E)N_\alpha}{N_\beta} > 0$, we change $(\alpha, \beta)$ to $(\alpha-\epsilon_\alpha, \beta-\epsilon_\beta)$ so that the solution is still optimal. Recall $(\epsilon_\alpha, \epsilon_\beta)$ satisfies the three inequalities and one condition: $a_{k} \leq \alpha-\epsilon_\alpha$, $ b_{l} \leq \beta-\epsilon_\beta$, $\alpha-\epsilon_\alpha \leq \beta-\epsilon_\beta$, and $\epsilon_\alpha N_\alpha+\epsilon_\beta N_\beta=0$. Among the possible $(\epsilon_\alpha, \epsilon_\beta)$, we choose the upper bound of $\epsilon_\alpha$ and the corresponding $\epsilon_\beta$ ($\epsilon_\beta=-\frac{N_\alpha}{N_\beta}\epsilon_\alpha$). To get an upper bound of $\epsilon_\alpha$, we find the equality conditions for each inequality and take the smallest value among them. If $1+\frac{N_\alpha}{N_\beta} \geq 0$, the last inequality ($\alpha-\epsilon_\alpha \leq \beta-\epsilon_\beta$) always holds because $\epsilon_\alpha \geq \epsilon_\beta$. Hence, we consider only the first two inequalities and set $\epsilon_\alpha$ to $min(\alpha-a_{k}, -\frac{N_\beta}{N_\alpha}(\beta-b_l)$. On the other hand, if $1+\frac{N_\alpha}{N_\beta} < 0$, we set $\epsilon_\alpha$ to $min(\alpha-a_{k}, -\frac{N_\beta}{N_\alpha}(\beta-b_l), -\frac{N_\beta(\beta-\alpha)}{N_\alpha+N_\beta})$ from the three inequalities. As a result, we can convert $(\alpha, \beta)$ to $(a_{k}, \beta+\frac{N_\alpha}{N_\beta}(\alpha-a_{k}))$, $(\alpha+\frac{N_\beta}{N_\alpha}(\beta-b_{l}), b_{l})$, or $(\frac{\alpha N_\alpha+\beta N_\beta}{N_\alpha+N_\beta}, \frac{\alpha N_\alpha+\beta N_\beta}{N_\alpha+N_\beta})$, which is one of the cases in Lemma~\ref{lem:twocluster}.
\end{itemize}

{\bf Case 2: $\frac{N_\alpha}{N_\beta}>0$:} The proof is similar to the proof for Case 1 except that we consider $(\alpha+\epsilon_\alpha, \beta-\epsilon_\beta)$ and $(\alpha-\epsilon_\alpha, \beta+\epsilon_\beta)$ instead of $(\alpha+\epsilon_\alpha, \beta+\epsilon_\beta)$ and
$(\alpha-\epsilon_\alpha, \beta-\epsilon_\beta)$. We now write the full proof for Case 2 for completeness.

Suppose we change $\alpha$ by $\epsilon_\alpha$ and $\beta$ by $-\epsilon_\beta$, respectively, such that
\begin{equation}
\label{equ:twoclustercase2positive}
\begin{gathered}
    0 \leq a_{1} \leq ... \leq a_{k} \leq \alpha+\epsilon_\alpha \leq a_{k+1} \leq ... \leq a_{U} \leq 1 \\
    0 \leq b_{1} \leq ... \leq b_{l} \leq \beta-\epsilon_\beta \leq b_{l+1} \leq ... \leq b_{V} \leq 1\\
    \alpha+\epsilon_\alpha \leq \beta-\epsilon_\beta \\
\end{gathered}
\end{equation}

We then compute the number of flips for $(\alpha+\epsilon_\alpha, \beta-\epsilon_\beta)$ as $A_1+B_1+(\alpha+\epsilon_\alpha)N_\alpha+(\beta-\epsilon_\beta)N_\beta$. In order to have the same number of flips as the initial value in Equation~\ref{equ:twoclusterinitial}, ($\epsilon_\alpha$, $\epsilon_\beta$) should satisfy $\epsilon_\alpha N_\alpha-\epsilon_\beta N_\beta=0$.

Similarly, we also change $\alpha$ by $-\epsilon_\alpha$ and $\beta$ by $\epsilon_\beta$, respectively, such that
\begin{equation}
\label{equ:twoclustercase2negative}
\begin{gathered}
    0 \leq a_{1} \leq ... \leq a_{k} \leq \alpha-\epsilon_\alpha \leq a_{k+1} \leq ... \leq a_{U} \leq 1 \\
    0 \leq b_{1} \leq ... \leq b_{l} \leq \beta+\epsilon_\beta \leq b_{l+1} \leq ... \leq b_{V} \leq 1\\
    \alpha-\epsilon_\alpha \leq \beta+\epsilon_\beta \\
\end{gathered}
\end{equation}

In this case, $(\alpha-\epsilon_\alpha, \beta+\epsilon_\beta)$ also maintains the same number of label flips if $\epsilon_\alpha N_\alpha-\epsilon_\beta N_\beta=0$. 

From now, we consider ($\epsilon_\alpha$, $\epsilon_\beta$) that also satisfies $\epsilon_\alpha N_\alpha-\epsilon_\beta N_\beta=0$, i.e., $\epsilon_\beta=\frac{N_\alpha}{N_\beta}\epsilon_\alpha$. Note that such ($\epsilon_\alpha$ and $\epsilon_\beta$) always exist because $a_{k} < \alpha < a_{k+1}$, $b_{l} < \beta < b_{l+1}$, and $\frac{N_\alpha}{N_\beta}>0$. Therefore, both $(\alpha+\epsilon_\alpha, \beta-\epsilon_\beta)$ and $(\alpha-\epsilon_\alpha, \beta+\epsilon_\beta)$ have the same number of label flipping as the initial number.


If we change $(\alpha, \beta)$ to $(\alpha+\epsilon_\alpha, \beta-\epsilon_\beta)$ while satisfying Equation~\ref{equ:twoclustercase2positive} and $\epsilon_\alpha N_\alpha-\epsilon_\beta N_\beta=0$, the total error becomes $(S_\alpha-E)(\alpha+\epsilon_\alpha)+(S_\beta+E)(\beta-\epsilon_\beta)+C$. Similarly, if we change $(\alpha, \beta)$ to $(\alpha-\epsilon_\alpha, \beta+\epsilon_\beta)$ while satisfying Equation~\ref{equ:twoclustercase2negative} and $\epsilon_\alpha N_\alpha-\epsilon_\beta N_\beta=0$, the total error becomes $(S_\alpha-E)(\alpha-\epsilon_\alpha)+(S_\beta+E)(\beta+\epsilon_\beta)+C$. Hence, the change in the total error for each case can be computed as follows:

\begin{equation}
\label{equ:twocase2changeinviolations}
\begin{split}
    (\alpha+\epsilon_\alpha, \beta-\epsilon_\beta) \rightarrow{} &\Delta(\text{Total Error}) = (S_\alpha-E)\epsilon_\alpha-(S_\beta+E)\epsilon_\beta\\
    &=\frac{(S_\alpha-E)N_\beta-(S_\beta+E)N_\alpha}{N_\beta}\epsilon_\alpha\\
    (\alpha-\epsilon_\alpha, \beta+\epsilon_\beta) \rightarrow{} &\Delta(\text{Total Error}) = -(S_\alpha-E)\epsilon_\alpha+(S_\beta+E)\epsilon_\beta \\
    &=-\frac{(S_\alpha-E)N_\beta-(S_\beta+E)N_\alpha}{N_\beta}\epsilon_\alpha \\
\end{split}
\end{equation}

From Equation~\ref{equ:twocase2changeinviolations}, we observe that one of the transformations always maintains or reduces the total error according to the sign of $\frac{(S_\alpha-E)N_\beta-(S_\beta+E)N_\alpha}{N_\beta}$, i.e., the solution is feasible.

\begin{itemize}
    \item If $\frac{(S_\alpha-E)N_\beta-(S_\beta+E)N_\alpha}{N_\beta} \leq 0$, we can change $(\alpha, \beta)$ to $(\alpha+\epsilon_\alpha, \beta-\epsilon_\beta)$ so that the solution is still optimal. Recall $(\epsilon_\alpha, \epsilon_\beta)$ satisfies the three inequalities and one condition: $\alpha+\epsilon_\alpha \leq a_{k+1}$, $b_{l+1} \leq \beta-\epsilon_\beta$, $\alpha+\epsilon_\alpha \leq \beta+\epsilon_\beta$, and $\epsilon_\alpha N_\alpha-\epsilon_\beta N_\beta=0$. Among the possible $(\epsilon_\alpha, \epsilon_\beta)$, we choose the upper bound of $\epsilon_\alpha$ and the corresponding $\epsilon_\beta$ ($\epsilon_\beta=\frac{N_\alpha}{N_\beta}\epsilon_\alpha$). To get an upper bound of $\epsilon_\alpha$, we find the equality conditions for each inequality and take the smallest value among them. Specifically, we set $\epsilon_\alpha$ to $min(a_{k+1}-\alpha, \frac{N_\beta}{N_\alpha}(\beta-b_l), \frac{N_\beta(\beta-\alpha)}{N_\alpha+N_\beta})$ and $\epsilon_\beta$. As a result, we can convert $(\alpha, \beta)$ to $(a_{k+1}, \beta-\frac{N_\alpha}{N_\beta}(a_{k+1}-\alpha))$, $(\alpha+\frac{N_\beta}{N_\alpha}(\beta-b_l), b_{l})$, or $(\frac{\alpha N_\alpha+\beta N_\beta}{N_\alpha+N_\beta}, \frac{\alpha N_\alpha+\beta N_\beta}{N_\alpha+N_\beta})$, which is one of the cases in Lemma~\ref{lem:twocluster}.
    
    \item If $\frac{(S_\alpha-E)N_\beta-(S_\beta+E)N_\alpha}{N_\beta} > 0$, we can change $(\alpha, \beta)$ to $(\alpha-\epsilon_\alpha, \beta+\epsilon_\beta)$ so that the solution is still optimal. Recall $(\epsilon_\alpha, \epsilon_\beta)$ satisfies the three inequalities and one condition: $a_k \leq \alpha-\epsilon_\alpha$, $ \beta+\epsilon_\beta \leq b_{l+1}$, $\alpha-\epsilon_\alpha \leq \beta+\epsilon_\beta$, and $\epsilon_\alpha N_\alpha-\epsilon_\beta N_\beta=0$. Among the possible $(\epsilon_\alpha, \epsilon_\beta)$, we choose the upper bound of $\epsilon_\alpha$ and the corresponding $\epsilon_\beta$ ($\epsilon_\beta=\frac{N_\alpha}{N_\beta}\epsilon_\alpha$). To get an upper bound of $\epsilon_\alpha$, we find the equality conditions for each inequality and take the smallest value among them. In this case, the last inequality ($\alpha-\epsilon_\alpha \leq \beta+\epsilon_\beta$) always hold. Hence, we consider only the first two conditions and set $\epsilon_\alpha$ to $min(\alpha-a_{k}, \frac{N_\beta}{N_\alpha}(b_{l+1}-\beta))$. As a result, we can convert $(\alpha, \beta)$ to either $(a_{k}, \beta+\frac{N_\alpha}{N_\beta}(a_{k}-\alpha))$ or $(\alpha-\frac{N_\beta}{N_\alpha}(b_{l+1}-\beta), b_{l+1})$, which is one of the cases in Lemma~\ref{lem:twocluster}.
\end{itemize}

We summarize the main results for each case below and conclude that $(\alpha, \beta)$ can be transformed into one of the five cases in Lemma~\ref{lem:twocluster} while maintaining an optimal optimal. As a result, we remove at least one of $\alpha$ and $\beta$. If both converted values already exist in the solution, we can even reduce two non-0/1 values.

\begin{itemize}
    \item If ($\frac{N_\alpha}{N_\beta}<0$, $\frac{(S_\alpha-E)N_\beta-(S_\beta+E)N_\alpha}{N_\beta} \leq 0$, $1+\frac{N_\alpha}{N_\beta}\leq 0$), we convert $(\alpha, \beta)$ to $(\alpha+\epsilon_\alpha, \beta+\epsilon_\beta)$ where $\epsilon_\alpha=min(a_{k+1}-\alpha, -\frac{N_\beta}{N_\alpha}(b_{l+1}-\beta)$ and $\epsilon_\beta=-\frac{N_\alpha}{N_\beta}\epsilon_\alpha$. 
    \item If ($\frac{N_\alpha}{N_\beta}<0$, $\frac{(S_\alpha-E)N_\beta-(S_\beta+E)N_\alpha}{N_\beta} \leq 0$, $1+\frac{N_\alpha}{N_\beta}> 0$), we convert $(\alpha, \beta)$ to $(\alpha+\epsilon_\alpha, \beta+\epsilon_\beta)$ where $\epsilon_\alpha=min(a_{k+1}-\alpha, -\frac{N_\beta}{N_\alpha}(b_{l+1}-\beta), \frac{N_\beta(\beta-\alpha)}{N_\alpha+N_\beta})$ and $\epsilon_\beta=-\frac{N_\alpha}{N_\beta}\epsilon_\alpha$. 
    \item If ($\frac{N_\alpha}{N_\beta}<0$, $\frac{(S_\alpha-E)N_\beta-(S_\beta+E)N_\alpha}{N_\beta} > 0$, $1+\frac{N_\alpha}{N_\beta} \geq 0$), we convert $(\alpha, \beta)$ to $(\alpha-\epsilon_\alpha, \beta-\epsilon_\beta)$ where $\epsilon_\alpha=min(\alpha-a_{k}, -\frac{N_\beta}{N_\alpha}(\beta-b_l)$ and $\epsilon_\beta=-\frac{N_\alpha}{N_\beta}\epsilon_\alpha$. 
    \item If ($\frac{N_\alpha}{N_\beta}<0$, $\frac{(S_\alpha-E)N_\beta-(S_\beta+E)N_\alpha}{N_\beta} > 0$, $1+\frac{N_\alpha}{N_\beta} < 0$), we convert $(\alpha, \beta)$ to $(\alpha-\epsilon_\alpha, \beta-\epsilon_\beta)$ where $\epsilon_\alpha=min(\alpha-a_{k}, -\frac{N_\beta}{N_\alpha}(\beta-b_l), -\frac{N_\beta(\beta-\alpha)}{N_\alpha+N_\beta})$ and $\epsilon_\beta=-\frac{N_\alpha}{N_\beta}\epsilon_\alpha$. 
    \item If ($\frac{N_\alpha}{N_\beta}>0$, $\frac{(S_\alpha-E)N_\beta-(S_\beta+E)N_\alpha}{N_\beta} \leq 0$), we convert $(\alpha, \beta)$ to $(\alpha+\epsilon_\alpha, \beta-\epsilon_\beta)$ where $\epsilon_\alpha=min(a_{k+1}-\alpha, \frac{N_\beta}{N_\alpha}(\beta-b_l), \frac{N_\beta(\beta-\alpha)}{N_\alpha+N_\beta})$ and $\epsilon_\beta=\frac{N_\alpha}{N_\beta}\epsilon_\alpha$. 
    \item If ($\frac{N_\alpha}{N_\beta}>0$, $\frac{(S_\alpha-E)N_\beta-(S_\beta+E)N_\alpha}{N_\beta} > 0$), we convert $(\alpha, \beta)$ to $(\alpha-\epsilon_\alpha, \beta+\epsilon_\beta)$ where $\epsilon_\alpha=min(\alpha-a_{k}, \frac{N_\beta}{N_\alpha}(b_{l+1}-\beta))$ and $\epsilon_\beta=\frac{N_\alpha}{N_\beta}\epsilon_\alpha$. 
\end{itemize}

\end{proof}

\section{Trade-off for other datasets}
\label{sec:tradeoffotherdatasets}
We continue from Section~\ref{sec:accuracyfairnesstradeoff} and show the trade-off results on the AdultCensus and Credit datasets in Figure~\ref{fig:tradeofftwodatasets}. The results are similar to the COMPAS dataset in Figure~\ref{fig:tradeoffcompas} where there is a clear trade-off between accuracy and fairness.

\begin{figure}[ht]
\vspace{-0.1cm}
  \centering
  \begin{subfigure}{0.8\columnwidth}
     \includegraphics[width=\columnwidth]{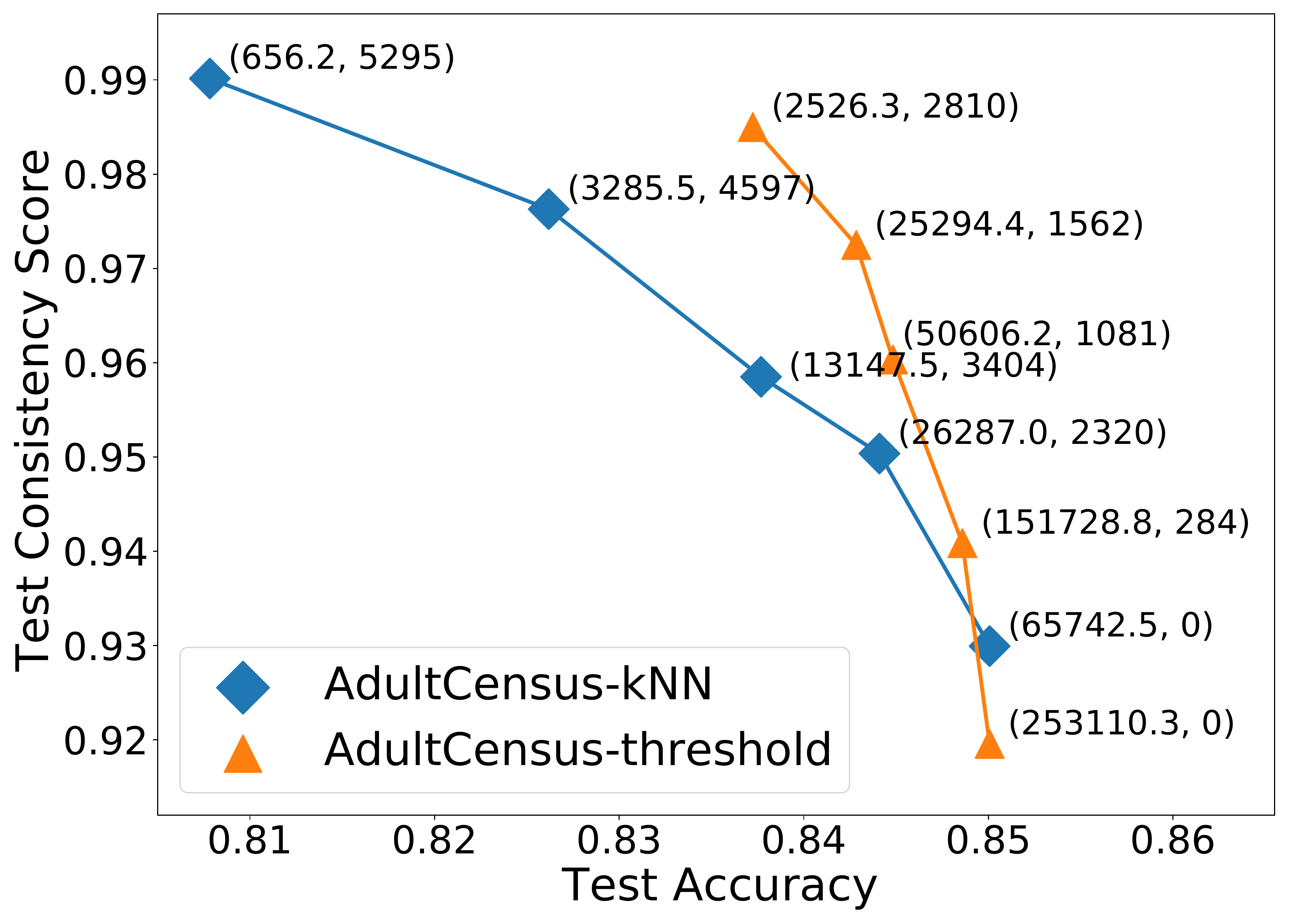}
     \vspace{-0.5cm}
     \caption{{\sf AdultCensus}}
     \label{fig:compastradeoff}
  \end{subfigure}
  \begin{subfigure}{0.8\columnwidth}
     \includegraphics[width=\columnwidth]{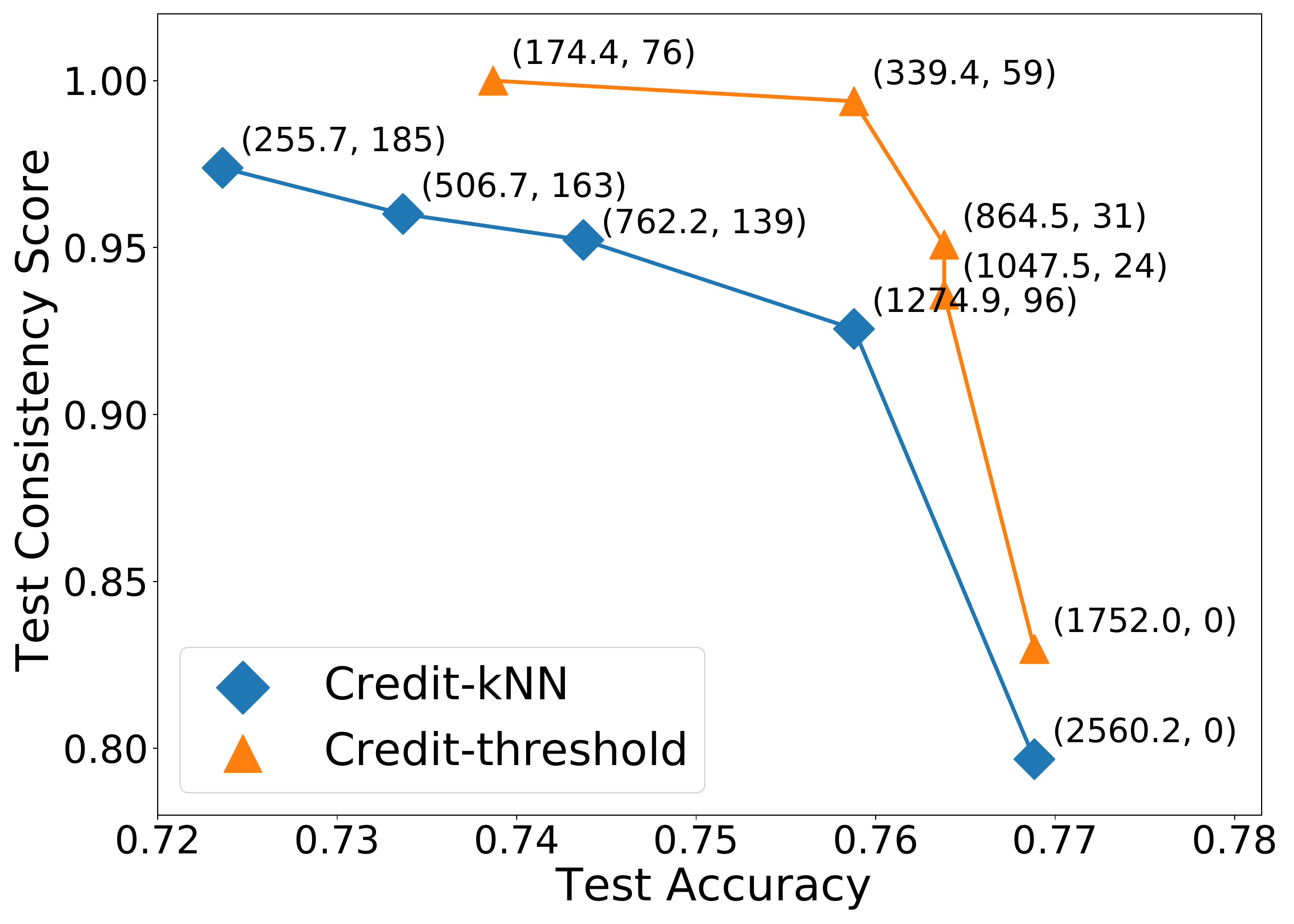}
     \vspace{-0.5cm}
     \caption{{\sf Credit}}
     \label{fig:adulttradeoff}
  \end{subfigure} 
     \caption{Trade-off curves on the AdultCensus and Credit datasets.}
     \vspace{-0.5cm}
 \label{fig:tradeofftwodatasets}
\end{figure}

\section{Bank and LSAC Datasets Results}
\label{sec:banklawschool}
We continue from Section~\ref{sec:experimentalsetting} and show experimental results on the Bank and LSAC datasets when training a logistic regression model. Table~\ref{tbl:banklawschool} shows consistency scores with respect to the threshold-based similarity matrix. As a result, both datasets have almost 1.0 consistency scores, which means that they are already inherently fair in terms of individual fairness. The reason is that these datasets have much smaller total errors compared to other fairness datasets. 

\begin{table}[ht]
  \centering
  \begin{tabular}{ccc}
    \toprule
    {\bf Dataset} & {\bf Test Accuracy} & {\bf Consistency Score}\\
    \midrule
    Bank & 0.956 & 0.997\\
    LSAC & 0.825 & 0.986\\
    \bottomrule
  \end{tabular}
  \caption{Accuracy and fairness results using logistic regression on the Bank and LSAC datasets.}
  \label{tbl:banklawschool}
\end{table}

\vspace{-0.8cm}

\section{Optimization Solutions for COMPAS}
\label{sec:optimizationcopmas}
We continue from Section~\ref{sec:optimizationcomparison} and perform the same experiments on the COMPAS dataset where we use the kNN-based similarity matrix. In Figure~\ref{fig:solutioncomparison_compas}, the key trends are still similar to Figure~\ref{fig:solutioncomparison} where \systems{} (1) always satisfies the total error limit while \greedy{} and \gradient{} result in infeasible solutions for some cases while \kmeans{} returns feasible solutions, but flips too many labels (Figure~\ref{fig:compas_violation}), (2) provides the solution closest to the optimal in terms the number of label flips (Figure~\ref{fig:compas_flip}), and (3) is much faster than other optimization solutions (Figure~\ref{fig:compas_runtime}). Compared to the AdultCensus results in Figure~\ref{fig:solutioncomparison}, \systems{} is the most efficient because the COMPAS dataset is relatively small and has a smaller total error.

\begin{figure}[ht]
  \centering
  \begin{subfigure}{0.69\columnwidth}
     \includegraphics[width=\columnwidth]{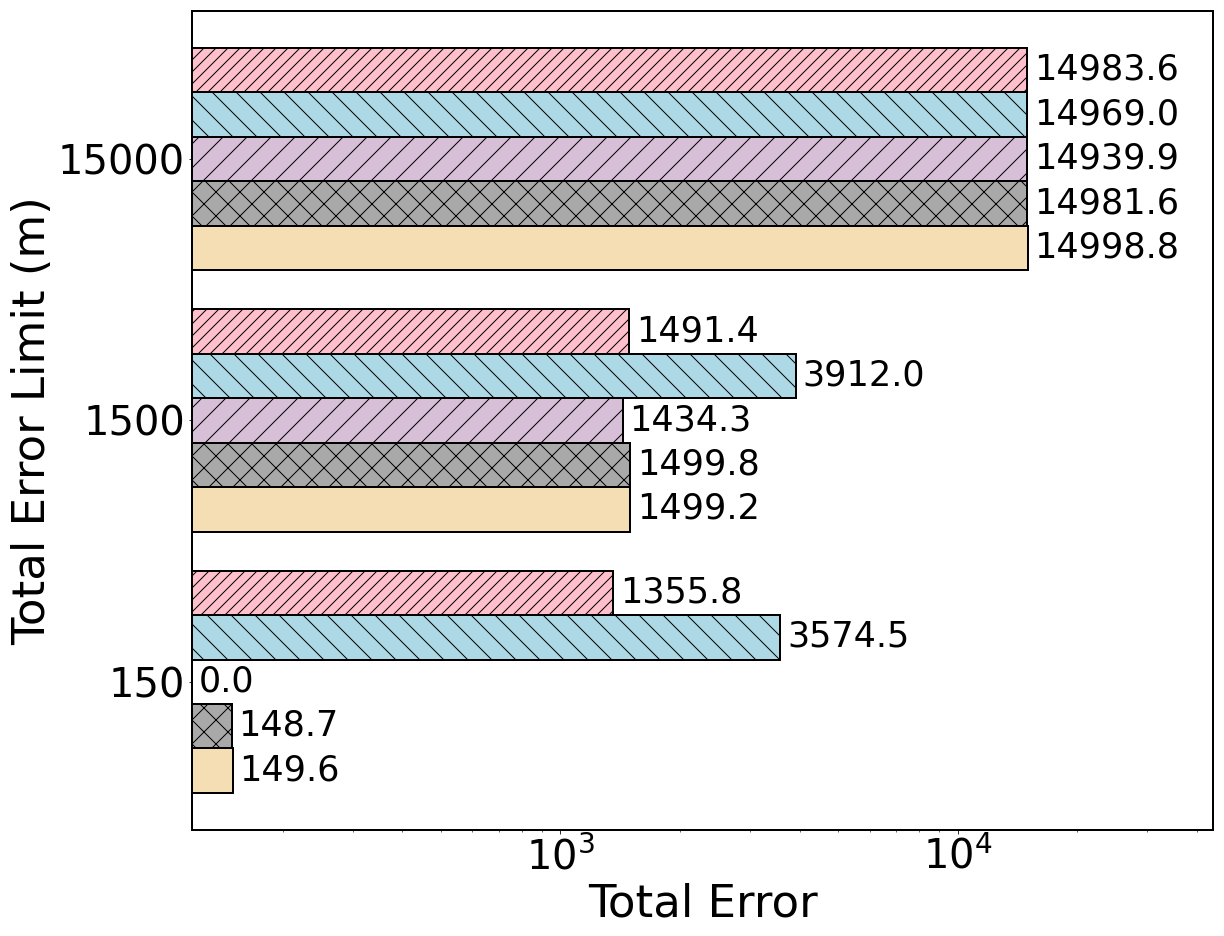}
     \vspace{-0.5cm}
     \caption{{\sf Total error}}
     \label{fig:compas_violation}
  \end{subfigure}
  \begin{subfigure}{0.69\columnwidth}
     \includegraphics[width=\columnwidth]{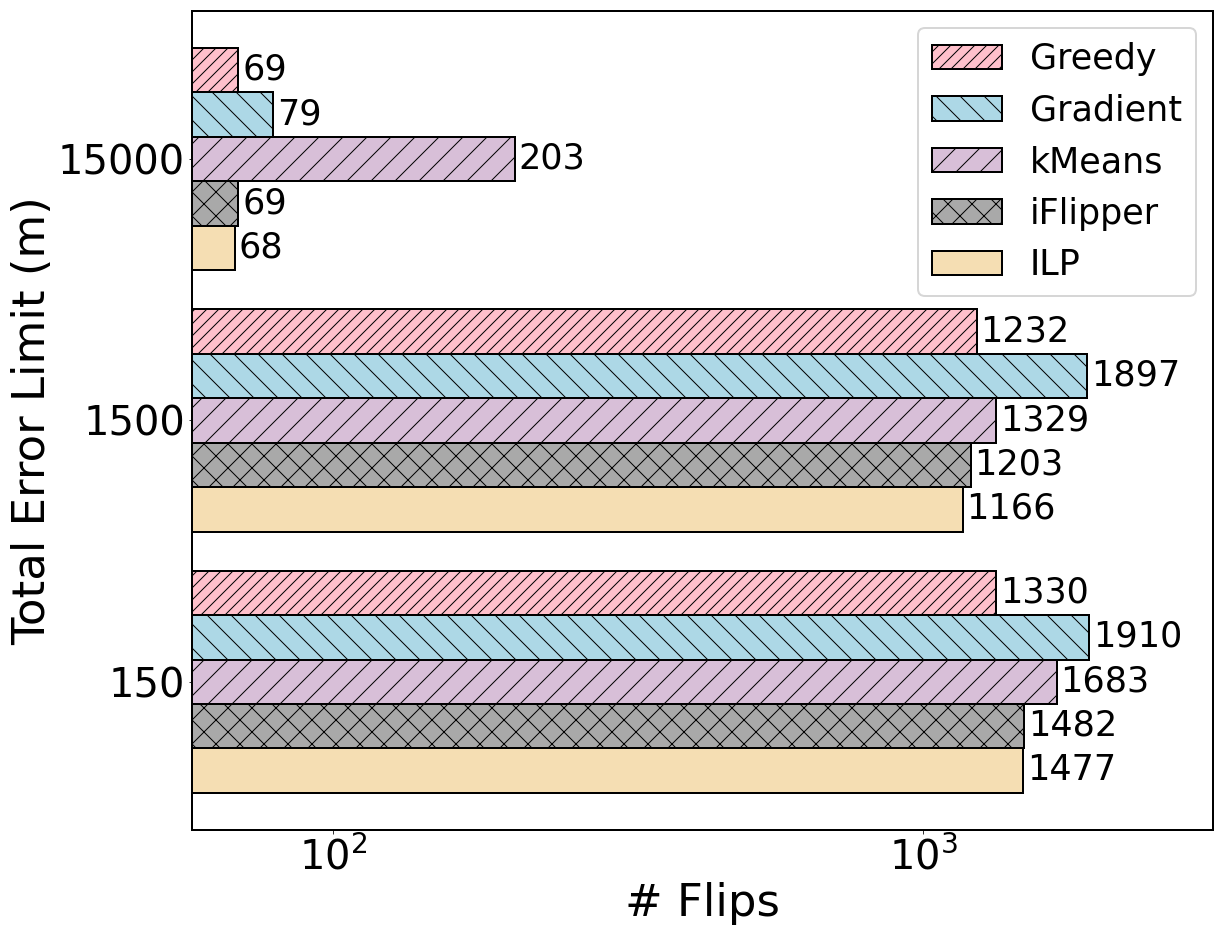}
     \vspace{-0.5cm}
     \caption{{\sf Number of flips}}
     \label{fig:compas_flip}
  \end{subfigure} 
  \begin{subfigure}{0.69\columnwidth}
     \includegraphics[width=\columnwidth]{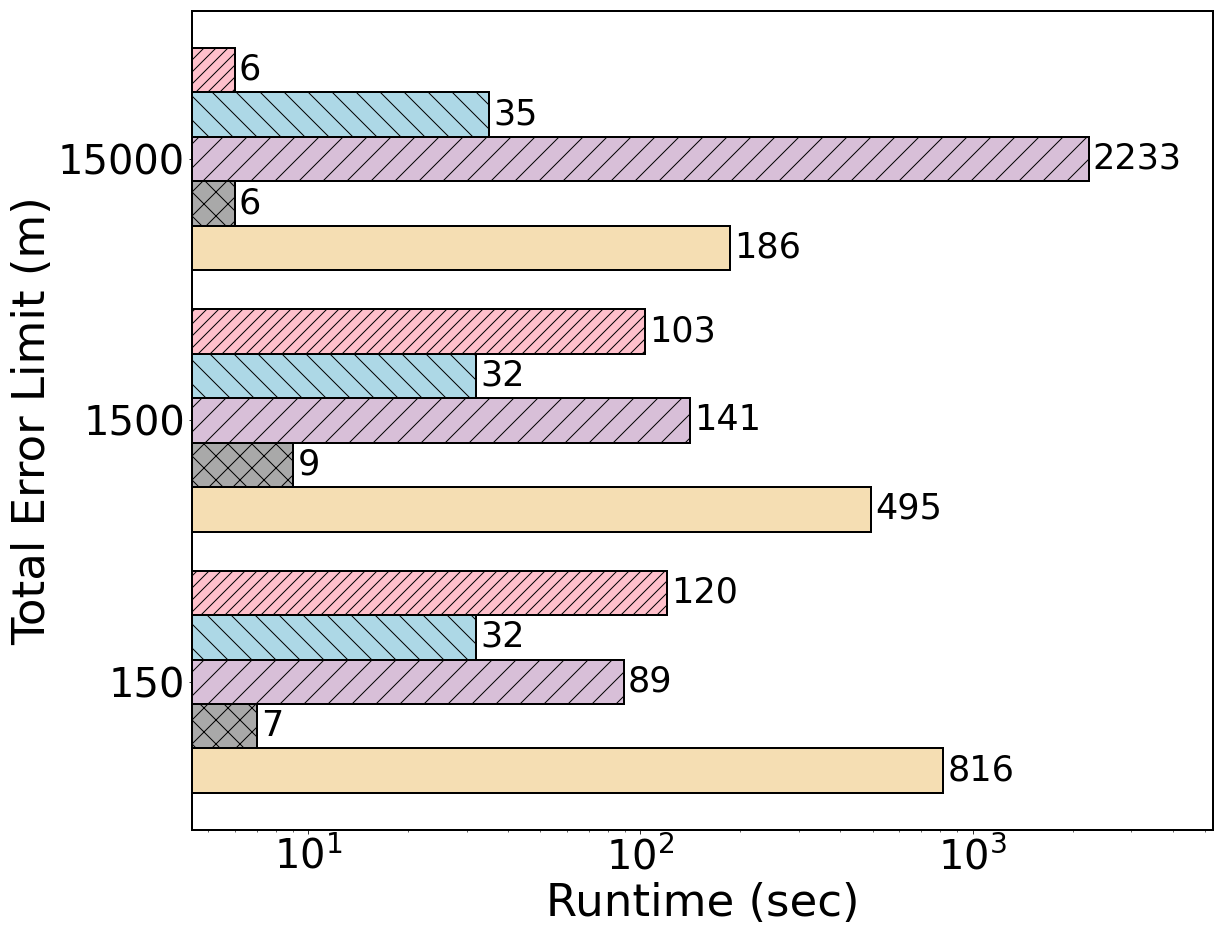}
     \caption{{\sf Runtime (sec)}}
     \label{fig:compas_runtime}
  \end{subfigure} 
  \vspace{-0.3cm}
     \caption{A detailed comparison of \systems{} against three na\"ive solutions (\greedy{}, \gradient{}, and \kmeans{}) and ILP solver on the COMPAS dataset where we use the kNN-based similarity matrix. Here the initial amount of total error is 16,454.0. We show the results for three different total error limits ($m$). All three subfigures use the same legends.}
 \label{fig:solutioncomparison_compas}
\end{figure}

\vspace{0.1cm}
\section{Ablation Study for COMPAS dataset}
\label{sec:ablationcompas}
We continue from Section~\ref{sec:ablationstudy} and provide the ablation study for the COMPAS dataset where we use the kNN-based similarity matrix in Figure~\ref{fig:ablationperformance_compas}. The observations are similar to those of Figure~\ref{fig:ablationperformance} where both adaptive rounding and reverse greedy algorithms are necessary for \systems{} to provide a near-exact solution. In addition, Table~\ref{tbl:ablationruntime_compas} shows the average runtime of each component in \systems{} in Figure~\ref{fig:ablationperformance_compas} and the results are similar to Table~\ref{tbl:ablationruntime} where the proposed algorithms are efficient in practice.

\begin{figure}[ht]
  \centering
  \begin{subfigure}{0.54\columnwidth}
     \includegraphics[width=\columnwidth]{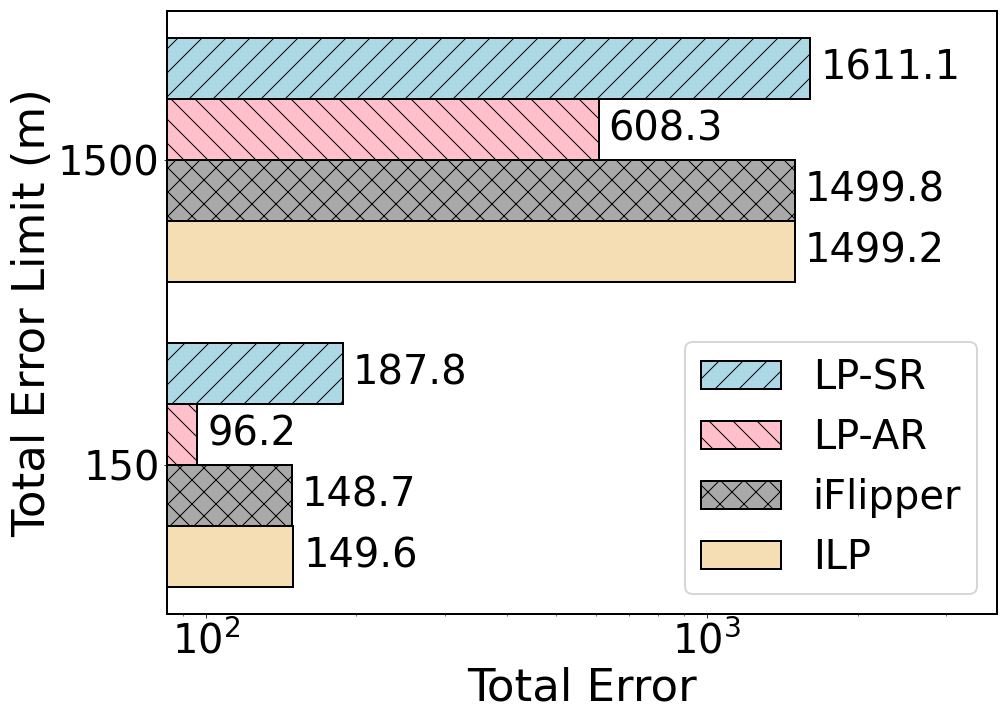}
     \vspace{-0.5cm}
     \caption{{\sf Total error}}
     \label{fig:ablationsolutionviolation_compas}
  \end{subfigure}
  \begin{subfigure}{0.44\columnwidth}
     \includegraphics[width=\columnwidth]{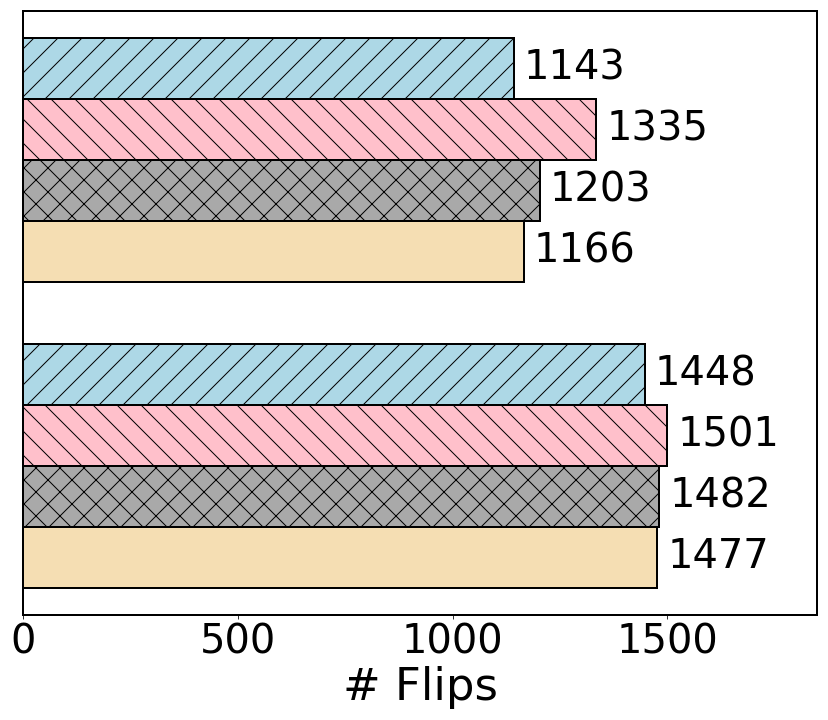}
     \vspace{-0.5cm}
     \caption{{\sf Number of flips}}
     \label{fig:ablationsolutionflip_compas}
  \end{subfigure} 
    \vspace{-0.3cm}
     \caption{Ablation study for \systems{} on the COMPAS dataset and the kNN-based similarity matrix.}
 \label{fig:ablationperformance_compas}
\end{figure}

\begin{table}[ht]
  \centering
  \begin{tabular}{lc}
    \toprule
    \multicolumn{1}{c}{\bf Method} & {\bf Avg. Runtime (sec)} \\ 
    \midrule
    LP Solver (effectively includes Alg.~\ref{alg:converting}) & 5.87\\
    + Adaptive Rounding (Alg.~\ref{alg:rounding}) & 0.09 \\
    + Reverse Greedy (Alg.~\ref{alg:reversegreedy}) & 2.19 \\
    \bottomrule
  \end{tabular}
  \caption{Avg. runtimes of \systems{}'s components in Figure~\ref{fig:ablationperformance_compas}.}
  \label{tbl:ablationruntime_compas}
\end{table}

\vspace{-0.6cm}
\section{Comparison with other ML models}
\label{sec:accuracyfairnessothermodels}
In Section~\ref{sec:baselinecomparison}, we compared \systems{} with the baselines using logistic regression. In this section, we perform the same experiments using random forest and neural network models. Figure~\ref{fig:tradeoffcurves_rf} and Figure~\ref{fig:tradeoffcurves_nn} are the trade-off results using the random forest and neural network, respectively. The key trends are still similar to Figure~\ref{fig:tradeoffcurves} where \systems{} consistently outperforms the baselines in terms of accuracy and fairness. The results clearly demonstrate that how \systems{}'s pre-processing algorithm benefits various ML models.

\begin{figure*}[ht]
  \centering
  \begin{subfigure}{0.33\textwidth}
     \includegraphics[width=\columnwidth]{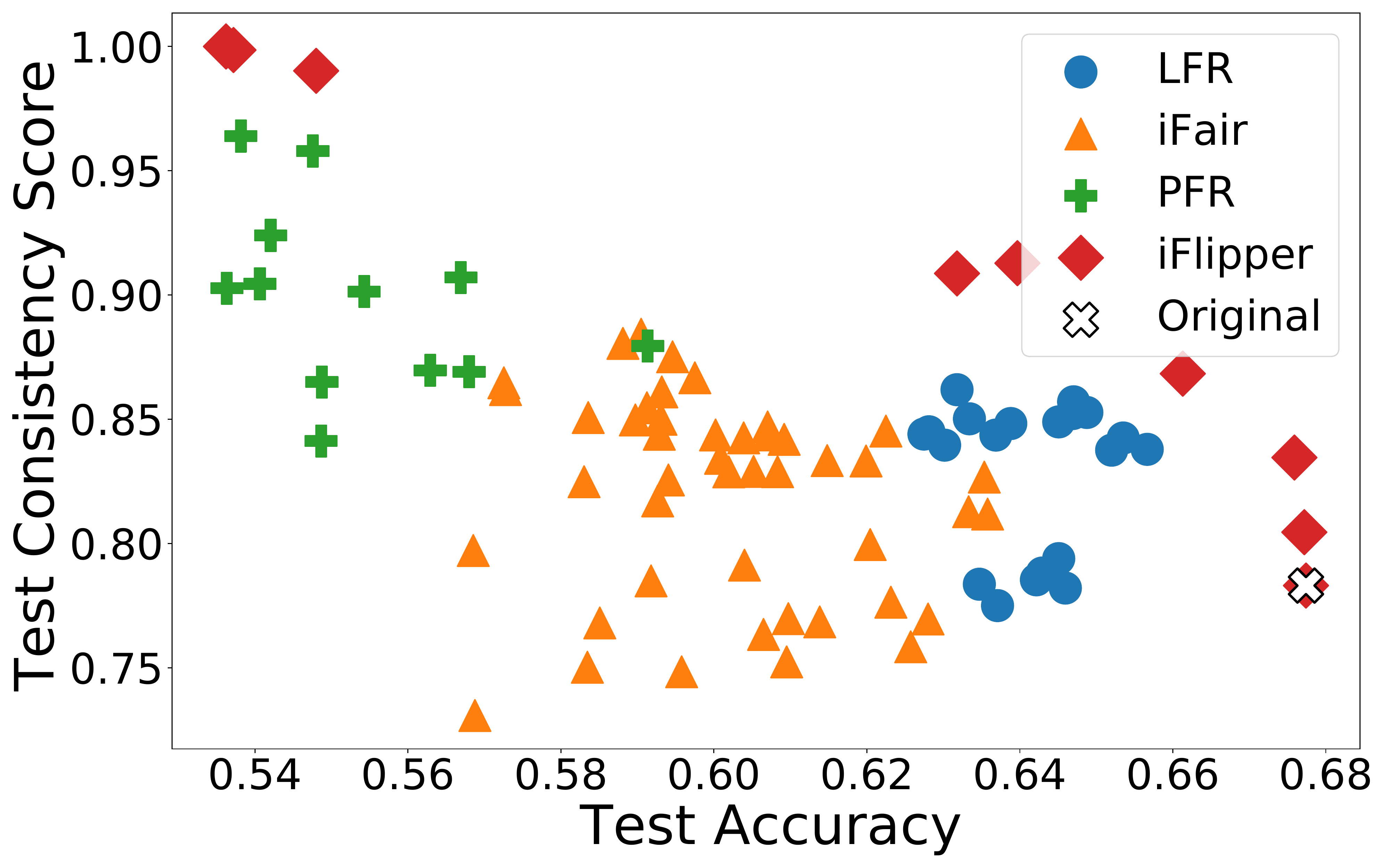}
     \caption{{\sf COMPAS-kNN}}
     \label{fig:COMPAS-kNN}
  \end{subfigure}
  \begin{subfigure}{0.33\textwidth}
    \includegraphics[width=\columnwidth]{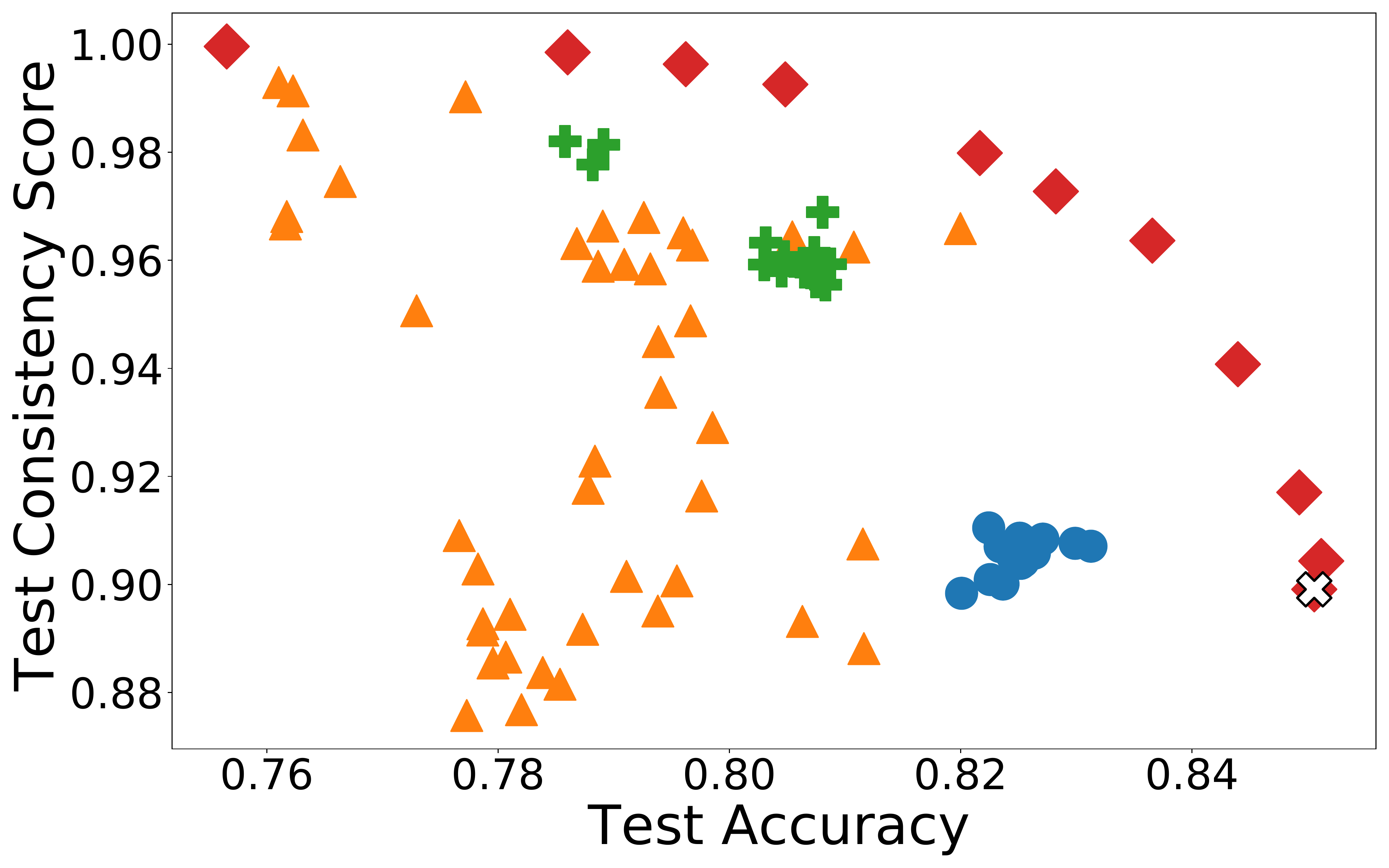}
     \caption{{\sf AdultCensus-kNN}}
     \label{fig:AdultCensus-kNN}
  \end{subfigure} 
  \begin{subfigure}{0.33\textwidth}
    \includegraphics[width=\columnwidth]{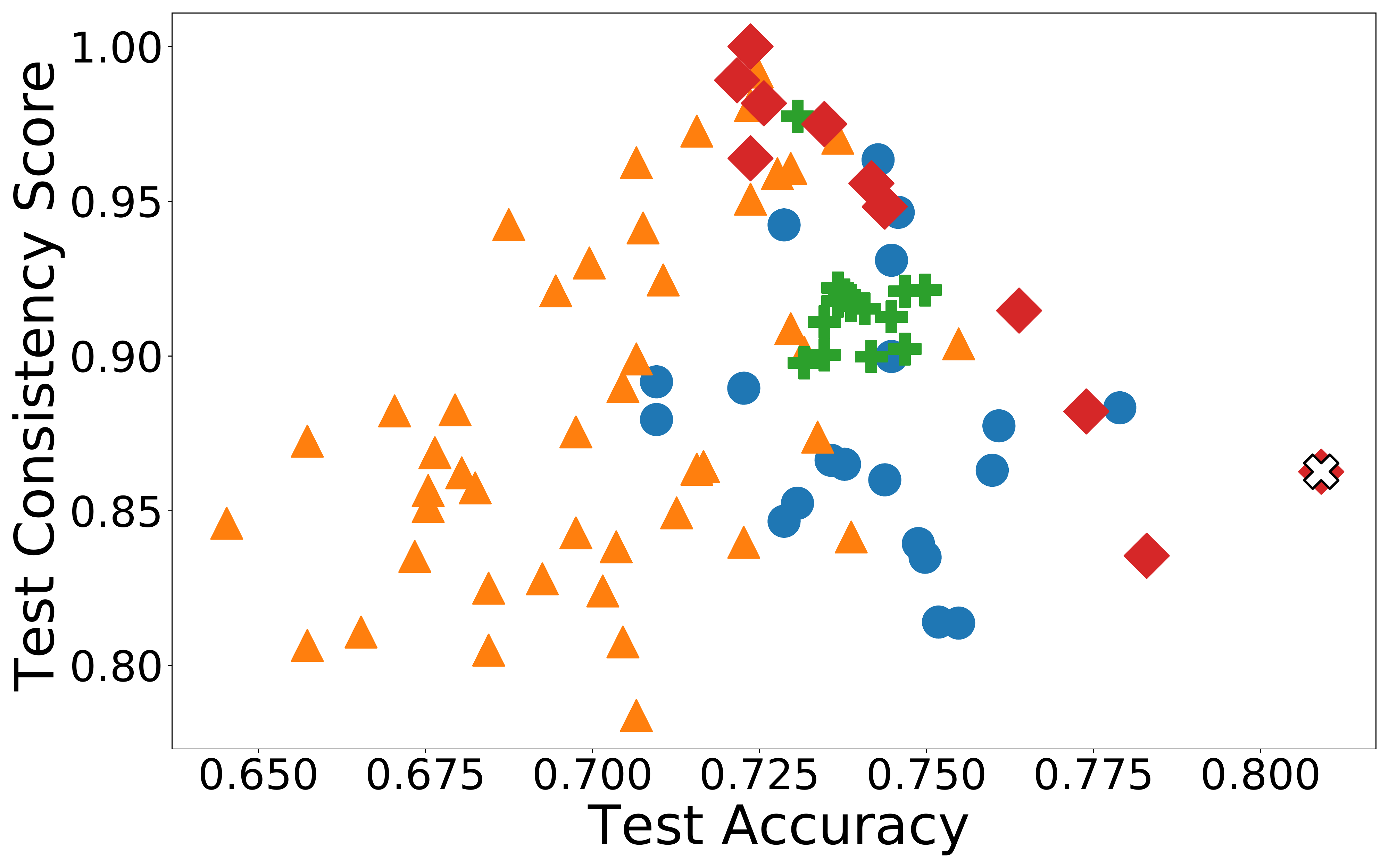}
     \caption{{\sf Credit-kNN}}
     \label{fig:Credit-kNN}
  \end{subfigure}
  \begin{subfigure}{0.33\textwidth}
     \includegraphics[width=\columnwidth]{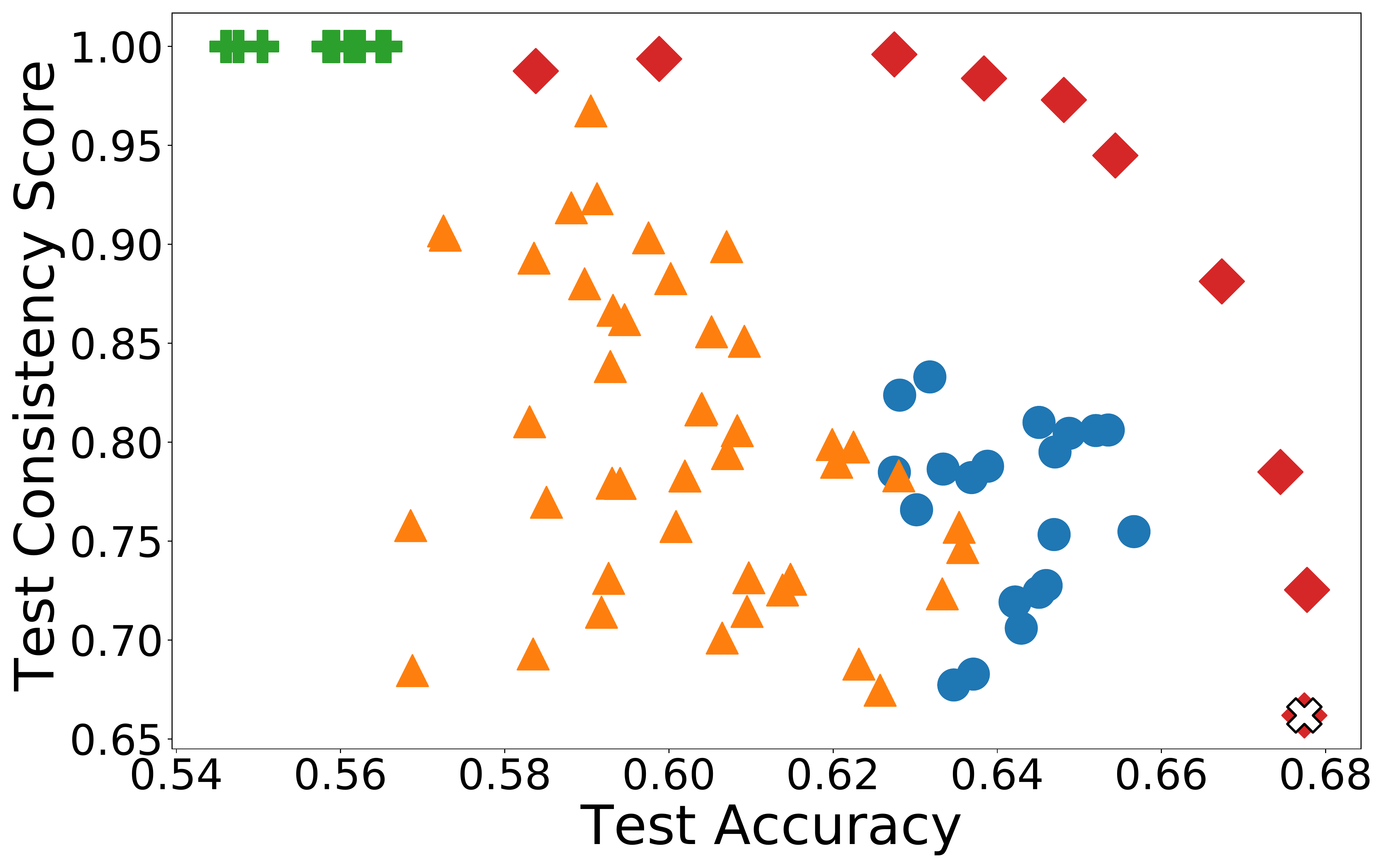}
     \caption{{\sf COMPAS-threshold}}
     \label{fig:COMPAS-Threshold}
  \end{subfigure} 
  \begin{subfigure}{0.33\textwidth}
    \includegraphics[width=\columnwidth]{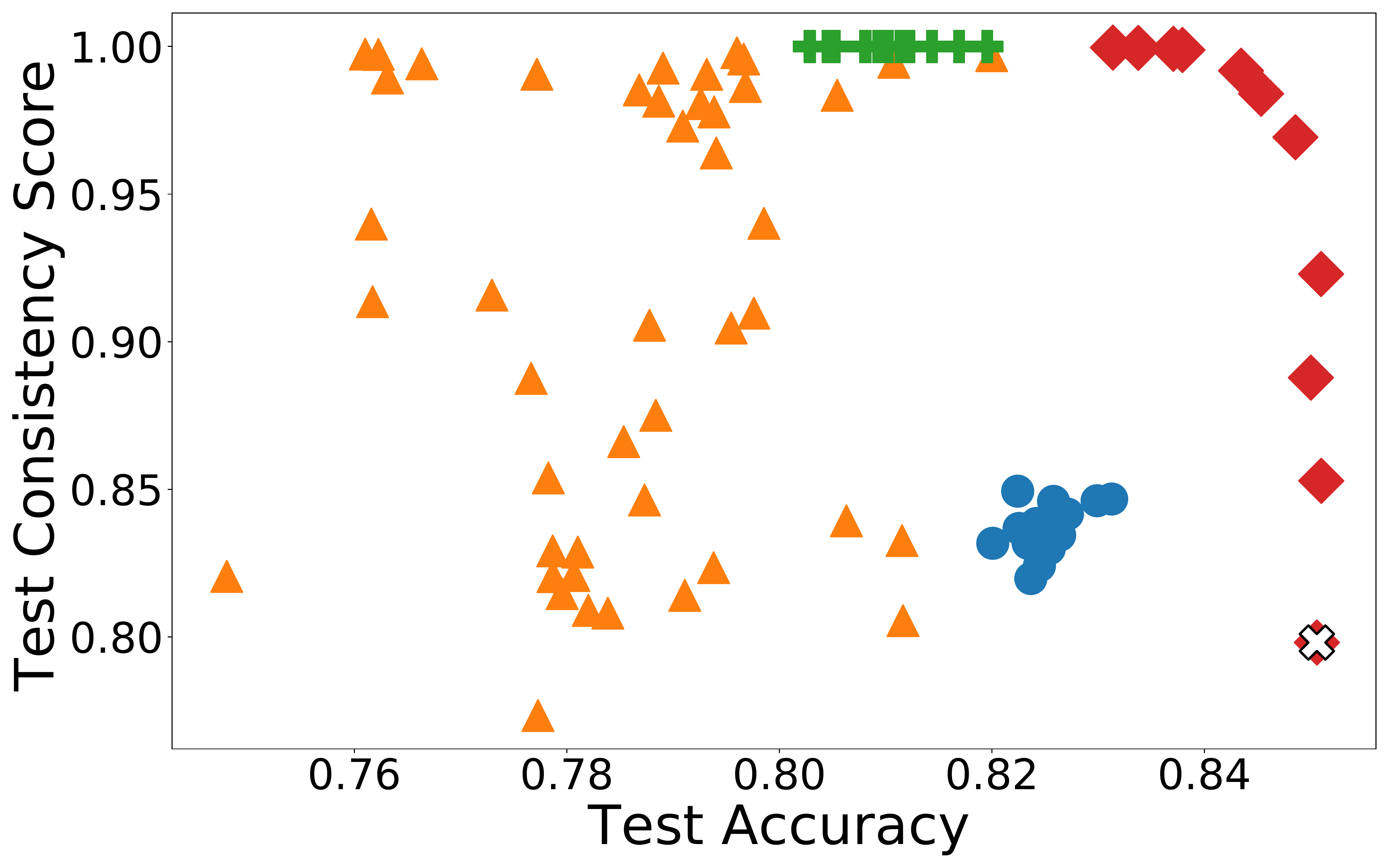}
     \caption{{\sf AdultCensus-threshold}}
     \label{fig:AdultCensus-Threshold}
  \end{subfigure}
  \begin{subfigure}{0.33\textwidth}
    \includegraphics[width=\columnwidth]{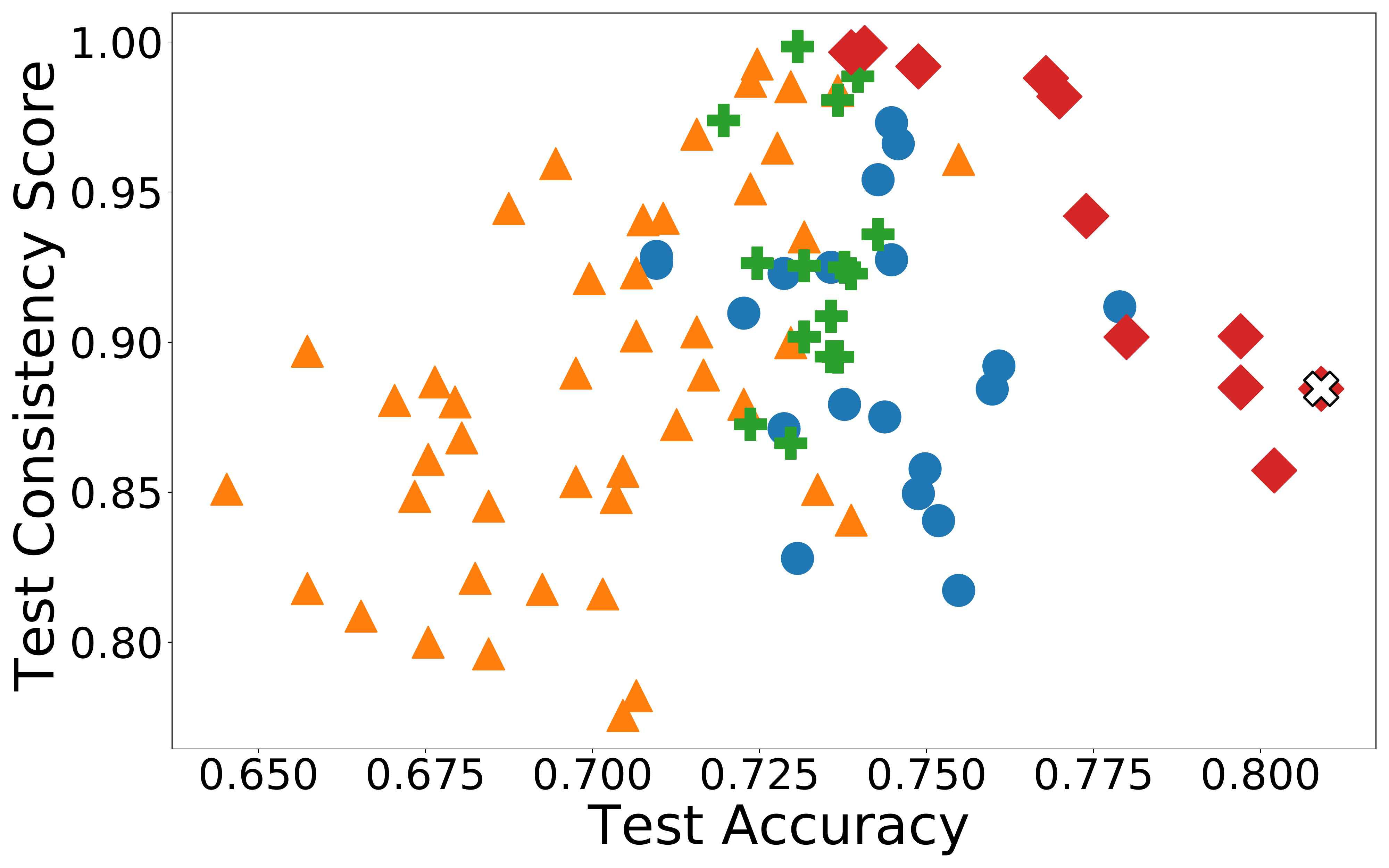}
     \caption{{\sf Credit-threshold}}
     \label{fig:Credit-Threshold}
  \end{subfigure}
     \caption{Accuracy-fairness trade-offs of random forest on the three datasets using the two similarity matrices. In addition to the four methods LFR, iFair, PFR, and \systems{}, we add the result of model training without any pre-processing and call it ``Original.'' As a result, only \systems{} shows a clear accuracy and fairness trade-off.}
 \label{fig:tradeoffcurves_rf}
\end{figure*}

\begin{figure*}[htb]
  \centering
  \begin{subfigure}{0.33\textwidth}
    \includegraphics[width=\columnwidth]{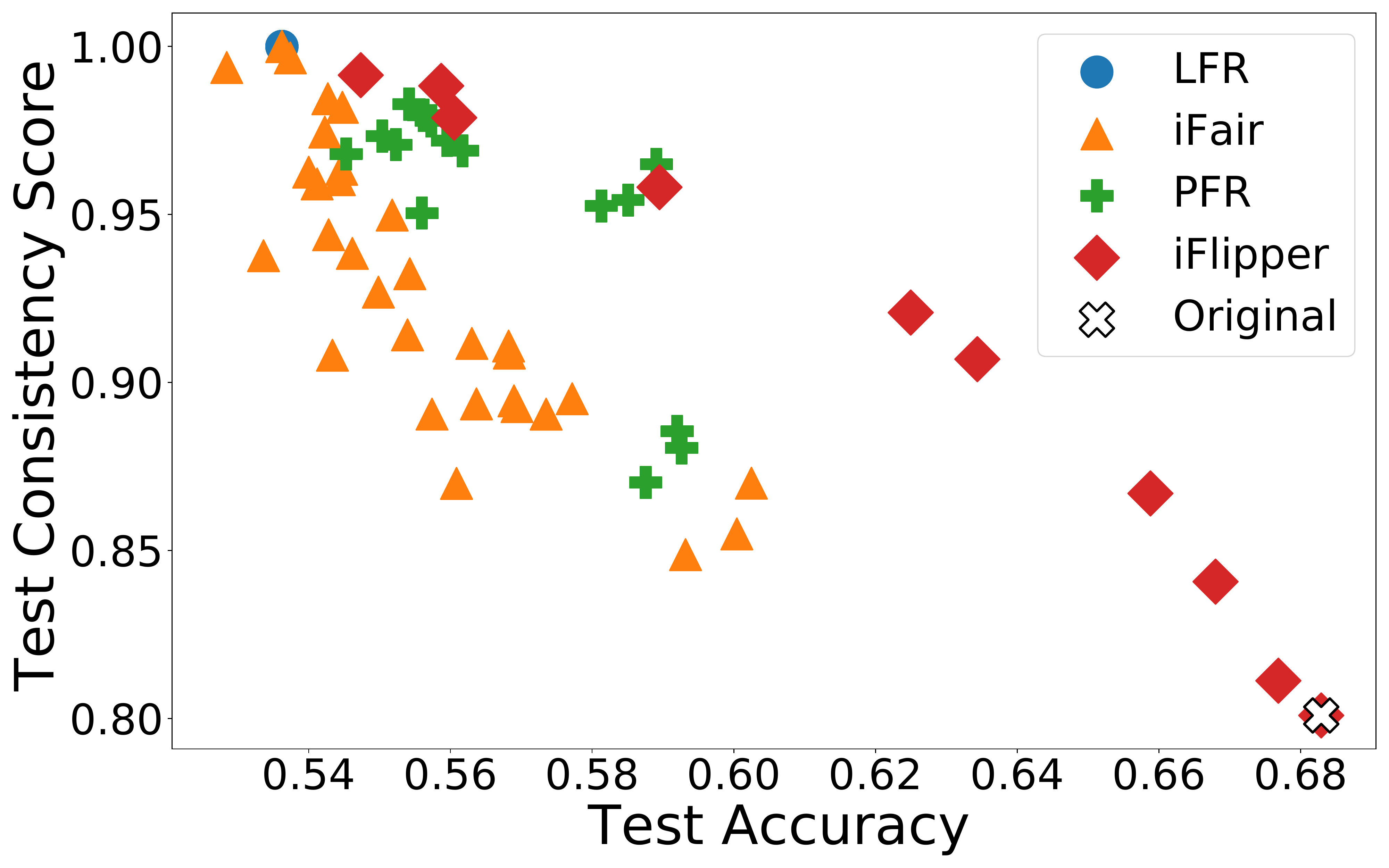}
     \caption{{\sf COMPAS-kNN}}
     \label{fig:COMPAS-kNN}
  \end{subfigure}
  \begin{subfigure}{0.33\textwidth}
    \includegraphics[width=\columnwidth]{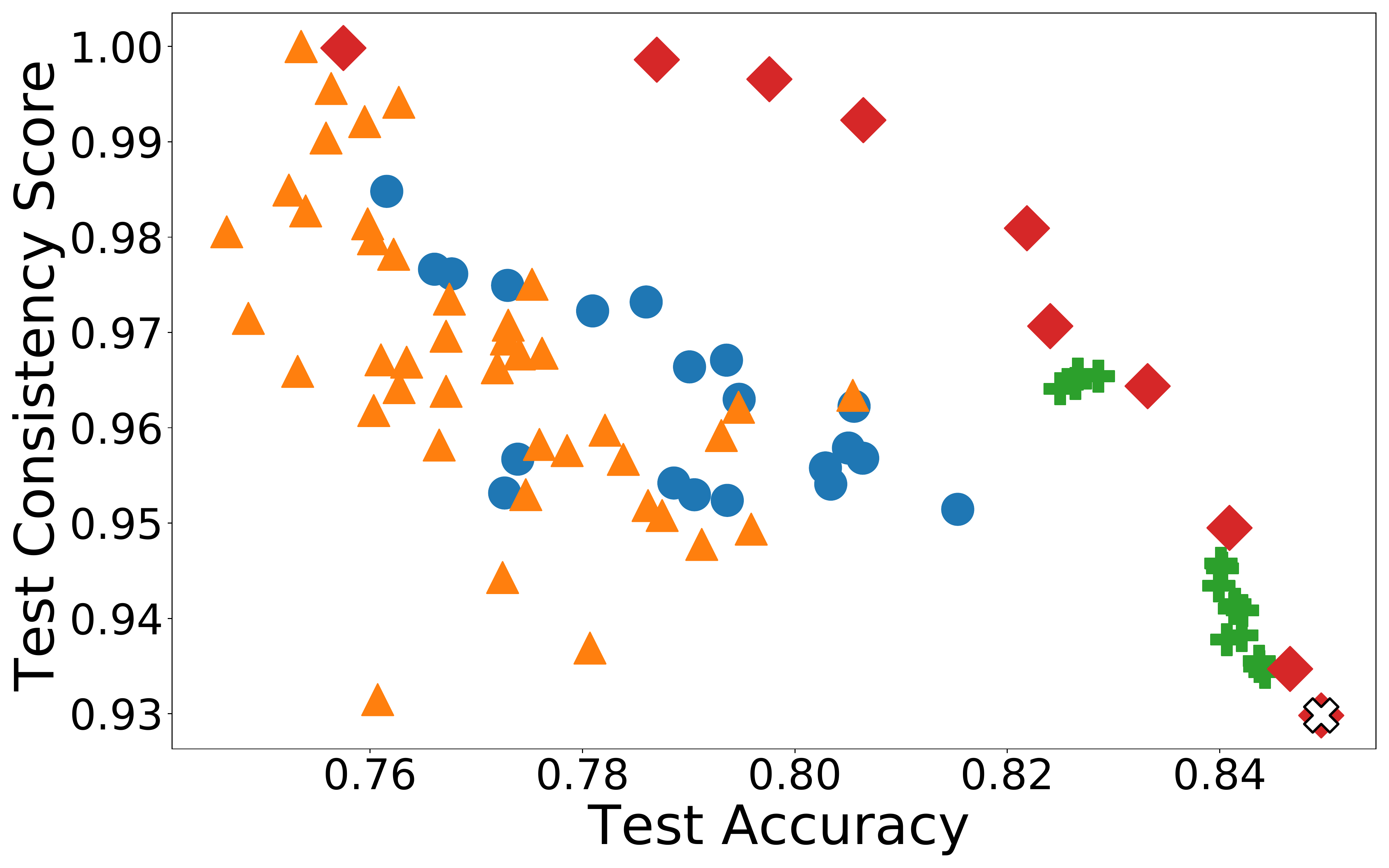}
     \caption{{\sf AdultCensus-kNN}}
     \label{fig:AdultCensus-kNN}
  \end{subfigure} 
  \begin{subfigure}{0.33\textwidth}
    \includegraphics[width=\columnwidth]{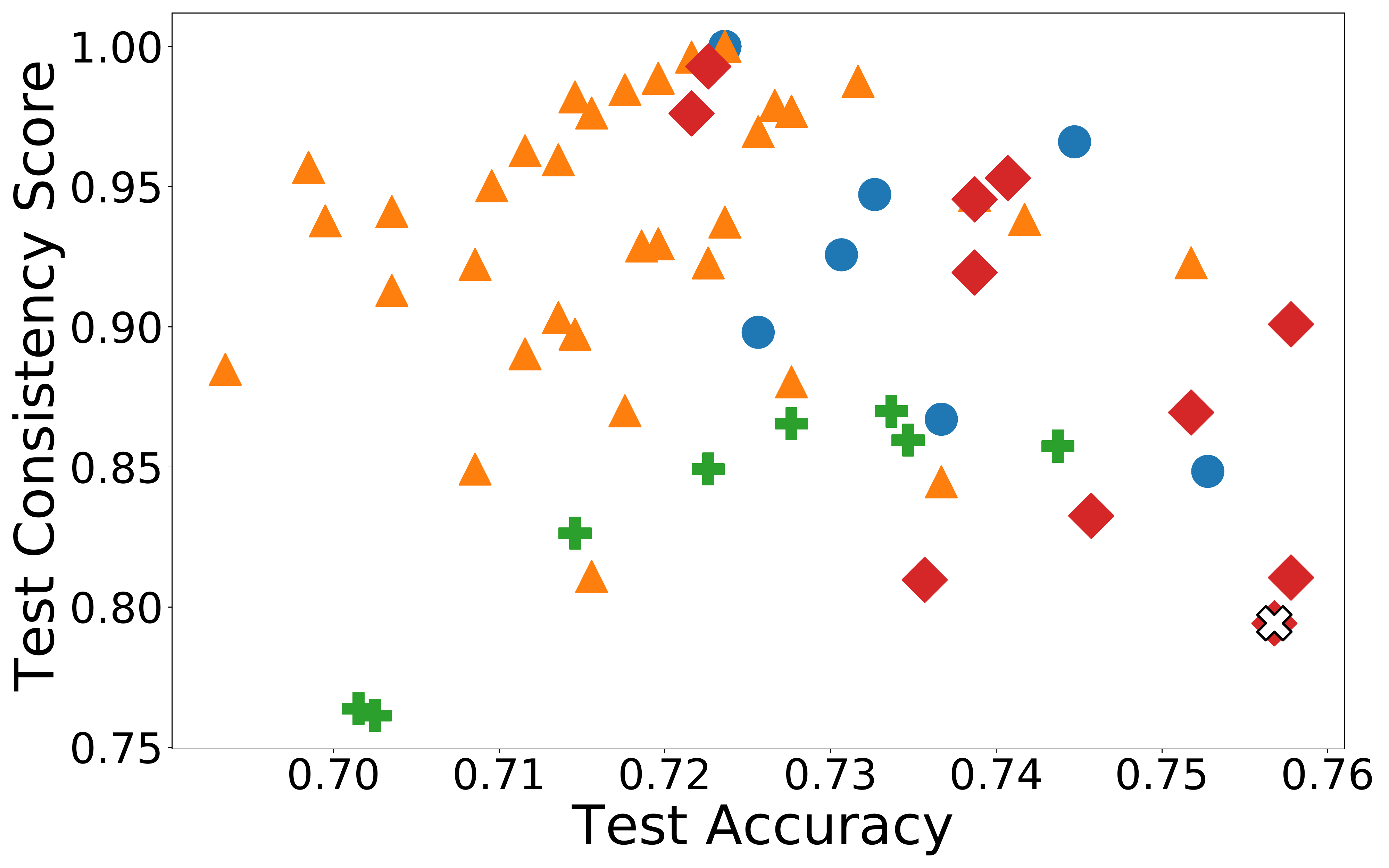}
     \caption{{\sf Credit-kNN}}
     \label{fig:Credit-kNN}
  \end{subfigure}
  \begin{subfigure}{0.33\textwidth}
    \includegraphics[width=\columnwidth]{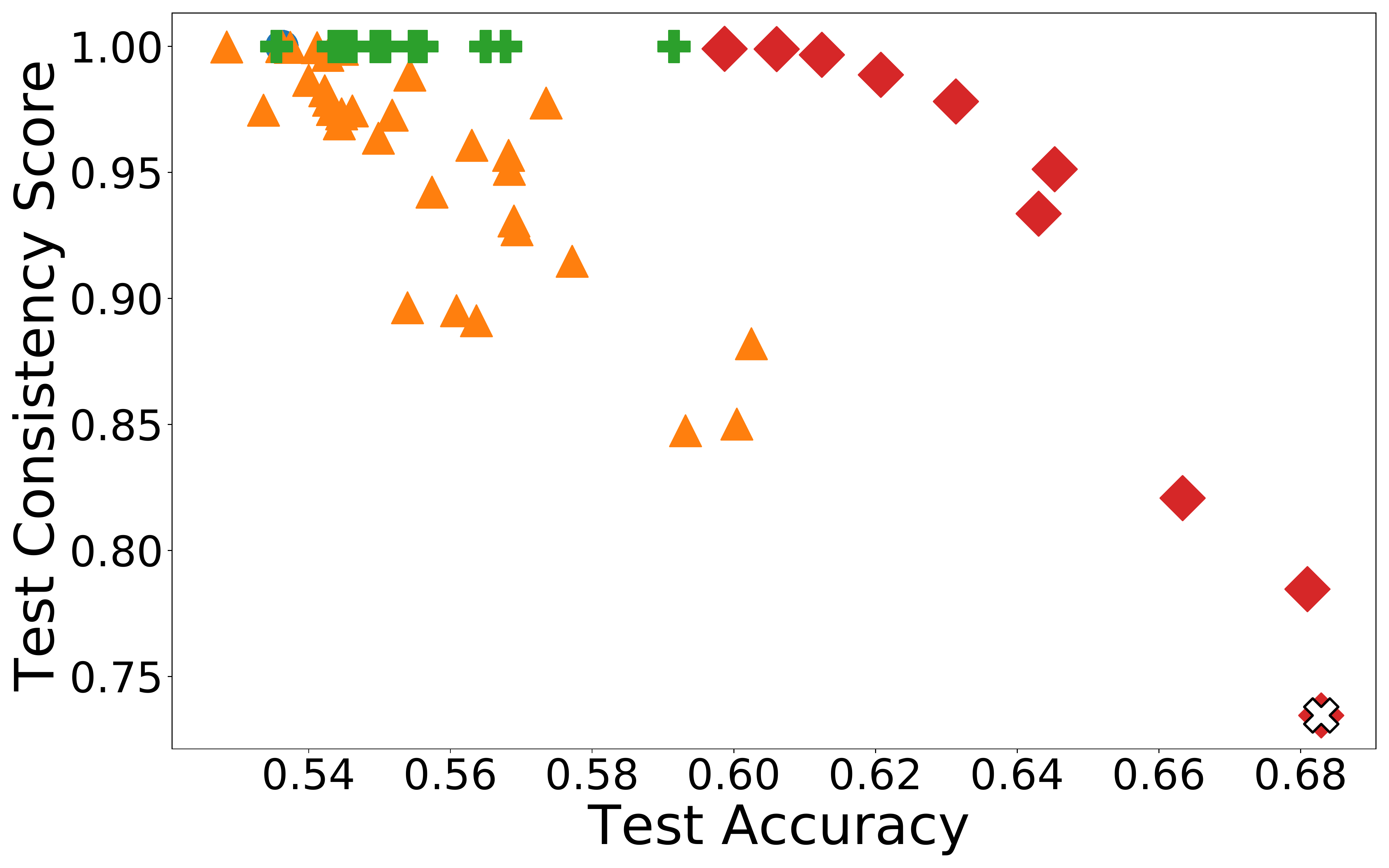}
     \caption{{\sf COMPAS-threshold}}
     \label{fig:COMPAS-Threshold}
  \end{subfigure} 
  \begin{subfigure}{0.33\textwidth}
    \includegraphics[width=\columnwidth]{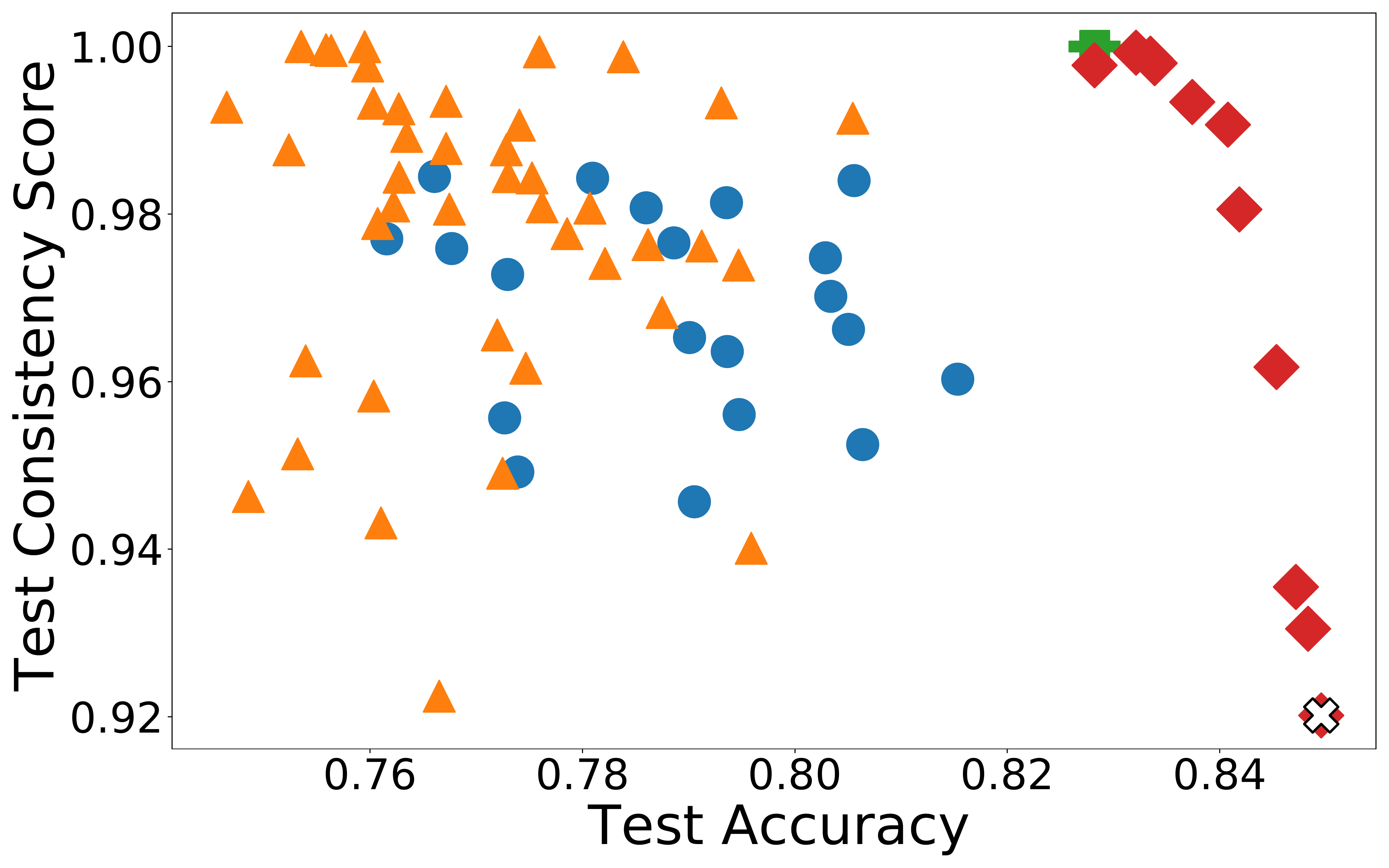}
     \caption{{\sf AdultCensus-threshold}}
     \label{fig:AdultCensus-Threshold}
  \end{subfigure}
  \begin{subfigure}{0.33\textwidth}
    \includegraphics[width=\columnwidth]{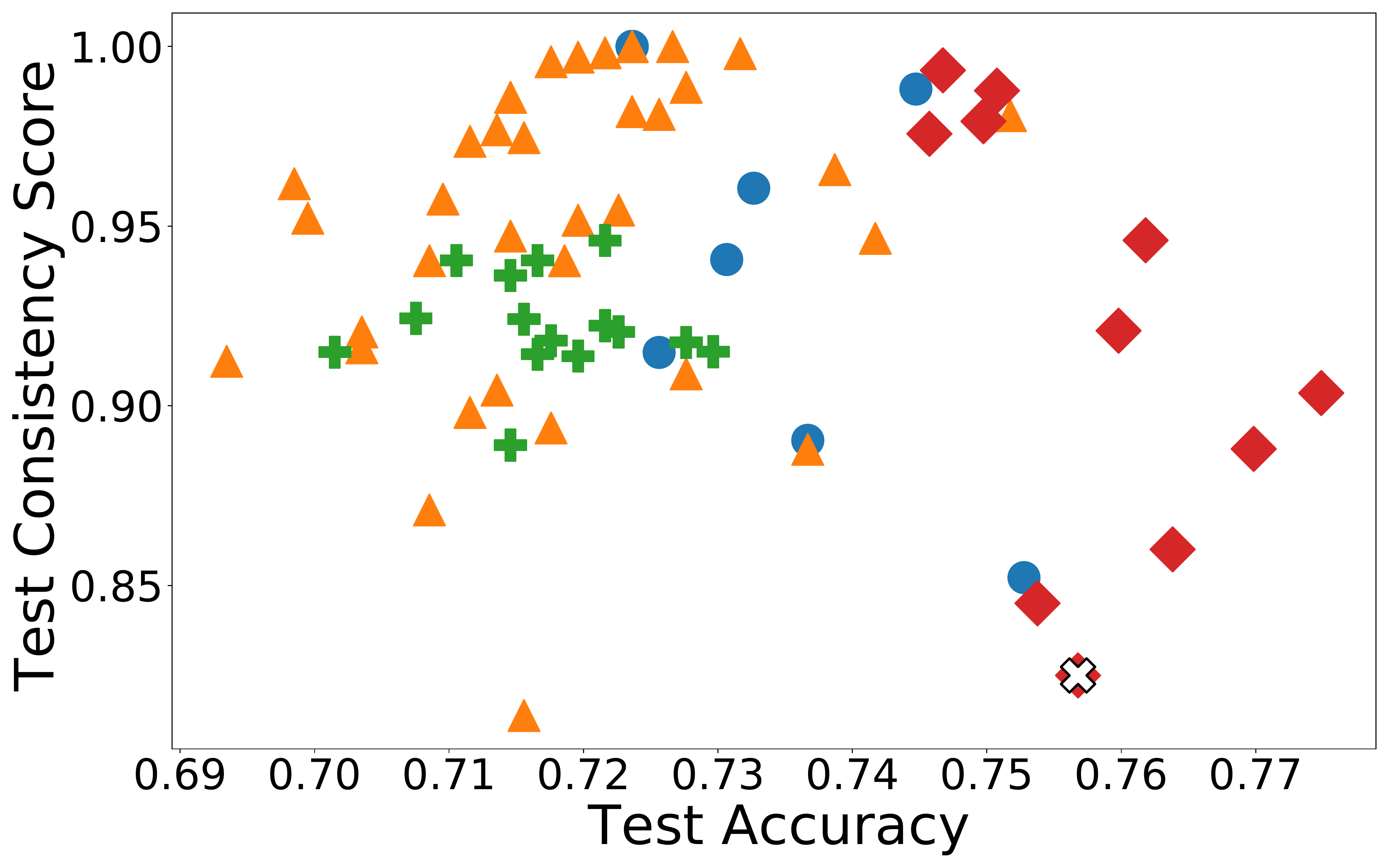}
     \caption{{\sf Credit-threshold}}
     \label{fig:Credit-Threshold}
  \end{subfigure}
     \caption{Accuracy-fairness trade-offs of neural network on the three datasets using the two similarity matrices. In addition to the four methods LFR, iFair, PFR, and \systems{}, we add the result of model training without any pre-processing and call it ``Original.'' As a result, only \systems{} shows a clear accuracy and fairness trade-off.}
 \label{fig:tradeoffcurves_nn}
\end{figure*}

}{}

\end{document}